\documentclass{article}

\usepackage[final]{neurips_2025}

\usepackage[utf8]{inputenc} % allow utf-8 input
\usepackage[T1]{fontenc}    % use 8-bit T1 fonts
\usepackage{hyperref}       % hyperlinks
\usepackage{url}            % simple URL typesetting
\usepackage{booktabs}       % professional-quality tables
\usepackage{amsfonts}       % blackboard math symbols
\usepackage{nicefrac}       % compact symbols for 1/2, etc.
\usepackage{microtype}      % microtypography
\usepackage{xcolor}         % colors

\usepackage{multirow} 
\usepackage{graphicx}
\usepackage{algorithm}
\usepackage{algorithmic}

\usepackage{amsmath}
\usepackage{amssymb}
\usepackage{mathtools}
\usepackage{amsthm}

\usepackage{authblk}

\newtheorem{theorem}{Theorem}
\newtheorem{definition}{Definition}
\newtheorem{repdefinition}{Definition}
\newtheorem{corollary}[theorem]{Corollary}

\usepackage{enumitem}
\usepackage{xcolor}
\usepackage{tikz}
\usepackage{subcaption}
\usetikzlibrary{shapes.symbols}
\usetikzlibrary{shapes.geometric}

\definecolor{myOrange}{HTML}{EF553B}
\definecolor{myPurple}{HTML}{AB63FA} % 定义颜色，HTML 方式
\definecolor{myBlue}{HTML}{636EFA}
\definecolor{myGreen}{HTML}{00CC96}

\newcommand{\squaresymbolp}{\textcolor{myPurple}{\rule{1ex}{1ex}}}

\newcommand{\circlesymbolb}{\tikz\draw[fill=myBlue,draw=myBlue] (0,0) circle (0.5ex);}
\newcommand{\circlesymbolg}{\tikz\draw[fill=myGreen,draw=myGreen] (0,0) circle (0.5ex);}
\newcommand{\circlesymbolo}{\tikz\draw[fill=myOrange,draw=myOrange] (0,0) circle (0.5ex);}

\title{DAIL: Beyond Task Ambiguity for Language-Conditioned Reinforcement Learning}

\author{
    \textbf{Runpeng Xie}$^{*1}$, 
    \textbf{Quanwei Wang}$^{*2}$,
    \textbf{Hao Hu}$^{3}$,
    \textbf{Zherui Zhou}$^{4}$,
    \textbf{Ni Mu}$^{2}$,
    \textbf{Xiyun Li}$^{5}$,
    \textbf{Yiqin Yang}$^{\dagger1}$,
    \textbf{Shuang Xu}$^{1}$,
    \textbf{Qianchuan Zhao}$^{2}$,
    \textbf{Bo Xu}$^{\dagger1}$\\
    \small
    $^{1}$The Key Laboratory of Cognition and Decision Intelligence for Complex Systems,\\ Institute of Automation, Chinese Academy of Sciences, Beijing, China \\
    $^{2}$Department of Automation, Tsinghua University\quad$^{3}$Moonshot AI\\
    $^{4}$Department of Computer Science and Engineering, Washington University\quad$^{5}$Tecent AI Lab\\
    \texttt{xierunpeng2021@ia.ac.cn, wqw21@mails.tsinghua.edu.cn}\\
}

\begin{document}

\maketitle

\begingroup
\renewcommand\thefootnote{} % 去掉脚注编号
\footnotetext{${}^{*}$Equal contribution.}
\footnotetext{${}^{\dagger}$Correspondence to Yiqin Yang and Bo Xu.}
\endgroup

\begin{abstract}
Comprehending natural language and following human instructions are critical capabilities for intelligent agents. 
However, the flexibility of linguistic instructions induces substantial ambiguity across language-conditioned tasks, severely degrading algorithmic performance.
To address these limitations, we present a novel method named DAIL (Distributional Aligned Learning), featuring two key components: distributional policy and semantic alignment.
Specifically, we provide theoretical results that the value distribution estimation mechanism enhances task differentiability.
Meanwhile, the semantic alignment module captures the correspondence between trajectories and linguistic instructions.
Extensive experimental results on both structured and visual observation benchmarks demonstrate that DAIL effectively resolves instruction ambiguities, achieving superior performance to baseline methods. Our implementation is available at \href{https://github.com/RunpengXie/Distributional-Aligned-Learning}{https://github.com/RunpengXie/Distributional-Aligned-Learning}.

% DAIL introduces two key innovations: a distributional language-guide policy to improve learning efficiency and a trajectory-wise semantic alignment algorithm to enhance task discrimination.
% We conduct an extensive evaluation on both structured and visual observation benchmarks.
% However, due to the flexibility of languages, 
% learning and generalizing language-conditioned policies is extremely challenging.
% While prior work in language-conditioned Reinforcement Learning~(RL) typically relies solely on pre-trained language  models to address this issue, few have explored this problem in a deeper aspect. 
% Through an illustrative experiment, we show that RL algorithms fail to discriminate tasks as the number of tasks explodes, significantly limiting the performance of language-conditioned RL.
\end{abstract}

\section{Introduction}
Artificial agents are anticipated to master diverse skills while effectively interpreting human instructions and generalizing across various tasks. Therefore, comprehension and following of natural language emerges as critical capabilities for agents in this context.
For example, language-conditioned agents have achieved remarkable success in robotic manipulation~\cite{luketina2019survey,brohan2023rt}, text-based environments~\cite{li2022pre, carta2023grounding}, visual navigation~\cite{zhou2023language, hermann2017grounded}, and autonomous driving~\cite{cui2024drive, roh2020conditional}.
% 语言指令任务的成功应用
% robotic manipulation~\cite{luketina2019survey}, visual navigation~\cite{zhou2023language}, and interactive text-based environments~\cite{li2022pre}.
The fundamental requirement has propelled language-conditioned reinforcement learning (RL) to the forefront of research, which focuses on enabling agents to interpret and execute natural language instructions through RL frameworks.
% 语言RL的特点（总体一句话）

Recent advancements in the language-conditioned RL domain have focused on bridging the gap between linguistic understanding and decision-making processes, aiming to create agents capable of executing complex instructions with human-like adaptability.
For example, some works~\cite{co2018guiding, bing2023meta} integrate language-conditioned policy with trial-and-error learning, significantly improving the performance and sample efficiency in robot task acquisition. Meanwhile, some studies~\cite{goyal2021pixl2r, goyal2019using} leverage expert demonstrations to map language instructions to reward signals directly to address the issue of sparse rewards in language-conditioned RL. 
% 举例1~2个前人的工作（一句话）
% For detailed introduction to related work, please refer to Appendix~\ref{appendix:related}.
However, linguistic instructions exhibit high flexibility, which induces exponential growth in task space.
In this case, identical tasks may have divergent expressions, while distinct tasks share overlapping language instructions. This variability makes language-conditioned RL methods face the significant challenge of task ambiguity, which hinders the agent from discerning the connection between rewards and task objectives, thereby significantly impairing learning efficiency.
% 任务模糊性对策略的危害（半句话）

To solve this issue, we propose a novel method, DAIL, which consists of two main components: a distributional language-guided policy and a trajectory-wise semantic alignment module. 
Specifically, the distributional policy module~\cite{bellemare2017distributional} estimates the value distribution, preserving more information to aid task discrimination. Theoretically, we analyze the sample complexity required to guarantee task disambiguation and establish that distributional estimation methods are sample-efficient in offline settings.
On the other hand, the semantic alignment module constrains the language instruction representations by maximizing the mutual information between trajectories and the instructions, thereby achieving better differentiation across instructions.
With theoretical guarantees, the two modules enable precise disambiguation and execution of linguistic instructions, thereby improving learning efficiency.

We conduct extensive experiments on both structured observation~\cite{chevalier2018babyai} and visual observation~\cite{shridhar2020alfred} benchmarks. 
This design is to validate the external validity of the DAIL agent, progressing from less complex structured inputs to more complex and expressive visual observations.
The experimental results show that DAIL outperforms the state-of-the-art language-conditioned RL methods in both benchmarks.
Further, the visualization analysis demonstrates that DAIL can learn a non-ambiguous task representation compared with baselines.
Our main contributions are summarized as follows:
\begin{itemize}
    \item First, we highlight the critical issue of task ambiguity and empirically analyze the limitations of current mainstream methods. We define the task distinction in our setting and analyze the sample complexity to avoid task ambiguity theoretically.
    \item Second, we propose DAIL, a simple yet efficient language-conditioned learning framework, which addresses the task ambiguity issue based on distributional policy and semantic alignment.
    \item Lastly, we conduct extensive experiments to show that DAIL significantly outperforms conventional language-conditioned methods. 
    The results indicate that by improving task discrimination, we can effectively mitigate the task ambiguity issue, thereby broadening the application of language-conditioned RL.
\end{itemize}

\section{Preliminaries}

\begin{figure*}[t]
    \centering
    \includegraphics[width=1.0\textwidth]{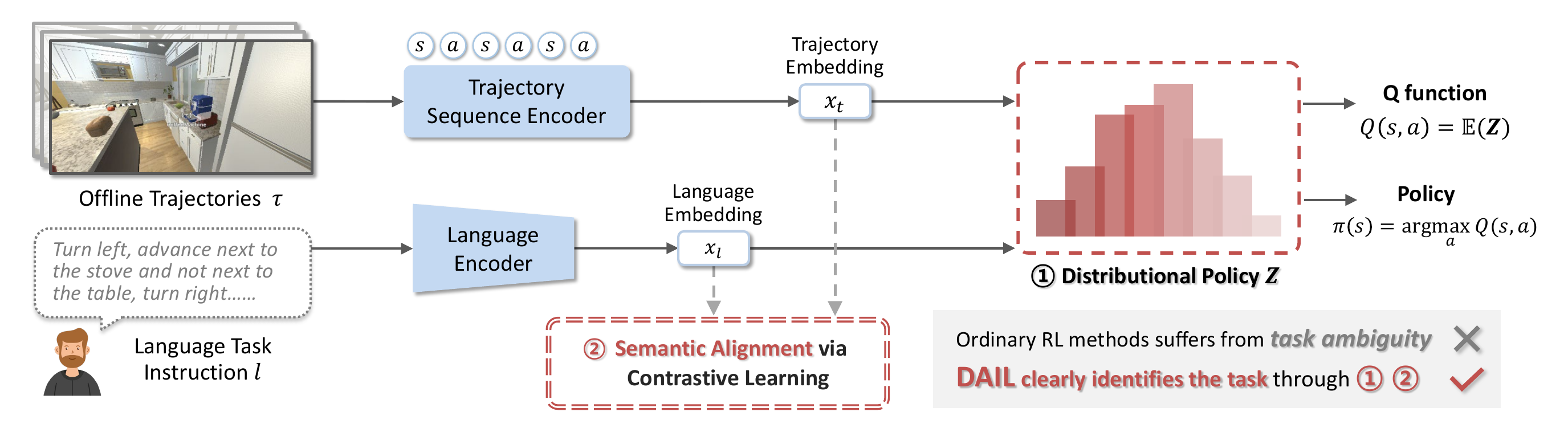} 
    \vspace{-10pt}
    \caption{
    The framework of our method.
    The key ideas are: (1) use the distributional language-guided policy to aid task discrimination, (2) use the trajectory-wise semantic alignment module to extract task representation.}
    \label{fig:arch}
\end{figure*}

\paragraph{Language-conditioned RL} 
Based on contextual markov decision process (CMDP)~\cite{hallak2015contextual}, we consider Language-conditioned Markov Decision Process (LCMDP) as a model consisting of a tuple $\mathcal{M}=(\mathcal{S, A}, P, r,\gamma,\mathcal{L},p_0,p_l)$, where $\mathcal{S}$ denotes the state space, $\mathcal{A}$ represents the action space, $\mathcal{L}$ is the language instruction space, $P(s'|s, a, l)$ represents the probabilistic transition model, $r:\mathcal{S\times A \times L}\rightarrow \mathbb{R}$ is the reward function conditioned on language instructions and
$\gamma$ is the discount factor. 
$p_0$ represents the probability distribution of the initial state, and $p_l$ denotes the probability distribution of language instruction. 
We establish language instructions $l$ as task descriptors, and each instruction uniquely specifies a task.

Language-conditioned RL aims to obtain a policy $\pi(\cdot|s,l)$ that maximizes the cumulative discounted returns under a specific distribution of language instructions:
\begin{equation}
\begin{aligned}
    \pi^*= {\arg\max}_{\pi} \mathbb{E}_{\pi} \left[\sum_{t=0}^{\infty}\gamma^t r(s_t,a_t,l)\right],
\end{aligned}
\end{equation}
where $s_0\sim p_0(\cdot), l\sim p_l(\cdot), a_t\sim \pi(\cdot|s_t, l)$ and $s_{t+1}\sim P(\cdot|s_t,a_t,l)$.
The temporal difference loss in language-conditioned RL is adapted as follows:
\begin{equation}
\begin{aligned}
L_{\text{TD}}(\theta) = &\mathbb{E}_{(s,a,s',l)\sim \mathcal{D}}[(r(s,a,l)+
\gamma\mathrm{max}_{a'}Q_{\hat{\theta}}(s',a',l)-Q_{\theta}(s,a,l))^2]
\end{aligned}
\end{equation}  
where $Q_{\theta}(s,a,l)$ is the parameterized $Q$-function conditioned language instruction $l$, and $Q_{\hat{\theta}}(s,a,l)$ is the target $Q$-network.

\paragraph{Offline RL}
Due to the high cost of real-time interaction with the environment, we consider the offline learning setting, in which we learns a policy $\pi$ without interacting with an environment.
Rather, the learning is based on a dataset $\mathcal{D}$ generated by a behavior policy $\pi_{\beta}$. One of the major challenges in offline RL is the issue of distributional shift~\citep{fujimoto2019off}, where the learned policy is different from the behavioral policy. Existing offline RL methods apply various forms of regularization to limit the deviation of the current learned policy:
\begin{equation}
    \pi^* = {\arg\max}_{\pi}\left[ J_{\mathcal{D}}(\pi) - \alpha D(\pi, \pi_{\beta})\right],
    \label{eq: offline opt}
\end{equation}
where $J_{\mathcal{D}}(\pi)$ is the cumulative discounted return of policy $\pi$ on the empirical MDP induced by the dataset $\mathcal{D}$, and $D(\pi, \pi_{\beta})$ is a divergence measure between $\pi$ and $\pi_{\beta}$.
As for the language-conditioned task, we will provide the language instruction in the evaluation.
To make our writing more concise, let $\tau$ be a full trajectory, and $\tau_t$ be the trajectory ending at time-step $t$. 
Let $p_{\mathcal{D}}(\tau, l)$ represents the joint distribution of trajectory-instruction pairs in the dataset $\mathcal{D}$.

% \paragraph{Distributional RL} is a reinforcement learning method aimed at learning the probability distribution of rewards rather than just the expected value. For a fixed policy $\pi$, we denote the random variable representing the cumulative discounted reward obtained along the policy $\pi$ as $Z^{\pi}$, and is related to standard RL as follows:
% \begin{equation}
% \begin{aligned}
%     &V^{\pi}(s) :=\mathbb{E}[Z^{\pi}(s)] = \mathbb{E}\left[\sum_{t=0}^{\infty}\gamma^t r(s_t,a_t)|s_0=s\right] \\
%     &Q^{\pi}(s,a) :=\mathbb{E}[Z^{\pi}(s,a)] = \mathbb{E}\left[\sum_{t=0}^{\infty}\gamma^t r(s_t,a_t)\right],
% \end{aligned}
% \end{equation}
% where $s_t\sim P(\cdot|s_{t-1},a_{t-1}, a_t\sim\pi(\cdot|s_t), s_0=s,a_0=a)$.
% In distributional RL, the probability law of $Z^{\pi}$ instead of $\mathbb{E}[Z^{\pi}]$ plays a center role. 
% It captures is the probability distribution of the random variable $Z^{\pi}$, denoted as $Z$. 
% The original Bellman operator is adapted into the distributional version:
% \begin{equation}
% \begin{aligned}
%     \mathcal{T}^{\pi} Z(s,a) &:\overset{D}{=} r(s,a) + \gamma Z(s',a') \\
%     s' &\sim P(\cdot|s,a),a'\sim \pi(\cdot|s') 
% \end{aligned}
% \end{equation}
% where $\mathcal{T}$ is the distributional Bellman operator, and $A :\overset{D}{=} B$ denotes that $A$ equals $B$ by probability laws. $r$ is a reward function as a random vector, which records the randomness of the rewards. 

\section{Ambiguity on Language-Conditioned Tasks}
\label{sec:toy}

In practical tasks, language instructions have high flexibility.
For example, similar tasks may employ divergent expressions, while distinct tasks share overlapping language instructions.
The variability induces exponential growth of the task space.
Therefore, when the number of language instructions increases, the agent is required to accurately identify the tasks; otherwise, it will significantly affect the agent's performance.
We name this issue the ambiguity in language-conditioned tasks. We give a formal definition of semantics distinction from instructions as follows.

\begin{definition}[Semantics Instructions Distinction]
In a multi-task RL setting with known task instruction space $\mathcal{L}$ and unknown semantics space $\mathcal{G}$, for a task distinction threshold $\delta$ and sub-optimality gap $\epsilon$,  two task instructions \( l_i, l_j \in \mathcal{L}\) are considered with \textbf{different underlying semantics}, $g_i\not\leftrightarrow g_j$, if the expected Q-values under any shared $\epsilon$-optimal policy \( \pi \) satisfy:
\[
\mathbb{E}_{\pi}\left[|Q_{\pi}(s, a, l_i) - Q_{\pi}(s, a, l_j)|\right] \geq \delta.
\]
Conversely, if
\[
\mathbb{E}_{\pi}\left[|Q_{\pi}(s, a, l_i) - Q_{\pi}(s, a, l_j)|\right] \leq \delta/2,
\]
then \( l_i \) and \( l_j \) are considered with the \textbf{same underlying semantics}, $g_i\leftrightarrow g_j$, where $s_0\sim p_0(\cdot), a\sim \pi(\cdot|s),s'\sim p(\cdot|s,a)$. $V_{\pi}(s)\geq V^*(s)-\epsilon, \forall s \in \mathcal{S}$, $V^*$ is optimal value function.

% In the case of distributional reinforcement learning, the Wasserstein distance \( W_1 \) between the return distributions \( Z(s, a) \) is used as the criterion. Specifically, \( l_i \) and \( l_j \) are considered to represent \textbf{different semantics} if
% \[
% \mathbb{E}_{\pi}\left[W_1\left(Z_{\pi}(s, a, l_i), Z_{\pi}(s, a, l_j)\right)\right] \geq d,
% \]
% and the \textbf{same semantics} if this expectation is less than or equal \( d/2 \).

\label{def: distintion}
\end{definition}

\begin{figure*}[t]
    \centering
    \includegraphics[width=\linewidth]{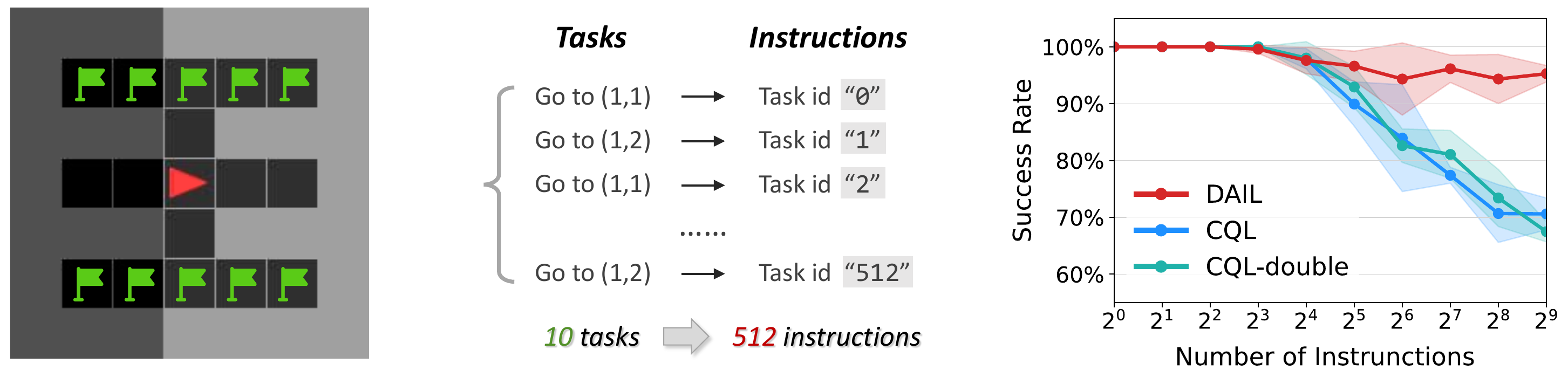}
    \vspace{-20pt}
    \caption{\textbf{Left:} The green flags in the map denote the accessible goals.
    \textbf{Middle:} An illustration of the mapping between goal positions and instructions.
    \textbf{Right:} Average success rates over 100 evaluations for each number of instructions and 3 seeds.}
    % We double the number of model parameters of CQL and name it CQL-double.
    % Both CQL methods are significantly affected by the number of tasks, while ours is much more stable.
    \label{fig:toy_res}
\end{figure*}

To make this point more straightforward, we introduce a toy experiment based on the Minigrid environment~\cite{MinigridMiniworld23}.
% Specifically, the map contains 10 accessible goal positions labeled as $g\in G$, where $G$ is the set of these goal positions. 
% The agent (red triangle-shaped) must follow the language instructions $l$ to reach one specific goal position.
% To show the flexibility of linguistic instructions, we create multiple different instructions $l$ for reaching each goal position.
% To simulate the flexibility of linguistic instructions, we create multiple different instructions $l$ for reaching each goal position. 
% (介绍Toy example怎么设计的，尤其是为什么能有这么多指令数量)
% Therefore, agent must correctly understand the instructions $l$ and accurately navigate to the right location.
The setup consists of 10 accessible goal positions $\mathcal{G}=\{g_0,g_1,...,g_{9}\}$, where $\mathcal{G}$ represents the set of all possible goal positions. 
The agent (red triangle-shaped) must follow a given instruction $l\in \mathcal{L}$ to navigate to a specific goal position (Left of Figure~\ref{fig:toy_res}). 
% The agent must correctly understand the instructions $l$ and accurately navigate to the right location.
We simulate instructions using numerical task IDs, making $\mathcal{L}\subset \mathbb{N}$. We employ a random mapping $F:\mathcal{L} \to \mathcal{G}$ to assign each $l$ to goal position $g$, which simulates semantics of instructions and is hidden from the agent (Middle of Figure~\ref{fig:toy_res}). With these settings, we can control the number of instructions $|\mathcal{L}|$ by simply adjusting the set of valid task IDs and the mappings. We conduct 10 experiments, varying the number of instructions from 1 to $2^9$. For each experiment, we generate a new mapping $F$ and collect 1024 random trajectories as the offline dataset. 

% XXX, making it challenging for the agent to correctly interpret the given instruction $l$ and navigate to the corresponding target location.
% To emulate the variability of natural language, one goal position can be associated with multiple distinct instructions $l$, making it challenging for the agent to correctly interpret the given instruction $l$ and navigate to the corresponding target location.

We evaluate the standard offline RL algorithm, CQL~\cite{kumar2020conservative}, on the above settings.
In addition, to test how model size affects the task ambiguity, we create an enhanced version called CQL-double by doubling its parameters.
The experimental results in Figure~\ref{fig:toy_res} show that if the number of instructions is small (e.g., under 16), agents can understand instructions and reach the goal position with high success.
However, as instructions multiply, the number of demonstrations for each instruction declines, which increases task ambiguity and consequently makes CQL and CQL-double performance drop significantly.
These observations reveal fundamental limitations in language-conditioned RL regarding task ambiguity, highlighting that model scaling alone remains insufficient to resolve this challenge.
For detailed information on the toy experiment, please refer to Appendix~\ref{appendix:implement}.

\section{Method}
\label{sec:method}
To address the issue of task ambiguity, we hope to improve the algorithm's ability to distinguish tasks.
In this section, we propose the Distributional Aligned Learning algorithm~(DAIL), which enhances the task differentiability from policy and representation.
Specifically, DAIL adopts the distributional language-guided policy to estimate the value distribution, preserving more information to aid task discrimination.
On the other hand, DAIL uses the trajectory-wise semantic alignment module to extract task representations and help discriminate different instructions by maximizing the mutual information between trajectories and instructions.
We show the overall framework of DAIL in Figure~\ref{fig:arch} and Algorithm~\ref{alg:main} in Appendix~\ref{appendix:pseudocode}.

\subsection{Distributional Language-Guided Policy}
\label{sec: distributional}
Current RL methods aim to estimate the expectation of the cumulative discounted reward.
However, as shown in Section~\ref{sec:toy}, when the number of instructions increases while keeping the number of actually-distinct tasks constant, the estimated expectation for different task instructions becomes similar. 
As a result, it is difficult for the agent to complete the task instructions accurately.
Differently, the distributional technique addresses this issue by calculating the distribution of the cumulative discounted reward, which is more distinguishable than expectation, as illustrated through an extreme yet straightforward example in Figure~\ref{fig:insight}.
\begin{figure*}[t]
    \centering
    \includegraphics[width=1\textwidth]{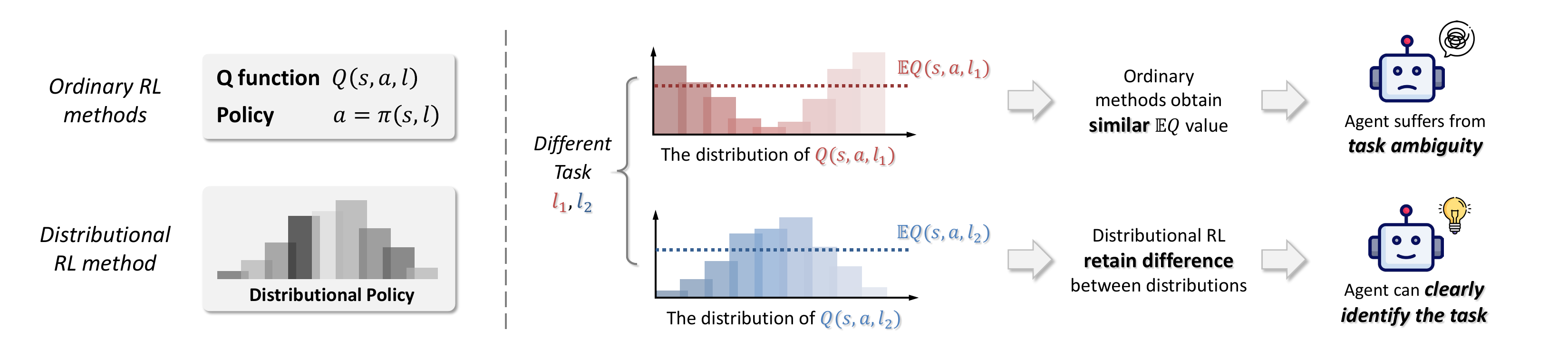}
    \vspace{-10pt}
    \caption{An illustrative case where $Q(s,a)$ functions in two tasks share the same expectations but have different distributions. In this case, traditional RL cannot discriminate between the tasks while distributional RL can. Please refer to Appendix~\ref{appendix:theory} for more details.}
    \centering
    \vspace{-5pt}
    \label{fig:insight}
\end{figure*}
For a fixed policy $\pi$, let the random variable $Z^{\pi}$ represent the cumulative discounted reward obtained along the policy $\pi$, and it has the following  relationships:
\begin{equation}
\begin{aligned}
    &V^{\pi}(s, l) :=\mathbb{E}[Z^{\pi}(s, l)] = \mathbb{E}\left[\sum_{t=0}^{\infty}\gamma^t r(s_t,a_t, l)|s_0=s, \pi \right] \\
    &Q^{\pi}(s,a, l) :=\mathbb{E}[Z^{\pi}(s,a, l)] = \mathbb{E}\left[\sum_{t=0}^{\infty}\gamma^t r(s_t,a_t, l) | s_0=s, a_0=a, \pi\right],
\end{aligned}
\end{equation}
Let $\mathcal{Z}$ denote the space of value distributions. In the following, for simplicity of notation, we denote $Z^\pi\in \mathcal{Z}$ as $Z$.
Instead of estimating the expectation, we calculate the probability distribution of the random variable $Z$:
\begin{equation}
% \begin{aligned}
    \mathcal{T}^{\pi} Z(s,a,l) :\overset{D}{=} R(s,a,l) + \gamma Z(s',a',l),
% \end{aligned}
\label{eq: c51}
\end{equation}
where $\mathcal{T}$ is the distributional Bellman operator, and $A :\overset{D}{=} B$ denotes that $A$ equals $B$ by probability laws. $R\in\mathcal{Z}$ is a function depicting the reward distributions.
Since modeling continuous distributions is challenging, we can discretize the value function distribution $Z$ and train it by minimizing the cross-entropy loss:
\begin{equation}
\begin{aligned}
    \mathcal{L}_{\text{Dist}}(\theta) = \mathbb{E}_{(\tau,l)\sim p_{\mathcal{D}}(\cdot, \cdot)} \frac{1}{T}\sum_{t=0}^{T-1}D_{\text{KL}}\left(\Phi \hat{\mathcal{T}} Z_{\hat{\theta}}(s_t,a_t,l)\|Z_{\theta}(s_t,a_t,l)\right)
\end{aligned},
\label{eq: dist}
\end{equation}
where $T$ is the trajectory length, $\Phi\hat{\mathcal{T}}$ is the sample Bellman update, which projects the value function distribution onto the parametric discrete distribution.
$Z_{\theta}$ and $Z_{\hat{\theta}}$ are estimated value distributions parameterized by $\theta$ and $\hat{\theta}$, respectively.
By optimizing Equation~\ref{eq: dist}, we obtain the distribution $Z_{\theta}$ that exhibits strong discriminative power across different instructions $l$. Please refer to Appendix~\ref{appendix:distrl} for the details.
% For detailed information on the optimization of Equation~\ref{eq: dist}, please refer to Appendix~\ref{appendix:distrl}.
% 在附录里介绍distributional RL

In addition, we conduct the following theoretical analysis.
Let $n_{\text{value}},n_{\text{dist}}$ denote the number of samples needed to avoid task ambiguity for value-based and distributional settings, respectively.
As shown in Theorem~\ref{thm: sample-complexity} and Corollary~\ref {cor: sample-complexity}, distributional RL achieves better sample efficiency compared with estimating the expectation when the number of tasks is sufficiently large, $n_{\text{value}}\geq n_{\text{dist}}$. Proofs and further details can be found in Appendix \ref{appendix:theory}.

% \begin{theorem}[Sample Complexity for Task Instruction Disambiguation]
% Consider an offline multi-task RL setting with \( M \) distinct tasks (with different semantics). In direct Q-value estimate setting, suppose 
% % and task distinction threshold \( \delta > 0 \). In the value-based setting where \( Q(s, a,l) \in [0, Q_{\max}],~ \forall (s,a,l) \sim \mathcal{D}\), task-level disambiguation is achieved with confidence \( 1 - \eta \) if
% \[
% n_{\text{value}} \geq \frac{C_{\text{value}} \log(M^2 / 2\eta)}{\delta^2}.
% \]
% In the distributional setting, let \( Z(s, a, l) \) denote the learned return distribution and the distinction threshold be a 1-Wasserstein distance \( d > 0 \). Then, it suffices that
% \[
% n_{\text{dist}} \geq \frac{C_{\text{dist}} \log(M^2 / \eta)}{d^2},
% \]
% for some universal constants \( C_{\text{value}},C_{\text{dist}}>0 \) depending on certain attributes of Q-value distribution.
% \label{thm: sample-complexity}
% \end{theorem}

\begin{theorem}[Sample Complexity for Task Instruction Disambiguation]
Consider an offline multi-task RL setting with \( M \) distinct tasks~(with different semantics). In direct Q-value estimate setting, suppose \( Q(s, a, l) \in [0, Q_{\max}],~ \forall (s,a,l) \sim \mathcal{D}\) with finite $Q_{\max}$, and the task distinction threshold is \( \delta > 0 \). When the number of training samples \( n_{\textup{value}} \) satisfies:
\[
n_{\textup{value}} \geq \frac{C_{\textup{value}} \log(3M^2 / \eta)}{\delta^2}.
\]
The mean value estimate algorithm achieves task-level semantic disambiguation with confidence at least \( 1 - \eta \).
In the distributional RL setting, let \( Z(s, a,l) \) denote the learned return distribution, and suppose the task distinction threshold is given by a 1-Wasserstein distance \( d > 0 \). Then, to ensure semantic disambiguation of task instructions with confidence at least \( 1 - \eta \), it suffices that:
\[
n_{\textup{dist}} \geq \frac{C_{\textup{dist}} \log(3M^2 / \eta)}{d^2},
\]
where \( C_{\textup{value}}, C_{\textup{dist}} >0 \) are universal constants depending on certain attributes of Q-value distribution.

\label{thm: sample-complexity}
\end{theorem}

\subsection{Trajectory-Wise Semantic Alignment}
\label{sec: align}

In this subsection, we focus on learning trajectory embedding that enhances the correspondence between instructions and trajectories, thereby reducing task ambiguity at the representational level.
To achieve this, we attempt to maximize the mutual information between the language instructions and trajectory:
\begin{equation}
\begin{aligned}
    w = \underset{w}{\mathrm{max}}~I(X_{\tau}(w); X_{l}(w)),
\end{aligned}
\label{eq: mutual}
\end{equation}
where $w$ is the parameter of the representation module, $X_{\tau}, X_l$ are the random variables of trajectory embedding and language instruction, respectively.
For the trajectory embedding $x_\tau$, we adopt the sequence model.
Specifically, we first use $u_{w}(\cdot, \cdot)$ to encode the state-action pairs, and then pass the sequence of embeddings through a sequence model $h_w(\cdot)$:
\begin{equation}
\begin{aligned}
    x_{\tau} = h_{w}(u_w(s_1,a_1), u_w(s_2,a_2), ..., u_w(s_T,a_T)).
\end{aligned}
\label{eq: x_t}
\end{equation}
As for the language instruction representation, we employ a language encoder to tokenize and encode the instructions into language embeddings $x_l$. 
Please refer to Appendix~\ref{appendix: architecture} for the detailed model architecture.
Let $f_w(\tau, l)=\frac{x_\tau^T x_l}{||x_\tau||~||x_l||}$ measure the similarity between the trajectory embedding and language instruction representation.
Since minimizing InfoNCE is equivalent to maximizing the lower bound of the mutual information~\cite{oord2018representation}, we can maximize the mutual information in Equation~\ref{eq: mutual} by minimizing the following NCE loss:
\begin{equation}
\begin{aligned}
    \mathcal{L}_c(w)=\mathbb{E}_{\underset{l^-\sim p_l(\cdot)}{(\tau, l^+)\sim p_{\mathcal{D}}(\cdot,\cdot)}}[&\mathrm{log}\sigma(f_w(\tau, l^+))+\mathrm{log}(1-\sigma(f_w(\tau, l^-)))],
\end{aligned}
\label{eq: align loss}
\end{equation}
where $l^+$ denotes the positive samples, which are sampled from the distribution of trajectory-instruction pair $p_{\mathcal{D}}(\cdot, \cdot)$.
$l^-$ denotes the negative samples, generated by uniform sampling over the language instructions from the offline datasets.  % 怎么得到的负样本

\subsection{Practical Implementation}
\label{sec: implementation}
To address the partial observability issue in practical implementation, the trajectory encoding $x_t$ defined in Equation~\ref{eq: x_t} is incorporated as an additional input $(s_t,a_t,x_{t-1},l)$.
The Equation~\ref{eq: c51} is transformed into correspondingly:
\begin{equation}
\begin{aligned}
    \mathcal{T}^{\pi} Z(s_t,a_t,x_{t-1},l) :\overset{D}{=} R(s_t,a_t,l) + \gamma Z(s_{t+1},a_{t+1},x_{t},l)
\end{aligned}.
\end{equation}
In practice, we estimate the value distribution $Z_{\theta}$ with discrete distribution, using a set of atoms $\{z_i=V_{\text{MIN}}+i\Delta z\}_{i=1}^{M-1},\Delta z=\frac{V_{\text{MAX}}-V_{\text{MIN}}}{M-1}$. $V_{\text{MIN}}, V_{\text{MAX}} \in \mathbb{R}$ are the lower and upper bounds of the distributions with support, respectively, and $M \in \mathbb{N}$ is the number of atoms. Then the discrete value distribution is modeled as:
\begin{equation}
Z_{\theta}(s_t,a_t,x_{t-1},l)=z_i\text{, with probability}~p_i(s_t,a_t,x_{t-1},l):=\frac{e^{\theta_i(s_t,a_t,x_{t-1},l)}}{\sum_j e^{\theta_j(s_t,a_t,x_{t-1},l)}}
\label{eq:distribution}
\end{equation}
where $\theta_i:\mathcal{S}\times{\mathcal{A}}\times\mathcal{X}\times\mathcal{L}\to \mathbb{R}$ is a parametric model employed to approximate $Z_{\theta}$ and updated by Equation~\ref{eq: dist}.
Based on the analysis in Section~\ref{sec: distributional}, we compute $Q$-function with distributional mechanism by $Q_{\theta}(s_t,a_t,x_{t-1},l)=\mathbb{E}[Z_{\theta}(s_t,a_t,x_{t-1},l)]=\sum_{i=0}^{M-1}p_{i}(s_t,a_t,x_{t-1},l)z_i$.
Please refer to Appendix~\ref{appendix:distrl} for the details. % 指向Distributional的附录章节

In addition, we consider the offline learning setting in this work, which learns a policy without interacting with the environment.
For this reason, we adopt the standard offline learning term, $\text{CQL}(\mathcal{H})$~\cite{kumar2020conservative}, to address the distribution shift issue in offline RL learning~\cite{fujimoto2019off}:
\begin{equation}
\begin{aligned}
    &\mathcal{L}_{\text{CQL}}(\theta) = \mathbb{E}_{(\tau,l)\sim p_{\mathcal{D}}(\cdot, \cdot)} \frac{1}{T}\sum_{t=0}^{T-1} \mathrm{log}\sum_a \mathrm{exp}(Q_{\theta}(s_t,a,x_{t-1},l))- \mathbb{E}_{a\sim\pi_{\beta}(a|s)}[Q_{\theta}(s_t,a,x_{t-1},l)],
\end{aligned}
\label{eq: cql}
\end{equation}
Combining all the above loss functions, the total loss function is:
\begin{equation}
\begin{aligned}
    \mathcal{L}_{\text{tot}} =  \mathcal{L}_{\text{Dist}} + \lambda \mathcal{L}_c + \alpha \mathcal{L}_{\text{CQL}}
\end{aligned}
\label{eq: total loss}
\end{equation}
where $\lambda, \alpha$ are weights of the trajectory-wise semantic alignment module and the offline learning term, respectively.
Finally, we select the action with the highest $Q$-value:
\begin{equation}
\begin{aligned}
    a^*_t = \underset{a}{\mathrm{argmax}} Q_{\theta}(s_t,a,x_{t-1},l)
\end{aligned}
\end{equation}
The complete process of our method is shown in Algorithm~\ref{alg:main} in Appendix~\ref{appendix:pseudocode} and Figure~\ref{fig:neuralnetwork} in Appendix~\ref{appendix: architecture}.
\begin{figure*}[t]
    \centering
    \includegraphics[width=1\textwidth]{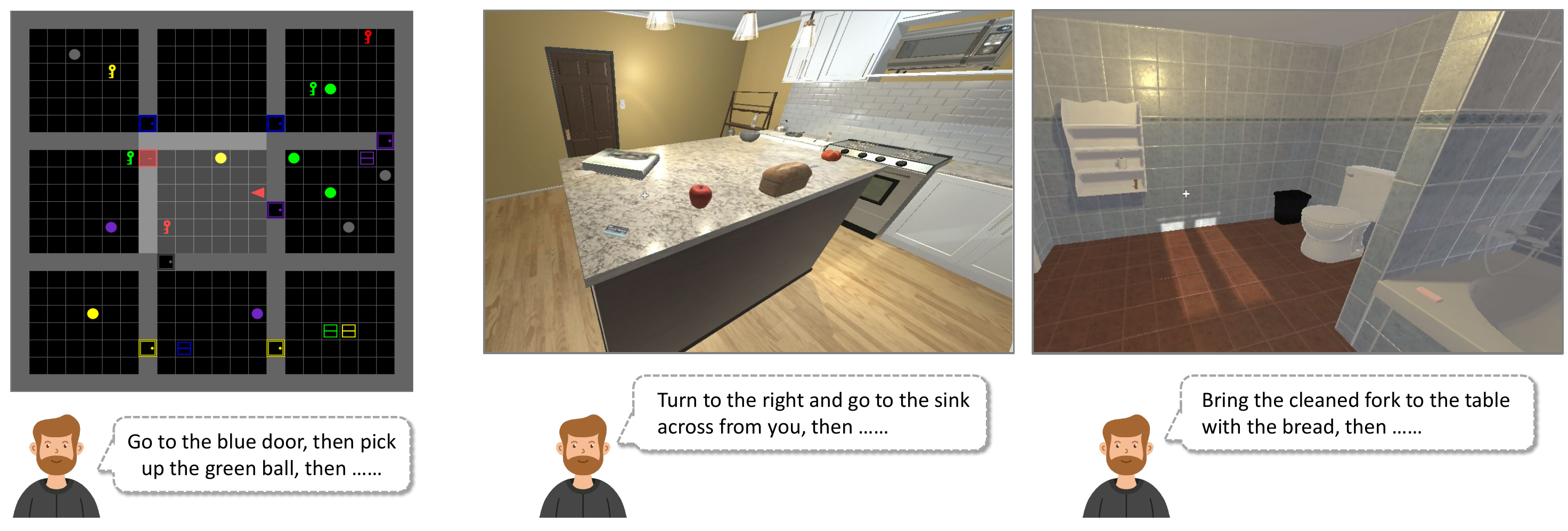}
    \vspace{-10pt}
    \caption{\textbf{Left:} An example of the \texttt{SynthLoc} task in the BabyAI environment: ``put the yellow key next to a green ball''. \textbf{Right:} Two example scenes with instructions from ALFRED. Agents are asked to finish specific tasks in the 3D household environment according to received instructions.}
    \centering
    \vspace{-5pt}
    \label{fig:environment}
\end{figure*}

\section{Experiments}
\label{sec:exp}
We designed our experiments to answer the following questions: 
\textit{Q1}: How does DAIL compare to other state-of-the-art methods on offline language-conditioned tasks? 
\textit{Q2}: How does DAIL perform as the number of tasks explodes? 
\textit{Q3}: Can DAIL learn a meaningful alignment between trajectories and instructions?
\textit{Q4}: What is the contribution of each of the proposed techniques in DAIL?

\subsection{Experimental Setting}
We evaluate various methods in language-conditioned tasks, with detailed introductions as follows:

\paragraph{BabyAI}~\cite{chevalier2018babyai} is a language learning research platform with different levels of tasks, shown on the Left of Figure~\ref{fig:environment}. 
We choose level $\verb|SynthLoc|$ for evaluation, which contains four major groups of tasks $\verb|Open|$, $\verb|Goto|$, $\verb|PickUp|$, and $\verb|PutNext|$.
The agent is asked to operate with an assigned object, like ``open a red door'' or ``put the gray box in front of you next to the blue key''. 
Tasks vary as the colors, types, or locations of objectives change, making around 6000 different tasks in total.
To evaluate the algorithms' generalizations, we divide the task space into in-distribution tasks and out-of-distribution tasks.
In-distribution tasks account for approximately 60\% of the total number of tasks, with a total of 3325. 
For the offline learning, we construct an offline dataset with 50k expert trajectories, 50k imitation learning agent trajectories, and 25k random trajectories.

% alfred 包含什么样的任务类型，选择了什么样的任务，为什么选择，具体怎么复杂，要求智能体干啥事
\paragraph{ALFRED}~\cite{shridhar2020alfred} benchmarks sequential decision-making tasks involving household activities (e.g, cleaning, heating food) through language instructions and first-person vision, shown on the Right of Figure~\ref{fig:environment}. 
The dataset provides 8055 expert demonstrations with 25k human-annotated language instructions detailing both high-level goals and sub-goal step-by-step guidance. 
As our work primarily focuses on low-level policy learning rather than high-level planning, we specifically concentrate on the $\verb|GOTO|$ sub-goal setting for our evaluation.
In this task set, the agent must go to specific locations according to instructions like ``Move to other side of couch on the right side of the table before the door''. To simulate the presence of noisy data in real-world applications, we augment the training set with 30k random-agent trajectories, resulting in 97896 total trajectories with 53442 unique instructions across 108 household scenes.

\paragraph{Baselines}
We choose three state-of-the-art offline RL algorithms, CQL~\cite{kumar2020conservative}, IQL~\cite{kostrikov2021offline}, and adapt them into a language-conditioned manner as our RL baselines.
% , and DT~\cite{chen2021decision}
In addition, for imitation baselines, we include GCBC~\cite{ding2019goal}, BC-Z~\cite{jang2022bc} and GRIF~\cite{myers2023goal}, which are also adapted into a language-conditioned manner by the benchmarks~\cite{chevalier2018babyai, shridhar2020alfred}. 
Further details about baselines are shown in Appendix~\ref{appendix:implement}.

\begin{figure}[t]
    \centering
    \includegraphics[width=\textwidth]{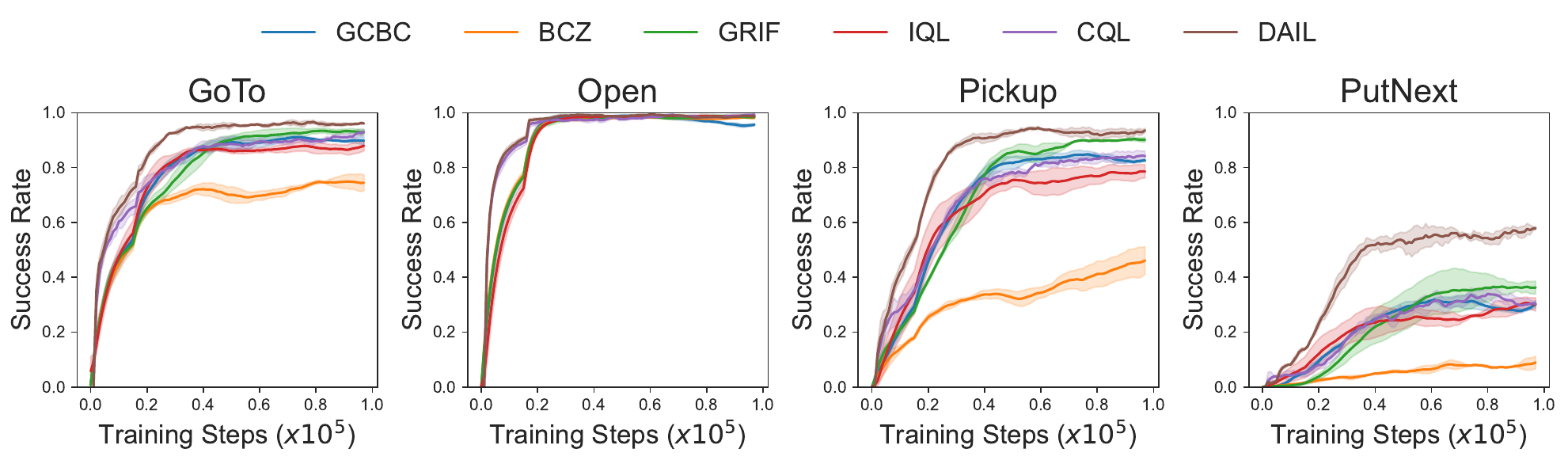} 
    \vspace{-10pt}
    \caption{Training curves on in-distribution BabyAI tasks. Success rates are evaluated over 3 seeds. }
    % The results are smoothed by applying an exponential moving average, with $span=5$.}
    % Due to the different update methods used by DT, only its final performance is reported.}
    \centering
    \label{fig:babyai_curve_two}
\end{figure}
\begin{table*}[t]
    \caption{Success rate of out-of-distribution BabyAI tasks. Each score is evaluated over 3 seeds.}
    \label{table: babyai}
\begin{center}
    \begin{tabular}{l|c|c|c|c|c|c} 
    \toprule
      Tasks & GCBC & BC-Z & GRIF & IQL & CQL & DAIL (ours) \\
      \midrule
      Open & 94.4\textpm2.5 & 93.7\textpm1.0 & 95.9\textpm1.7 & 98.0\textpm0.4 & 98.8\textpm0.5 & \textbf{99.0\textpm0.2}\\
      Goto & 90.3\textpm1.6 & 76.9\textpm3.0 & 88.8\textpm2.6 & 86.1\textpm1.2 & 88.9\textpm2.1 & \textbf{91.3\textpm1.0} \\
      PickUp & 78.4\textpm2.1 & 45.4\textpm1.5 & 75.6\textpm3.9 & 70.4\textpm3.6 & 71.9\textpm2.2 & \textbf{87.6\textpm2.0}\\
      PutNext  & 27.4\textpm1.6& 11.2\textpm3.3 & 22.5\textpm2.7 & 21.4\textpm3.1 & 27.6\textpm0.8 & \textbf{49.1\textpm1.8}\\
      \midrule
      All & 74.1\textpm0.7 & 57.9\textpm1.8  & 71.2\textpm2.6 & 69.7\textpm2.3 & 72.6\textpm0.4 & \textbf{81.7\textpm1.3}\\ 
      \bottomrule
    \end{tabular}
\end{center}
\end{table*}

\subsection{Main Results}
\paragraph{BabyAI experimental results.}
We provide training curves of in-distribution tasks in Figure~\ref{fig:babyai_curve_two} and out-of-distribution experimental results in Table~\ref{table: babyai}. The complete results are shown in Table \ref{table:babyai-high} in Appendix~\ref{appendix:add-res}. It shows that our method achieves superior performance compared with other baselines, especially in the $\verb|PutNext|$ task category. Vanilla offline RL algorithms like CQL and IQL underperform compared to imitation learning methods like GCBC, which we attribute to the adverse impact of task ambiguity on RL-based approaches as discussed in Theorem~\ref{thm: sample-complexity}.
On the other hand, modified algorithms designed for language-conditioned IL (BC-Z and GRIF) perform poorly under our setting. This is primarily because their contrastive learning objectives are not robust in the presence of noisy or suboptimal data. In contrast, our alignment-based approach, built on an offline RL framework, maintains strong performance. We further evaluate various algorithms on a more challenging dataset with fewer expert trajectories~(Table~\ref{table:babyai-medium} in Appendix~\ref{appendix:add-res}).
% Due to the lower proportion of successful trajectories, learning in this dataset is more challenging. As a result, all methods show a significant decline in performance. However, 
Our method still achieves the optimal results. Further details about the experiment are shown in Appendix~\ref{appendix:implement}.

\paragraph{ALFRED experimental results.}
The complexity and variety of instructions in ALFRED challenge the agent's ability to discern instructions. 
The experimental results in Table~\ref{table:alfred_res} show that our approach demonstrates the highest success rate~(SR), validating its positive impact on instruction recognition capability compared to other baselines. 
GCBC exhibits poorer resistance to suboptimal data, resulting in performance comparable to other RL baseline models. 
We also report path-length weighted success rates~(PLW SR), which considers the length of expert demonstration and demonstrates the effectiveness of the behavior policy following~\cite{shridhar2020alfred}. We further illustrate the observation trajectories generated by DAIL under the validation set instructions of ALFRED with varied instructions and scenes in Figure~\ref{fig:alfred_vis}. This visualization demonstrates DAIL's robust task execution in complex scenes under diverse instructions, with additional trajectory demonstrations provided in Appendix~\ref{appendix:alfred_res}.
\begin{figure}[t]
    \centering
    \includegraphics[width=\textwidth]{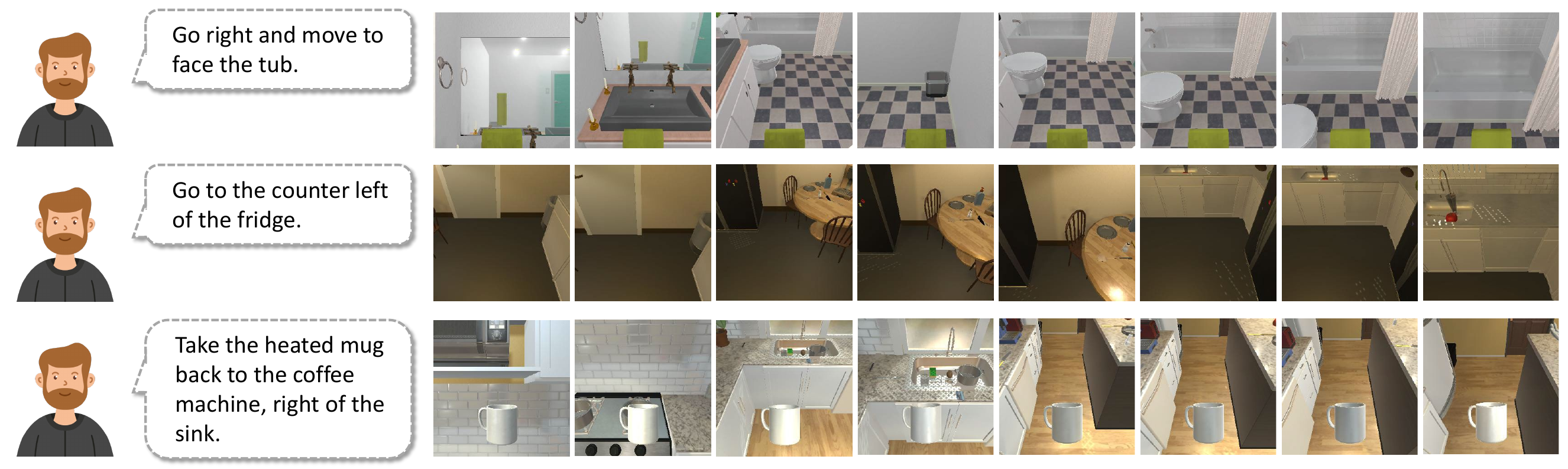} 
    \vspace{-10pt}
    \caption{Example trajectories of DAIL in ALFRED validation set instructions with varied instructions and scenes.}
    \centering
    \label{fig:alfred_vis}
\end{figure}
% \begin{table}[t]
%     \begin{center}
%         \begin{tabular}{c|c|c|c|c}
%         \toprule
%             \multirow{2}{*}{Algorithm} & \multicolumn{2}{c|}{Training Set} & \multicolumn{2}{c}{Validation Set (Seen)}\\
%             \cmidrule(lr){2-3} \cmidrule(lr){4-5}
%             & SR & PLW SR & SR & PLW SR \\
%             \midrule
%             GCBC  & 87.9\textpm2.4 & 84.3\textpm2.6 & 47.1\textpm2.4 & 40.5\textpm3.1 \\
%             BC-Z  & 86.5\textpm0.8  & 82.3\textpm2.7  & 43.0\textpm2.5  &  39.5\textpm2.2\\
%             GRIF  & 87.8\textpm1.6 & 82.3\textpm2.6 & 48.6\textpm1.0  & 43.2\textpm2.5 \\
%             IQL  & 88.4\textpm0.6 & 84.4\textpm2.7 & 52.0\textpm3.0 & 47.0\textpm3.2 \\
%             CQL  & 87.2\textpm1.2 & 83.0\textpm2.9 & 50.4\textpm1.7 & 44.4\textpm1.3 \\
%             DAIL (ours)  & \textbf{92.3}\textpm2.2 & \textbf{90.2}\textpm3.1 & \textbf{56.8}\textpm2.1 & \textbf{50.3}\textpm2.4 \\
%         \bottomrule
%         \end{tabular}
%         \caption{Success rate (SR) and path-length weighted success scores (PLW SR) in the ALFRED tasks. 
%         Each score is evaluated over 3 seeds.}
%     \label{table:alfred_res}
%     \end{center}
% \end{table}

\begin{table*}[t]
        \caption{Success rate (SR) and path-length weighted success scores (PLW SR) in the ALFRED tasks. The results are shown on the training set and validation set respectively.
        Each score is evaluated over 3 seeds.}
    \label{table:alfred_res}
    \begin{center}
        \begin{tabular}{c|c|c|c|c|c|c}
        \toprule
            Tasks & GCBC & BC-Z & GRIF & IQL & CQL & DAIL (ours) \\
            \midrule
            SR (Training) & 87.9\textpm2.4 & 86.5\textpm0.8 & 87.8\textpm1.6 & 88.4\textpm0.6 & 87.2\textpm1.2 & \textbf{92.3\textpm2.2}\\
            PLW SR (Training) & 84.3\textpm2.6  & 82.3\textpm2.7  & 82.3\textpm2.6  & 84.4\textpm2.7 & 83.0\textpm2.9 & \textbf{90.2\textpm3.1}\\
            SR (Validation) & 47.1\textpm2.4 & 43.0\textpm2.5 & 48.6\textpm1.0  & 52.0\textpm3.0 & 50.4\textpm1.7 & \textbf{56.8\textpm2.1}\\
            PLW SR (Validation) & 40.5\textpm3.1 & 39.5\textpm2.2 & 43.2\textpm2.5 & 47.0\textpm3.2 & 44.4\textpm1.3 & \textbf{50.3\textpm2.4}\\
        \bottomrule
        \end{tabular}

    \end{center}
\end{table*}
\begin{figure*}[t]
    \centering
    \includegraphics[width=0.95\textwidth]{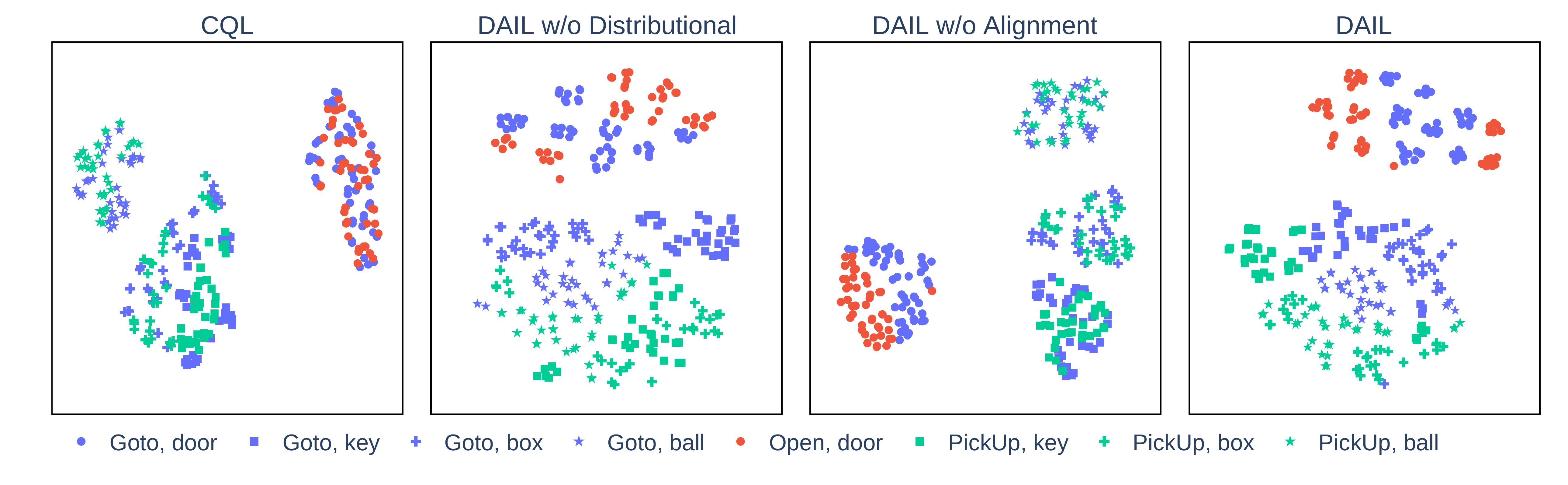}
    \vspace{-10pt}
    \caption{T-SNE visualization of instructions in BabyAI tasks. The figure distinguishes between different task categories~(e.g., Open) and target object types~(e.g., box), using marker colors and shapes to represent each separately.
    For example, ``pick up a red box'' corresponds to \squaresymbolp, ``go to a green box'' corresponds to \textcolor{myGreen}{$\mathrel{\textbf{+}}$}.}
    \centering
    \label{fig:visual}
\end{figure*}

\subsection{Visualization}
\label{subsec:vis}
To better understand how our proposed method enhances the performance, we show the t-SNE~\cite{van2008visualizing} results of the language instructions internal embedding on BabyAI $\verb|SynthLoc|$. Each point represents the internal representation of a unique language instruction within the policy network.
% We present the representations of several subtasks by distinguishing them based on the task types and target object types.
The experimental results in Figure \ref{fig:visual} show that vanilla RL can only marginally separate some broad task categories, but fails at distinguishing \textcolor{myGreen}{\texttt{PickUp}} and \textcolor{myBlue}{\texttt{Goto}}. Moreover, it is completely confused between \textcolor{myOrange}{\texttt{Open}~(door)} and \textcolor{myBlue}{\texttt{Goto}~(door)}.
% Vanilla RL can distinguish between several major task categories like \textcolor{myBlue}{\texttt{PutNext}}, but fails to differentiate task categories like \textcolor{myPurple}{\texttt{PickUp}} and \textcolor{myOrange}{\texttt{Goto}}. 
Alignment and distributional methods help discriminate tasks between and within categories. Our method substantially enhances task representation by clearly differentiating between task categories and target object types (for example, separating \circlesymbolb~and \circlesymbolo, \textcolor{myBlue}{$\star$}~and \textcolor{myGreen}{$\star$}). 

For more precise visualization, we use the same method to visualize the task representations of algorithms on the same task category, illustrated in Figure~\ref{fig:cql-pickupgoto} in Appendix~\ref{appendix:vis}. 
Our method effectively distinguishes all tasks in these two categories without overlapping confusion and group similar tasks into smaller clusters (\textcolor{red}{\rule{1ex}{1ex}}, \tikz\draw[fill=blue,draw=blue] (0,0) circle (0.5ex);, \textcolor{purple}{$\mathrel{\textbf{+}}$}, \textcolor{red}{$\star$} and so on). We conduct the visualization experiments under 3 training seeds, all of which yield consistent experimental conclusions. Moreover, we quantitatively measure the clustering quality of our proposed components with the Silhouette score\cite{rousseeuw1987silhouettes} using the learned language embeddings. The experimental results in Table \ref{table:clustering} in Appendix \ref{appendix:add-res} demonstrate that our method effectively enhances clustering performance, which aligns with the experimental findings from the visualization.
Details of the quantitative evaluation and more visualizations of task representations can be referred in Appendix \ref{appendix:add-res} and \ref{appendix:vis}, respectively.

\subsection{Ablation Studies}
\textbf{Ablation of components.} To study the contribution of each component in our learning framework, we conduct the following ablation study.
We compare the performance of algorithms that only apply trajectory-wise alignment or distributional language-guided policy alone with our method on \texttt{SynthLoc}. 
The experimental results in Table~\ref{table:ablation_main} show that both modules significantly improve the performance over vanilla CQL on in-distribution and out-of-distribution tasks.
Further, combining both components can achieve the best performance compared to other approaches. 

\begin{table*}[ht]
    \caption{Ablation results for components of our method. Each score is evaluated over 3 seeds.}
    \begin{center}
        \begin{tabular}{c|c|c|c|c}
        \toprule
        \multirow{2}{*}{Algorithm} & \multicolumn{2}{c|}{In Distribution} & \multicolumn{2}{c}{Out of Distribution}\\
        \cmidrule(lr){2-3} \cmidrule(lr){4-5}
            & PutNext & All & PutNext & All \\
        % \midrule
        %     GCBC-F  & 38.0 & 83.3 & 35.6 & 78.9\\
        %     \makecell{GCBC-F+Alignment} & 44.9 & 81.7 & 41.2 & 77.6 \\
        \midrule
        %     IQL  & 26.2\textpm3.5 & 75.2\textpm0.7 & 21.4\textpm3.1 & 69.7\textpm2.3 \\
        %     \makecell{IQL+Alignment} & 30.6\textpm5.8 & 79.0\textpm1.1 & 25.8\textpm2.8 & 70.6\textpm1.6 \\
        % \midrule
            CQL  & 25.6\textpm2.5 & 78.1\textpm1.6 & 27.6\textpm0.8 & 72.6\textpm0.4 \\
            {DAIL w/o Distributional} & 39.6\textpm0.6 & 83.3\textpm0.2 & 39.3\textpm1.7 & 77.3\textpm0.8\\
            {DAIL w/o Alignment} & 39.1\textpm1.0 & 82.2\textpm1.3 & 32.0\textpm0.9 & 75.1\textpm1.6\\
        \midrule
            DAIL  & \textbf{57.9\textpm0.9} & \textbf{89.2\textpm0.5} & \textbf{49.1\textpm1.8} & \textbf{81.7\textpm1.3}\\
        \bottomrule
        \end{tabular}
    \end{center}
    \label{table:ablation_main}
\end{table*}

\textbf{Ablation of alignment weight $\lambda$.} In Equation~\ref{eq: total loss}, $\lambda$ is the weight of the alignment loss.
For this reason, we evaluate the choice of $\lambda$ in BabyAI tasks with various $\lambda$.
The experimental results in Figure~\ref{fig:lambda_main} show that there is no noticeable performance difference between $\lambda=0.2$ and $\lambda=1$. The performance begins to degrade when the influence of the loss is either too small~($\lambda=0.01$) or too large~($\lambda=2$). Therefore, we recommend choosing a value between 0.2 and 1. 

\begin{figure}[ht]
    \centering
    \includegraphics[width=0.5\textwidth]{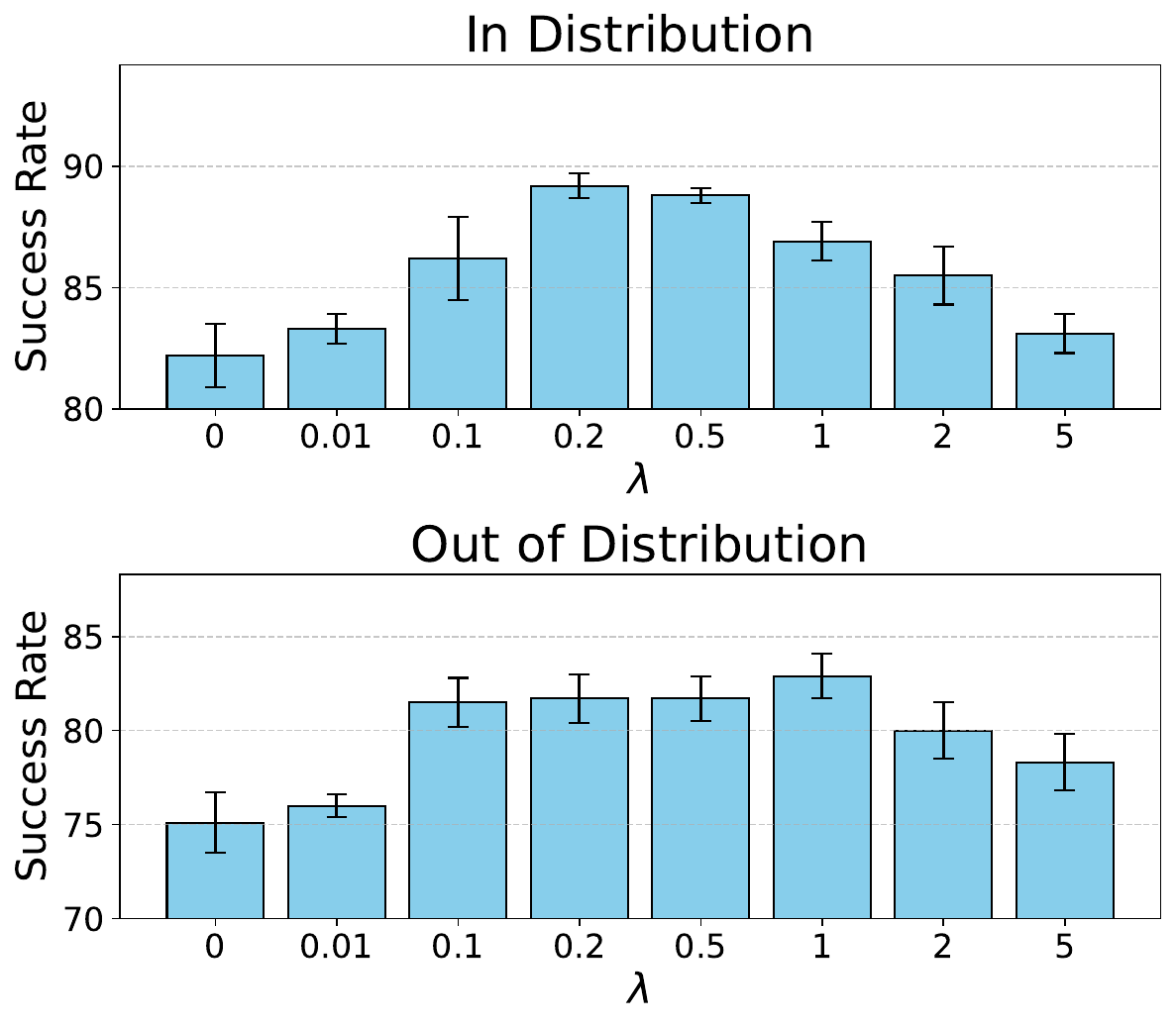} 
    \caption{Percent difference of the performance of an ablation over $\lambda$, compared to the average results of all $\lambda$s evaluated.}
    \label{fig:lambda_main}
\end{figure}

% The performance is evaluated by the success rate of all tasks as illustrated in Figure \ref{fig:lambda}. 
% We fix 0.2 as the default weight for other experiments in our work.
% We next ablate the choice of the TWS-A weight $\lambda$. 
% \begin{table}[h]
%     \begin{center}
%         \begin{tabular}{c|c|c|c|c|c|c|c|c}
%            \toprule
%            Weight  &  0 & 0.01 & 0.1 & 0.2 & 0.5 & 1 & 2 & 5\\
%            \midrule
%            Open    &  98.0  &  98.3    &  98.5   & 97.2    & 99.1 & 98.9 & 98.6 & 97.7\\
%            Goto    & 90.9 & 94.6 & 96.1 & 96.5 & 94.9 & 94.7 & 94.3 &  92.2\\
%            Pickup    & 89.0 & 87.5 & 91.1 & 94.9 & 90.9 & 91.0 & 90.3 & 89.2 \\
%            PutNext    & 39.1 & 37.4 & 44.9 & 57.9 & 52.3 & 53.7 & 49.4 & 43.4\\
%            All.    & 82.2 & 82.9 & 85.2 & 89.2 & 87.0 & 87.2 & 86.4 & 83.8\\
%            \midrule
%            Open    & 96.4   &  98.4    &  98.5   &  99.0   & 96.8 & 98.1 & 98.0 & 97.4\\
%            Goto    & 89.8 & 90.3 & 93.4 & 91.3 & 93.3 & 92.2 & 92.1 & 92.3 \\
%            Pickup    & 79.0 & 76.8 & 81.9 & 87.6 & 86.4 & 88.3 & 84.6 & 81.7 \\
%            PutNext    & 32.0 & 35.6 & 45.6 & 49.1 & 46.2 & 45.4 & 41.5 & 41.6 \\
%            All.    & 75.1 & 76.0 & 79.9 & 81.7 & 80.9 & 81.1 & 79.1 & 79.4 \\
%            \bottomrule
%         \end{tabular}
%         \caption{Success rate of our methods with varied $\lambda$. All values are percentages.}
%     \end{center}
% \end{table}
\section{Conclusion and Future Work}
\label{sec:conclusion}
In this work, we aim to address the task discrimination and comprehension challenges in language-conditioned RL. To achieve this, we propose a novel method called DAIL, which incorporates a distributional language-guided policy and a trajectory-wise semantic alignment module. 
We first theoretically demonstrate that distributional RL methods are more sample-efficient than traditional RL methods for learning in language-conditioned tasks. 
We then perform extensive experiments on both structured and visual observation benchmarks. The experimental results and visualization analysis show that our method learns language instruction representations with clearer semantics.
The simplicity and robustness of DAIL make it easily adaptable as a plug-in for other methods tackling language-conditioned problems.
While our method is theoretically and empirically validated, demonstrating its significant advantages in language-conditioned tasks, several limitations remain. First, due to experimental constraints, we are unable to test our method in real-world scenes to further validate the method’s effectiveness and robustness, which we will consider in future work. Second, our theoretical analysis of distributional RL's advantages relies on the assumptions of offline RL. Although we believe that this approach remains effective in online settings, we defer the theoretical analysis to future work.

\begin{ack}
This work is supported by Strategic Priority Research Program of the Chinese Academy of Sciences (No.XDA27040200).
\end{ack}

% \section*{Impact Statement}

% This paper represents research aimed at improving task discrimination ability and the overall performance of offline language-conditioned reinforcement learning. Although our work has many potential societal consequences, we believe that none must be specifically emphasized here.

\bibliographystyle{plain}
\bibliography{neurips_2025}

%%%%%%%%%%%%%%%%%%%%%%%%%%%%%%%%%%%%%%%%%%%%%%%%%%%%%%%%%%%%

\appendix

% \section{Technical Appendices and Supplementary Material}
% Technical appendices with additional results, figures, graphs and proofs may be submitted with the paper submission before the full submission deadline (see above), or as a separate PDF in the ZIP file below before the supplementary material deadline. There is no page limit for the technical appendices.

%%%%%%%%%%%%%%%%%%%%%%%%%%%%%%%%%%%%%%%%%%%%%%%%%%%%%%%%%%%%

\newpage
\appendix

\clearpage
\section*{NeurIPS Paper Checklist}

\begin{enumerate}

\item {\bf Claims}
    \item[] Question: Do the main claims made in the abstract and introduction accurately reflect the paper's contributions and scope?
    \item[] Answer: \answerYes{} % Replace by \answerYes{}, \answerNo{}, or \answerNA{}.
    \item[] Justification: The main claims in the abstract and introduction accurately summarize the paper's contributions, particularly in highlighting the proposed method's effectiveness in reducing task ambiguity and its robustness under noisy, offline settings.
    \item[] Guidelines:
    \begin{itemize}
        \item The answer NA means that the abstract and introduction do not include the claims made in the paper.
        \item The abstract and/or introduction should clearly state the claims made, including the contributions made in the paper and important assumptions and limitations. A No or NA answer to this question will not be perceived well by the reviewers. 
        \item The claims made should match theoretical and experimental results, and reflect how much the results can be expected to generalize to other settings. 
        \item It is fine to include aspirational goals as motivation as long as it is clear that these goals are not attained by the paper. 
    \end{itemize}

\item {\bf Limitations}
    \item[] Question: Does the paper discuss the limitations of the work performed by the authors?
    \item[] Answer: \answerYes{} % Replace by \answerYes{}, \answerNo{}, or \answerNA{}.
    \item[] Justification: We discuss the limitations of our method in Section~\ref{sec:conclusion}, including experimental settings and theoretical assumptions.
    \item[] Guidelines:
    \begin{itemize}
        \item The answer NA means that the paper has no limitation while the answer No means that the paper has limitations, but those are not discussed in the paper. 
        \item The authors are encouraged to create a separate "Limitations" section in their paper.
        \item The paper should point out any strong assumptions and how robust the results are to violations of these assumptions (e.g., independence assumptions, noiseless settings, model well-specification, asymptotic approximations only holding locally). The authors should reflect on how these assumptions might be violated in practice and what the implications would be.
        \item The authors should reflect on the scope of the claims made, e.g., if the approach was only tested on a few datasets or with a few runs. In general, empirical results often depend on implicit assumptions, which should be articulated.
        \item The authors should reflect on the factors that influence the performance of the approach. For example, a facial recognition algorithm may perform poorly when image resolution is low or images are taken in low lighting. Or a speech-to-text system might not be used reliably to provide closed captions for online lectures because it fails to handle technical jargon.
        \item The authors should discuss the computational efficiency of the proposed algorithms and how they scale with dataset size.
        \item If applicable, the authors should discuss possible limitations of their approach to address problems of privacy and fairness.
        \item While the authors might fear that complete honesty about limitations might be used by reviewers as grounds for rejection, a worse outcome might be that reviewers discover limitations that aren't acknowledged in the paper. The authors should use their best judgment and recognize that individual actions in favor of transparency play an important role in developing norms that preserve the integrity of the community. Reviewers will be specifically instructed to not penalize honesty concerning limitations.
    \end{itemize}

\item {\bf Theory assumptions and proofs}
    \item[] Question: For each theoretical result, does the paper provide the full set of assumptions and a complete (and correct) proof?
    \item[] Answer: \answerYes{} % Replace by \answerYes{}, \answerNo{}, or \answerNA{}.
    \item[] Justification: In the main text, we mainly conduct theoretical analysis in Section \ref{sec:toy} and \ref{sec: distributional}. The paper provides the full set of assumptions and a complete proof in Appendix \ref{appendix:theory}.
    \item[] Guidelines:
    \begin{itemize}
        \item The answer NA means that the paper does not include theoretical results. 
        \item All the theorems, formulas, and proofs in the paper should be numbered and cross-referenced.
        \item All assumptions should be clearly stated or referenced in the statement of any theorems.
        \item The proofs can either appear in the main paper or the supplemental material, but if they appear in the supplemental material, the authors are encouraged to provide a short proof sketch to provide intuition. 
        \item Inversely, any informal proof provided in the core of the paper should be complemented by formal proofs provided in appendix or supplemental material.
        \item Theorems and Lemmas that the proof relies upon should be properly referenced. 
    \end{itemize}

    \item {\bf Experimental result reproducibility}
    \item[] Question: Does the paper fully disclose all the information needed to reproduce the main experimental results of the paper to the extent that it affects the main claims and/or conclusions of the paper (regardless of whether the code and data are provided or not)?
    \item[] Answer: \answerYes{} % Replace by \answerYes{}, \answerNo{}, or \answerNA{}.
    \item[] Justification: We discuss the implementation of our method in Section \ref{sec: implementation} and details of experiments in Section \ref{sec:exp} and Appendix \ref{appendix:implement}.
    \item[] Guidelines:
    \begin{itemize}
        \item The answer NA means that the paper does not include experiments.
        \item If the paper includes experiments, a No answer to this question will not be perceived well by the reviewers: Making the paper reproducible is important, regardless of whether the code and data are provided or not.
        \item If the contribution is a dataset and/or model, the authors should describe the steps taken to make their results reproducible or verifiable. 
        \item Depending on the contribution, reproducibility can be accomplished in various ways. For example, if the contribution is a novel architecture, describing the architecture fully might suffice, or if the contribution is a specific model and empirical evaluation, it may be necessary to either make it possible for others to replicate the model with the same dataset, or provide access to the model. In general. releasing code and data is often one good way to accomplish this, but reproducibility can also be provided via detailed instructions for how to replicate the results, access to a hosted model (e.g., in the case of a large language model), releasing of a model checkpoint, or other means that are appropriate to the research performed.
        \item While NeurIPS does not require releasing code, the conference does require all submissions to provide some reasonable avenue for reproducibility, which may depend on the nature of the contribution. For example
        \begin{enumerate}
            \item If the contribution is primarily a new algorithm, the paper should make it clear how to reproduce that algorithm.
            \item If the contribution is primarily a new model architecture, the paper should describe the architecture clearly and fully.
            \item If the contribution is a new model (e.g., a large language model), then there should either be a way to access this model for reproducing the results or a way to reproduce the model (e.g., with an open-source dataset or instructions for how to construct the dataset).
            \item We recognize that reproducibility may be tricky in some cases, in which case authors are welcome to describe the particular way they provide for reproducibility. In the case of closed-source models, it may be that access to the model is limited in some way (e.g., to registered users), but it should be possible for other researchers to have some path to reproducing or verifying the results.
        \end{enumerate}
    \end{itemize}

\item {\bf Open access to data and code}
    \item[] Question: Does the paper provide open access to the data and code, with sufficient instructions to faithfully reproduce the main experimental results, as described in supplemental material?
    \item[] Answer: \answerYes{} % Replace by \answerYes{}, \answerNo{}, or \answerNA{}.
    \item[] Justification: We provide data and code in the supplemental material.
    \item[] Guidelines:
    \begin{itemize}
        \item The answer NA means that paper does not include experiments requiring code.
        \item Please see the NeurIPS code and data submission guidelines (\url{https://nips.cc/public/guides/CodeSubmissionPolicy}) for more details.
        \item While we encourage the release of code and data, we understand that this might not be possible, so “No” is an acceptable answer. Papers cannot be rejected simply for not including code, unless this is central to the contribution (e.g., for a new open-source benchmark).
        \item The instructions should contain the exact command and environment needed to run to reproduce the results. See the NeurIPS code and data submission guidelines (\url{https://nips.cc/public/guides/CodeSubmissionPolicy}) for more details.
        \item The authors should provide instructions on data access and preparation, including how to access the raw data, preprocessed data, intermediate data, and generated data, etc.
        \item The authors should provide scripts to reproduce all experimental results for the new proposed method and baselines. If only a subset of experiments are reproducible, they should state which ones are omitted from the script and why.
        \item At submission time, to preserve anonymity, the authors should release anonymized versions (if applicable).
        \item Providing as much information as possible in supplemental material (appended to the paper) is recommended, but including URLs to data and code is permitted.
    \end{itemize}

\item {\bf Experimental setting/details}
    \item[] Question: Does the paper specify all the training and test details (e.g., data splits, hyperparameters, how they were chosen, type of optimizer, etc.) necessary to understand the results?
    \item[] Answer: \answerYes{} % Replace by \answerYes{}, \answerNo{}, or \answerNA{}.
    \item[] Justification: We discuss the experimental details of experiments in Section \ref{sec:exp} and Appendix \ref{appendix:implement}. We provide the hyperparameters and network architectures in Appendix \ref{appendix: architecture}.
    \item[] Guidelines:
    \begin{itemize}
        \item The answer NA means that the paper does not include experiments.
        \item The experimental setting should be presented in the core of the paper to a level of detail that is necessary to appreciate the results and make sense of them.
        \item The full details can be provided either with the code, in appendix, or as supplemental material.
    \end{itemize}

\item {\bf Experiment statistical significance}
    \item[] Question: Does the paper report error bars suitably and correctly defined or other appropriate information about the statistical significance of the experiments?
    \item[] Answer: \answerYes{} % Replace by \answerYes{}, \answerNo{}, or \answerNA{}.
    \item[] Justification: We provide the standard deviation in all the experimental results in the paper.
    \item[] Guidelines:
    \begin{itemize}
        \item The answer NA means that the paper does not include experiments.
        \item The authors should answer "Yes" if the results are accompanied by error bars, confidence intervals, or statistical significance tests, at least for the experiments that support the main claims of the paper.
        \item The factors of variability that the error bars are capturing should be clearly stated (for example, train/test split, initialization, random drawing of some parameter, or overall run with given experimental conditions).
        \item The method for calculating the error bars should be explained (closed form formula, call to a library function, bootstrap, etc.)
        \item The assumptions made should be given (e.g., Normally distributed errors).
        \item It should be clear whether the error bar is the standard deviation or the standard error of the mean.
        \item It is OK to report 1-sigma error bars, but one should state it. The authors should preferably report a 2-sigma error bar than state that they have a 96\% CI, if the hypothesis of Normality of errors is not verified.
        \item For asymmetric distributions, the authors should be careful not to show in tables or figures symmetric error bars that would yield results that are out of range (e.g. negative error rates).
        \item If error bars are reported in tables or plots, The authors should explain in the text how they were calculated and reference the corresponding figures or tables in the text.
    \end{itemize}

\item {\bf Experiments compute resources}
    \item[] Question: For each experiment, does the paper provide sufficient information on the computer resources (type of compute workers, memory, time of execution) needed to reproduce the experiments?
    \item[] Answer: \answerYes{} % Replace by \answerYes{}, \answerNo{}, or \answerNA{}.
    \item[] Justification: The information or computer resources are provided in Appendix \ref{appendix: architecture}.
    \item[] Guidelines:
    \begin{itemize}
        \item The answer NA means that the paper does not include experiments.
        \item The paper should indicate the type of compute workers CPU or GPU, internal cluster, or cloud provider, including relevant memory and storage.
        \item The paper should provide the amount of compute required for each of the individual experimental runs as well as estimate the total compute. 
        \item The paper should disclose whether the full research project required more compute than the experiments reported in the paper (e.g., preliminary or failed experiments that didn't make it into the paper). 
    \end{itemize}
    
\item {\bf Code of ethics}
    \item[] Question: Does the research conducted in the paper conform, in every respect, with the NeurIPS Code of Ethics \url{https://neurips.cc/public/EthicsGuidelines}?
    \item[] Answer: \answerYes{} % Replace by \answerYes{}, \answerNo{}, or \answerNA{}.
    \item[] Justification: We confirm that our research conform with the NeurIPS Code of Ethics.
    \item[] Guidelines:
    \begin{itemize}
        \item The answer NA means that the authors have not reviewed the NeurIPS Code of Ethics.
        \item If the authors answer No, they should explain the special circumstances that require a deviation from the Code of Ethics.
        \item The authors should make sure to preserve anonymity (e.g., if there is a special consideration due to laws or regulations in their jurisdiction).
    \end{itemize}

\item {\bf Broader impacts}
    \item[] Question: Does the paper discuss both potential positive societal impacts and negative societal impacts of the work performed?
    \item[] Answer: \answerNA{} % Replace by \answerYes{}, \answerNo{}, or \answerNA{}.
    \item[] Justification: Our work focuses on addressing task ambiguity in offline reinforcement learning; this work does not present any foreseeable societal consequences.
    \item[] Guidelines:
    \begin{itemize}
        \item The answer NA means that there is no societal impact of the work performed.
        \item If the authors answer NA or No, they should explain why their work has no societal impact or why the paper does not address societal impact.
        \item Examples of negative societal impacts include potential malicious or unintended uses (e.g., disinformation, generating fake profiles, surveillance), fairness considerations (e.g., deployment of technologies that could make decisions that unfairly impact specific groups), privacy considerations, and security considerations.
        \item The conference expects that many papers will be foundational research and not tied to particular applications, let alone deployments. However, if there is a direct path to any negative applications, the authors should point it out. For example, it is legitimate to point out that an improvement in the quality of generative models could be used to generate deepfakes for disinformation. On the other hand, it is not needed to point out that a generic algorithm for optimizing neural networks could enable people to train models that generate Deepfakes faster.
        \item The authors should consider possible harms that could arise when the technology is being used as intended and functioning correctly, harms that could arise when the technology is being used as intended but gives incorrect results, and harms following from (intentional or unintentional) misuse of the technology.
        \item If there are negative societal impacts, the authors could also discuss possible mitigation strategies (e.g., gated release of models, providing defenses in addition to attacks, mechanisms for monitoring misuse, mechanisms to monitor how a system learns from feedback over time, improving the efficiency and accessibility of ML).
    \end{itemize}
    
\item {\bf Safeguards}
    \item[] Question: Does the paper describe safeguards that have been put in place for responsible release of data or models that have a high risk for misuse (e.g., pretrained language models, image generators, or scraped datasets)?
    \item[] Answer: \answerNA{} % Replace by \answerYes{}, \answerNo{}, or \answerNA{}.
    \item[] Justification: The paper poses no such risks.
    \item[] Guidelines:
    \begin{itemize}
        \item The answer NA means that the paper poses no such risks.
        \item Released models that have a high risk for misuse or dual-use should be released with necessary safeguards to allow for controlled use of the model, for example by requiring that users adhere to usage guidelines or restrictions to access the model or implementing safety filters. 
        \item Datasets that have been scraped from the Internet could pose safety risks. The authors should describe how they avoided releasing unsafe images.
        \item We recognize that providing effective safeguards is challenging, and many papers do not require this, but we encourage authors to take this into account and make a best faith effort.
    \end{itemize}

\item {\bf Licenses for existing assets}
    \item[] Question: Are the creators or original owners of assets (e.g., code, data, models), used in the paper, properly credited and are the license and terms of use explicitly mentioned and properly respected?
    \item[] Answer: \answerYes{} % Replace by \answerYes{}, \answerNo{}, or \answerNA{}.
    \item[] Justification: We properly credit the existing assets and respect the terms of use in our research.
    \item[] Guidelines:
    \begin{itemize}
        \item The answer NA means that the paper does not use existing assets.
        \item The authors should cite the original paper that produced the code package or dataset.
        \item The authors should state which version of the asset is used and, if possible, include a URL.
        \item The name of the license (e.g., CC-BY 4.0) should be included for each asset.
        \item For scraped data from a particular source (e.g., website), the copyright and terms of service of that source should be provided.
        \item If assets are released, the license, copyright information, and terms of use in the package should be provided. For popular datasets, \url{paperswithcode.com/datasets} has curated licenses for some datasets. Their licensing guide can help determine the license of a dataset.
        \item For existing datasets that are re-packaged, both the original license and the license of the derived asset (if it has changed) should be provided.
        \item If this information is not available online, the authors are encouraged to reach out to the asset's creators.
    \end{itemize}

\item {\bf New assets}
    \item[] Question: Are new assets introduced in the paper well documented and is the documentation provided alongside the assets?
    \item[] Answer: \answerYes{} % Replace by \answerYes{}, \answerNo{}, or \answerNA{}.
    \item[] Justification: The model and code we propose are well documented.
    \item[] Guidelines:
    \begin{itemize}
        \item The answer NA means that the paper does not release new assets.
        \item Researchers should communicate the details of the dataset/code/model as part of their submissions via structured templates. This includes details about training, license, limitations, etc. 
        \item The paper should discuss whether and how consent was obtained from people whose asset is used.
        \item At submission time, remember to anonymize your assets (if applicable). You can either create an anonymized URL or include an anonymized zip file.
    \end{itemize}

\item {\bf Crowdsourcing and research with human subjects}
    \item[] Question: For crowdsourcing experiments and research with human subjects, does the paper include the full text of instructions given to participants and screenshots, if applicable, as well as details about compensation (if any)? 
    \item[] Answer: \answerNA{} % Replace by \answerYes{}, \answerNo{}, or \answerNA{}.
    \item[] Justification: The paper does not involve crowdsourcing nor research with human subjects.
    \item[] Guidelines:
    \begin{itemize}
        \item The answer NA means that the paper does not involve crowdsourcing nor research with human subjects.
        \item Including this information in the supplemental material is fine, but if the main contribution of the paper involves human subjects, then as much detail as possible should be included in the main paper. 
        \item According to the NeurIPS Code of Ethics, workers involved in data collection, curation, or other labor should be paid at least the minimum wage in the country of the data collector. 
    \end{itemize}

\item {\bf Institutional review board (IRB) approvals or equivalent for research with human subjects}
    \item[] Question: Does the paper describe potential risks incurred by study participants, whether such risks were disclosed to the subjects, and whether Institutional Review Board (IRB) approvals (or an equivalent approval/review based on the requirements of your country or institution) were obtained?
    \item[] Answer: \answerNA{} % Replace by \answerYes{}, \answerNo{}, or \answerNA{}.
    \item[] Justification: The paper does not involve crowdsourcing nor research with human subjects.
    \item[] Guidelines:
    \begin{itemize}
        \item The answer NA means that the paper does not involve crowdsourcing nor research with human subjects.
        \item Depending on the country in which research is conducted, IRB approval (or equivalent) may be required for any human subjects research. If you obtained IRB approval, you should clearly state this in the paper. 
        \item We recognize that the procedures for this may vary significantly between institutions and locations, and we expect authors to adhere to the NeurIPS Code of Ethics and the guidelines for their institution. 
        \item For initial submissions, do not include any information that would break anonymity (if applicable), such as the institution conducting the review.
    \end{itemize}

\item {\bf Declaration of LLM usage}
    \item[] Question: Does the paper describe the usage of LLMs if it is an important, original, or non-standard component of the core methods in this research? Note that if the LLM is used only for writing, editing, or formatting purposes and does not impact the core methodology, scientific rigorousness, or originality of the research, declaration is not required.
    %this research? 
    \item[] Answer: \answerNA{} % Replace by \answerYes{}, \answerNo{}, or \answerNA{}.
    \item[] Justification: The core method development in this research does not involve LLMs as any important, original, or non-standard components.
    \item[] Guidelines:
    \begin{itemize}
        \item The answer NA means that the core method development in this research does not involve LLMs as any important, original, or non-standard components.
        \item Please refer to our LLM policy (\url{https://neurips.cc/Conferences/2025/LLM}) for what should or should not be described.
    \end{itemize}

\end{enumerate}

\clearpage

\section{Related Work}
\label{appendix:related}

\paragraph{Language-conditioned Agents.}
Two primary approaches for enabling an agent to follow human instructions are reinforcement learning (RL) and imitation learning (IL). 
Some approaches in language-conditioned RL focus on aligning language with policies or performing feature extraction to integrate language information into the learning process~\cite{chaplot2018gated, kaplan2017beating, andreas2017modular, misra2017mapping, shao2021concept2robot, bing2023meta}, while others prioritize reward shaping~\cite{squire2015grounding, fu2019language, goyal2019using, goyal2021pixl2r, shao2021concept2robot} to formulate language-conditioned reward functions.
However, most methods rely on simulators and are limited in scalability when applied in the offline setting. 
On the other hand, IL is designed to learn from large datasets but is heavily dependent on the quality of the data~\cite{hussein2017imitation}. To mitigate this dependency and enhance data efficiency, recent works have leveraged crowd-sourced annotations~\cite{nair2022learning}, multimodal alignment~\cite{jang2022bc, myers2023goal, nair2022r3m}, or carefully designed model architectures~\cite{mees2022matters, stepputtis2020language}. These methods enhance IL by supplying richer supervisory signals or building more robust model structures. 
Despite the success of IL methods in solving complex tasks, the scale and quality of the dataset remain significant constraints, particularly without external annotations or auxiliary datasets.

\paragraph{Offline Goal-conditioned RL.}
Learning a task-specific policy from demonstrations with different goals presents a significant challenge in offline goal-conditioned reinforcement learning (GCRL). By relabeling trajectories and treating intermediate states as additional goal states~\cite{andrychowicz2017hindsight, levy2017learning, li2020generalized, yang2022rethinking, yu2021conservative, ma2022offline}, offline GCRL has achieved notable improvements in sample efficiency. However, this approach is challenging to replicate in language-conditioned settings. First, replicating a specific goal state is particularly challenging in environments characterized by randomness or partial observability. Second, directly using goal states that lack semantic context as labels proves ineffective for learning. While some works use language instructions to guide planning~\cite{colas2020language, peng2023conceptual, pang2023natural}, the ambiguity of language can complicate this process. Several studies have applied hierarchical RL to guide policy through subgoals~\cite{levy2017learning, zhang2020generating, chane2021goal, park2024hiql, li2022hierarchical} and demonstrate strong performance. However, they rely heavily on a high-level policy that accurately decomposes language instructions into subgoals. 
\clearpage
\section{Algorithm}
\label{appendix:pseudocode}
\renewcommand{\algorithmiccomment}[1]{ // #1}

\begin{algorithm}[h]
   \caption{Distributional Aligned Learning}
   \label{alg:main}
    \begin{algorithmic}
    \REQUIRE Offline dataset $\mathcal{D}=\{(\tau,l)\}$, target network update frequency $K_{\rm update}$ and support atoms $Z_{\rm atoms}$. 
    \STATE Initialize policy parameters $\theta$ and target policy parameters $\hat{\theta}$.
    \FOR{each gradient step}
        \STATE Sample batch $\mathcal{B}=\{(\tau,l)\}_{i=1}^N\sim \mathcal{D}$
        \STATE Encode instructions: $\{x_l\}_{i=1}^N$
        \STATE Compute the history information $\{x_0,x_1,..,x_T\}_{i=1}^N$ using (\ref{eq: x_t})
        
        \STATE\COMMENT{\textbf{Distributional Language-Guide Policy}}
        \FOR{each transition $(s_t,a_t,r_t,s_{t+1},l)$ in batch $\mathcal{B}$}
            \STATE Compute estimated value distribution $Z_{\theta}(s_t,a_t,x_{t-1},l)$ using (\ref{eq: z_theta})
            \STATE Compute projected update $\Phi \hat{\mathcal{T}}Z_\theta(s,a,x,l)$ using (\ref{eq: Tz_theta})
            \STATE 
            $l_{\text{Dist}}(\theta)\leftarrow D_{KL}(\Phi\hat{\mathcal{T}}Z_{\hat{\theta}}(s_t,a_t,x_{t-1},l)||Z_\theta(s_t,a_t,x_{t-1},l))$
        \ENDFOR
    
        \STATE\COMMENT{\textbf{Trajectory-Wise Semantic Alignment}}
        \FOR{each trajectory-instruction pair $(\tau, l)$}
            \FOR{each pair $(\tau', l')$ other than $(\tau, l)$}
                \STATE View $l'$ as negative instruction $l^-$, and compute contrastive loss $l_c$ by using (\ref{eq: align loss})
            \ENDFOR
        \ENDFOR
        \STATE Add up the losses above to get $\mathcal{L}_{\rm Dist}\leftarrow\sum l_{\rm Dist}$, and $\mathcal{L}_{c}\leftarrow\sum_{(\tau,l)} l_{c}$
        \STATE Compute conservative Q-learning loss $\mathcal{L}_{\rm CQL}$ using (\ref{eq: cql})
        \STATE $\mathcal{L}_{\rm tot}\leftarrow\mathcal{L}_{\rm Dist}+\lambda\mathcal{L}_{c}+\alpha\mathcal{L}_{\rm CQL}$
        \STATE Update $\theta\leftarrow\theta-\eta\nabla_{\theta}\mathcal{L}_{\rm tot}(\theta)$
        \IF{step \% $K_{\rm update}$=0}
            \STATE Update target network: $\hat{\theta}\leftarrow\theta$
        \ENDIF
    \ENDFOR
    \STATE Extract policy: $\pi(s,x,l)\leftarrow \arg\max_a \sum_{i=1}^M p_i(s,a,x,l)z_i$
    \RETURN $\pi(s,x,l)$
\end{algorithmic}
\end{algorithm}

% 伪代码是把三个loss串起来，得到一些模块，然后输出是Pi

% \begin{algorithm}[H]
%     \caption{Categorical Loss~\cite{bellemare2017distributional}}
%     \label{algorithm:2}
%     \begin{algorithmic}[1]
%         \STATE {\bfseries Input:} $Q_{\theta}(s_{t+1},a^*,x_{t},l^+)$, $Z_{\theta}(s_t,a,x_{t-1},l^+)$, $r_t$, $\gamma$
%         \STATE Given hyper-parameters $V_{\text{MIN}}, V_{\text{MAX}}, M$
%         \STATE $m_i\leftarrow 0,~i\in 0,...,M-1$
%         \FOR{$j = 0$ to $M-1$}
%             \STATE $\hat{\mathcal{T}}z_j\leftarrow[r_t+\gamma z_i]_{V_{\text{MIN}}}^{V_{\text{MAX}}}$
%             \STATE $b_i\leftarrow(\hat{\mathcal{T}}z_j-V_{\text{MIN}})/\Delta z$
%             \STATE $l\leftarrow\lfloor b_j\rfloor$, $u\leftarrow\lceil b_j\rceil$
%             \STATE $m_l\leftarrow m_l+Q_{\theta}(s_{t+1},a^*,x_{t},l^+)(u-b_j)$
%             \STATE $m_u\leftarrow m_u+Q_{\theta}(s_{t+1},a^*,x_{t},l^+)(b_j-l)$
%         \ENDFOR
%         \STATE \COMMENT{Extract the probability density corresponding to the index from the discrete probability distribution.}
%         \STATE $p_i\leftarrow Z_{\theta}(s_t,a,x_{t-1},l^+), i\in 0,...,M-1$ 
%         \STATE {\bfseries Output:} $-\sum_{i=0}^{M-1} m_i\mathrm{log}p_i$
%     \end{algorithmic}
% \end{algorithm}

\clearpage
\section{Details on Distributional Language-Guided Policy}
\label{appendix:distrl}
In this section, we provide a more detailed introduction to the specific computational workflow of our proposed Distributional Language-Guided Policy. 
We follow the approach in~\cite{bellemare2017distributional}, employing discrete atoms to model and approximate the original value distribution. To facilitate understanding, we have included the calculation methodology for this section of the work here. 

Specifically, we model the discrete value distribution with discrete units called atoms $\{ z_i=V_{\rm MIN}+i\Delta z\}_{i=0}^{M-1}$, which are uniformly spaced support points that discretize the range of possible returns into $[V_{MIN}, V_{MAX}]$.
$M \in \mathbb{N}$ is the number of atoms. $V_{\rm MAX}, V_{\rm MIN}, M$ are pre-defined parameters, which together decide the step between the atoms of the categorical value distribution: $\Delta z:=\frac{V_{\rm MAX}-V_{\rm MIN}}{M-1}$. 

To represent the discrete value distribution, we need to estimate the probability $p_i$, which means the probability for the discrete value equaling $z_i$. In practical implementation, we approximate the probabilities $p_i$ by a parametric model $\theta:\mathcal{S}\times A\times\mathcal{X} \times \mathcal{L}\to \mathbb{R}^M$. We use $\theta_i(\cdot, \cdot, \cdot, \cdot)$ to denote the $i$-th dimension of this model's output. Therefore, the discrete value distribution can be written as:
\begin{equation}
\begin{aligned}
Z_{\theta}(s,a,x,l)&=z_i~~~ \text{w.p.} ~~~p_i(s,a,x,l):=\frac{e^{\theta_i(s,a,x,l)}}{\sum_j e^{\theta_j(s,a,x,l)}}\\
Q_{\theta}(s,a,x,l)&=\mathbb{E}[Z_{\theta}(s,a,x,l)]=\sum_{i=0}^{M-1}p_{i}(s,a,x,l)z_i
\label{eq: z_theta}
\end{aligned}
\end{equation}
where $\sum_{i=0}^{M-1}p_i=1$. We now explain how to learn $\theta$ through RL. Given a transition $(s,a,r,s',l)$, the Bellman update for each atom $z_i$ is computed as:
\begin{equation}
\hat{\mathcal{T}}z_i=r+\gamma z_i
\end{equation}
The Bellman update $\hat{\mathcal{T}}z_i$ maps the original support points to new locations that do not align with predefined discrete atoms $\{z_i\}_{i=0}^{M-1}$, making it impossible to represent as a valid discrete distribution over the fixed atoms. Therefore, we need to project these values back onto the fixed support $\{z_0,...,z_{M-1}\}$. To do so, we distribute probability mass $p_i(s',a',x',l)$ to the nearest two atoms in $[V_{\rm MIN}, V_{\rm MAX}]$, where $a'=\pi(s',x',l)$ is the output of the greedy policy $\pi(\cdot, \cdot, \cdot)$. $x'$ is computed as $x'=h_{w}(x,(s,a))$. Ultimately, we can compute the projected probability for $\hat{\mathcal{T}}z_i$ via a local interpolation mechanism:
\begin{equation}
    \text{Projected probability}=\begin{cases}
        \frac{z_{j+1}-\hat{\mathcal{T}}z_i}{\Delta z}\cdot p_i(s',a',x',l), & \text{assigned to}~z_j \\
        \frac{\hat{\mathcal{T}}z_i-z_j}{\Delta z}\cdot p_i(s',a',x',l), & \text{assigned to}~z_{j+1}
        \end{cases}
\end{equation}
Sum all the projected probabilities from all $\hat{\mathcal{T}}z_j$, we can compute the projected update probabilities by:
\begin{equation}
(\Phi \hat{\mathcal{T}}Z_\theta(s,a,x,l))_i=\sum_{j=0}^{M-1}[1-\frac{|[\hat{\mathcal{T}}z_j]_{V_{MIN}}^{V_{MAX}}-z_i|}{\Delta z}]_0^1~p_j(s',a',x',l)
\label{eq: Tz_theta}
\end{equation}
where $[\cdot]_{a}^b$ bounds the argument in the range $[a,b]$. Given the project update $\Phi\hat{\mathcal{T}}Z_{\hat{\theta}}$ and the current estimates of the discrete value distributions $Z_{\theta}$, we can update $\theta$ by minimizing the KL divergence:
\begin{equation}
D_{KL}(\Phi\hat{\mathcal{T}}Z_{\hat{\theta}}(s,a,x,l)||Z_\theta(s,a,x,l))
\end{equation}
where $\hat{\theta}$ is the target network.
\setcounter{theorem}{0}
\clearpage
\section{Theoretical Analysis}
\label{appendix:theory}

\subsection{Insights Using Distributional RL}
\label{subsec:insight}

The key insight of applying distributional RL to language-conditioned tasks lies in its capacity to capture value distributions, which provides fine-grained task differentiation across various tasks. Traditional RL methods, however, rely on learning scalar value expectations, discarding critical distributional information, making it require more samples to discriminate between tasks properly. We first demonstrate this key insight through an extreme yet illustrative example as illustrated in Figure~\ref{fig:insight_appendix}.

\begin{figure*}[h]
    \centering
    \includegraphics[width=1\textwidth]{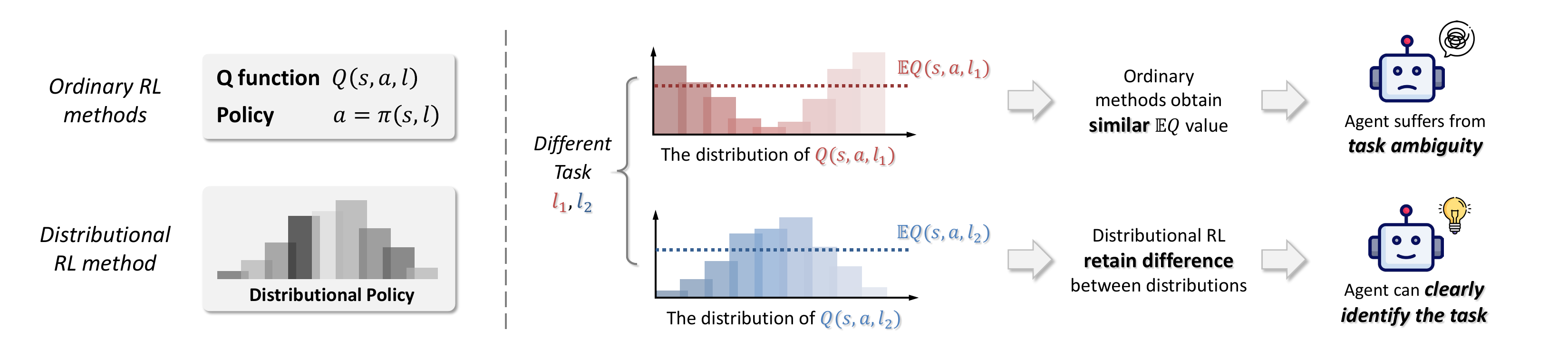}
    \vspace{-5pt}
    \caption{An illustrative case where $Q(s,a)$ functions in two tasks share the same expectations but have different distributions. In this case, traditional RL cannot discriminate between the tasks, while distributional RL can.}
    \centering
    \label{fig:insight_appendix}
\end{figure*}

Consider two distinct tasks $l_1$ and $l_2$ that share identical expected returns $Q_{\pi}(s,a,l_1)=Q_{\pi}(s,a,l_2)$ for a specific state-action pair $(s,a)$, but exhibit fundamentally divergent Q value distributions $Z(s,a,l_1)\neq Z(s,a,l_2)$. In this case, traditional RL methods relying on value estimation would inevitably conflate tasks $l_1,l_2$ at $(s,a)$, regardless of sample size, while distributional RL resolves this ambiguity through approximating the full distributions of values.

To formally validate this insight, we present the following theorem, accompanied by a rigorous proof.

\subsection{Theoretical Proof}

\begin{repdefinition}[Restatement of Definition~\ref{def: distintion}]
In a multi-task RL setting with known task instruction space $\mathcal{L}$ and unknown semantics space $\mathcal{G}$, for a task distinction threshold $\delta$ and sub-optimality gap $\epsilon$,  two task instructions \( l_i, l_j \in \mathcal{L}\) are considered with \textbf{different underlying semantics}, $g_i\not\leftrightarrow g_j$, if the expected Q-values under any shared $\epsilon$-optimal policy \( \pi \) satisfy:
\[
\mathbb{E}_{\pi}\left[|Q_{\pi}(s, a, l_i) - Q_{\pi}(s, a, l_j)|\right] \geq \delta.
\]
Conversely, if
\[
\mathbb{E}_{\pi}\left[|Q_{\pi}(s, a, l_i) - Q_{\pi}(s, a, l_j)|\right] \leq \delta/2,
\]
then \( l_i \) and \( l_j \) are considered with the \textbf{same underlying semantics}, $g_i\leftrightarrow g_j$, where $s_0\sim p_0(\cdot), a\sim \pi(\cdot|s),s'\sim p(\cdot|s,a)$. $V_{\pi}(s)\geq V^*(s)-\epsilon, \forall s \in \mathcal{S}$, $V^*$ is optimal value function.

In the case of distributional reinforcement learning, the Wasserstein distance \( W_1 \) between the return distributions \( Z(s, a) \) is used as the criterion. Specifically, \( l_i \) and \( l_j \) are considered to represent \textbf{different semantics} if
\[
\mathbb{E}_{\pi}\left[W_1\left(Z_{\pi}(s, a, l_i), Z_{\pi}(s, a, l_j)\right)\right] \geq d,
\]
and the \textbf{same semantics} if this expectation is less than or equal \( d/2 \).
\end{repdefinition}

\begin{theorem}[Sample Complexity for Task Instruction Disambiguation]
Consider an offline multi-task RL setting with \( M \) distinct tasks~(with different semantics). In direct Q-value estimate setting, suppose \( Q(s, a, l) \in [0, Q_{\max}],~ \forall (s,a,l) \sim \mathcal{D}\) with finite $Q_{\max}$, and the task distinction threshold is \( \delta > 0 \). When the number of training samples \( n_{\textup{value}} \) satisfies:
\[
n_{\textup{value}} \geq \frac{C_{\textup{value}} \log(3M^2 / \eta)}{\delta^2}.
\]
The mean value estimate algorithm achieves task-level semantic disambiguation with confidence at least \( 1 - \eta \).
In the distributional RL setting, let \( Z(s, a,l) \) denote the learned return distribution, and suppose the task distinction threshold is given by a 1-Wasserstein distance \( d > 0 \). Then, to ensure semantic disambiguation of task instructions with confidence at least \( 1 - \eta \), it suffices that:
\[
n_{\textup{dist}} \geq \frac{C_{\textup{dist}} \log(3M^2 / \eta)}{d^2},
\]
where \( C_{\textup{value}}, C_{\textup{dist}} >0 \) are universal constants depending on certain attributes of Q-value distribution.

\label{thm: sample-complexity-appendix}
\end{theorem}

\begin{proof}
    We denote $Q_{i}(s,a)$ as a shorthand for $Q_{\pi}(s,a,l_i)$ and $Z_{i}(s,a)$ as a shorthand for $Z_{\pi}(s,a,l_i)$ in the following discussion.
    For direct offline Q-learning, when the task distinction threshold is \( \delta > 0 \), we define the semantic ambiguity event $\mathcal{S}_a$ as:
    \begin{align}
        \mathbb{E}_{(s,a)\sim\mathcal{D}}\left[|\hat{Q}_{i}(s,a)-\hat{Q}_{j}(s,a)|\right] &\leq \frac{\delta}{2}, \\ 
        \mathbb{E}_{\pi}\left[|Q_{i}(s, a) - Q_{j}(s, a)|\right] &\geq \delta, g_i\not\leftrightarrow g_j,
    \label{eq: ambiguity}
    \end{align}
    where $\mathcal{D}$ is the offline dataset, $Q$ is the optimal Q function, $\pi$ is the optimal policy along with $Q$, and $\hat{Q}$ is the learned Q function.
    
    % For simplicity, we denote $Q_i$ as a shorthand for $Q_{l_i}$ and $\hat{Q}_i$ for $\hat{Q}_{l_i}$ in the following discussion.

    % $\mathbb{E}_{\pi}[|Q_i-Q_j|]$  can be regarded as the expectation of $[|\hat{Q}_i-\hat{Q}_j|],(s,a)\sim\mathcal{D}$. 

    Following the triangle inequality, we have:
    \begin{align}
        &\left|\mathbb{E}_{(s,a)\sim\mathcal{D}}[|\hat{Q}_i-\hat{Q}_j|]-\mathbb{E}_{\pi}[|Q_i-Q_j|]\right|\leq \\
        &\left|\mathbb{E}_{(s,a)\sim\mathcal{D}}[|\hat{Q}_i-\hat{Q}_j|]-\mathbb{E}_\mathcal{D}\mathbb{E}_{(s,a)\sim\mathcal{D}}[|\hat{Q}_i-\hat{Q}_j|]\right| + \left|\mathbb{E}_\mathcal{D}\mathbb{E}_{(s,a)\sim\mathcal{D}}[|\hat{Q}_i-\hat{Q}_j|]-\mathbb{E}_{\pi}[|Q_i-Q_j|]\right|
    \end{align}

    In general, we assume that the offline RL dataset is collected by $\epsilon$-optimal policy, i.e., there exists a $\epsilon$-optimal policy $\pi$ that $\mathcal{D}$ satisfies 
    \begin{equation}
        \mathbb{E}_\mathcal{D}\mathbb{E}_{(s,a)\sim\mathcal{D}}[|\hat{Q}_i-\hat{Q}_j|] = \mathbb{E}_{\pi}[|\hat{Q}_i-\hat{Q}_j|], \mathbb{E}_{\mathcal{D}}(\hat{Q})=Q_{\pi}=Q
        \label{eq: dataset_assumption}
    \end{equation}
    
    Besides, we have:
    \begin{align}
    \left|\mathbb{E}_{\pi}[|\hat{Q}_i-\hat{Q}_j|]-\mathbb{E}_{\pi}[|Q_i-Q_j|]\right| &= \left|\mathbb{E_{\pi}}\left( |\hat{Q}_i-\hat{Q}_j|-|Q_i-Q_j|\right)\right| \\
    &\leq \mathbb{E}_{\pi} \left||\hat{Q}_i-\hat{Q}_j|-|Q_i-Q_j|\right| \\
    &\leq \mathbb{E}_{\pi}\left|(\hat{Q}_i-\hat{Q}_j)-(Q_i-Q_j)\right| \\
    &\leq \mathbb{E}_{\pi}|\hat{Q}_i-Q_i| + \mathbb{E}_{\pi}|\hat{Q}_j-Q_j|
    \end{align}

    Therefore, 
    \begin{align}
        &\Pr \left(\left|\mathbb{E}_{(s,a)\sim\mathcal{D}}[|\hat{Q}_i-\hat{Q}_j|]-\mathbb{E}_{\pi}[|Q_i-Q_j|]\right| \geq \frac{\delta}{2}\right) \\
        &\leq\Pr\left(\left|\mathbb{E}_{(s,a)\sim\mathcal{D}}[|\hat{Q}_i-\hat{Q}_j|]-\mathbb{E}_{\pi}[|\hat{Q}_i-\hat{Q}_j|]\right| + \mathbb{E}_{\pi}|\hat{Q}_i-Q_i| + \mathbb{E}_{\pi}|\hat{Q}_j-Q_j|\geq \frac{\delta}{2}\right)\\
        & \leq \Pr\left(\left|\mathbb{E}_{(s,a)\sim\mathcal{D}}[|\hat{Q}_i-\hat{Q}_j|]-\mathbb{E}_{\pi}[|\hat{Q}_i-\hat{Q}_j|]\right|\geq \frac{\delta}{2}\right) \\
        &\quad + \Pr\left(\mathbb{E}_{\pi}|\hat{Q}_i-Q_i|\geq \frac{\delta}{2}\right) +\Pr\left(\mathbb{E}_{\pi}|\hat{Q}_j-Q_j|\geq \frac{\delta}{2}\right) 
        \label{eq: value_based_scale}
    \end{align}

    Since $\hat{Q},Q\in[0,Q_{\max}]$, $|\hat{Q}_i-\hat{Q}_j|,|\hat{Q}_i-Q_i|,|\hat{Q}_j-Q_j|\in[0,Q_{\max}]$, together with Equation \ref{eq: dataset_assumption}, we can apply Theorem 1 (Hoeffding's inequality) in~\cite{boucheron2003concentration}, and obtain that for $\forall l_i,l_j\in \mathcal{L}$,
    \begin{align}
        \mathrm{Pr}\left(\left|\mathbb{E}_{(s,a)\sim\mathcal{D}}[|\hat{Q}_i-\hat{Q}_j|] -\mathbb{E}_{\pi}[|\hat{Q_i}-\hat{Q_j}|]\right|\geq \delta/2\right)&\leq 2\exp(-n_{\textup{value}}\delta^2/C_{\text{value}}) \\
        \Pr\left(\mathbb{E}_{\pi}|\hat{Q}_i-Q_i|\geq \delta/2\right) &\leq 2\exp(-n_{\textup{value}}\delta^2/C_{\text{value}}) \\
        \Pr\left(\mathbb{E}_{\pi}|\hat{Q}_j-Q_j|\geq \delta/2\right) &\leq 2\exp(-n_{\textup{value}}\delta^2/C_{\text{value}})
    \end{align}
    where $n_{\textup{value}}=|\mathcal{D}|$ is the size of dataset, $C_{\text{value}}\in (0,\mathcal{O}(Q_{\max}^2)]$ is a constant value.
    
    Combining with Equation \ref{eq: value_based_scale}, we have
    \begin{equation}
        \Pr \left(\mathbb{E}_{(s,a)\sim\mathcal{D}}[|\hat{Q}_i-\hat{Q}_j|]-\mathbb{E}_{\pi}[|Q_i-Q_j|] \leq -\delta/2\right) \leq 6\exp(-n_{\textup{value}}\delta^2/C_{\text{value}})
    \end{equation}
    Consider the extreme case in task semantics distinction, when $\exists l_i,l_j \in \mathcal{L}, g_i \not \leftrightarrow g_j,\mathbb{E}_{\pi}[|Q_i-Q_j|]=\delta$ that can be exactly partitioned under the threshold. We have
    \begin{equation}
        \mathrm{Pr}\left(\mathbb{E}_{(s,a)\sim\mathcal{D}}[|\hat{Q}_i-\hat{Q}_j|]\leq \delta/2\right)\leq 6\exp(-n_{\textup{value}}\delta^2/C_{\text{value}})
    \end{equation}
    Consider all task pairs with different real semantics $\left(\begin{array}{c}
     M  \\
    2
    \end{array}\right)\approx \frac{M^2}{2}$:
    \begin{equation}
        \mathrm{Pr}\left(\exists i,j ,g_i\not\leftrightarrow g_j,\mathbb{E}_{(s,a)\sim\mathcal{D}}[|\hat{Q}_i-\hat{Q}_j|]\leq \delta/2\right)\leq 3M^2\exp(-n_{\textup{value}}\delta^2/C_{\text{value}})
    \end{equation}
    To ensure a confidence level of at least $1-\eta$ for  any task ambiguity, we require that the probability of the event $\mathcal{S}_a$ satisfies:
    \[
    \mathrm{Pr}(\mathcal{S}_a)\leq\eta,
    \]
    Thus, we need to let:
    \[
    3M^2\exp(-n_{\textup{value}}\delta^2/C_{\text{value}})\leq \eta.
    \]
    Then,
    \begin{equation}
        n_{\textup{value}}\geq \frac{C_{\text{value}}}{\delta^2}\log\frac{3M^2}{\eta}
    \end{equation}

    For the distributional RL setting, we have the estimated return distributions $\hat{Z}_i, \hat{Z}_j$ and real distributions $Z_i, Z_j$, with the threshold $d>0$. Similar to Equation \ref{eq: ambiguity}, we have the semantic ambiguity event $\mathcal{S}_a$ as:
    \begin{align}
        \mathbb{E}_{(s,a)\sim\mathcal{D}}\left[W_1(\hat{Z}_{i}(s,a),\hat{Z}_{j}(s,a))\right]&\leq \frac{d}{2},\\
        \mathbb{E}_{\pi}\left[W_1\left(Z_{i}(s, a), Z_{j}(s, a)\right)\right] &\geq d,g_i\not\leftrightarrow g_j,
    \end{align}
    where $Z$ is the optimal distribution of Q-value, $\pi$ is the optimal policy along with $Z$, and $\hat{Z}$ is the learned distribution.

    Similarly, we have:
    \begin{align}
        &\left|\mathbb{E}_{(s,a)\sim\mathcal{D}}\left[W_1(\hat{Z}_{i},\hat{Z}_{j})\right]-\mathbb{E}_{\pi}\left[W_1\left(Z_{i}, Z_{j}\right)\right]\right| \\
        &\quad\leq \left|\mathbb{E}_{(s,a)\sim\mathcal{D}}\left[W_1(\hat{Z}_{i},\hat{Z}_{j})\right]-\mathbb{E}_{\mathcal{D}}\mathbb{E}_{(s,a)\sim\mathcal{D}}\left[W_1(\hat{Z}_{i},\hat{Z}_{j})\right]\right| \\ 
        &\quad \quad+\left|\mathbb{E}_{\mathcal{D}}\mathbb{E}_{(s,a)\sim\mathcal{D}}\left[W_1(\hat{Z}_{i},\hat{Z}_{j})\right]-\mathbb{E}_\pi\left[W_1(Z_{i},Z_{j})\right]\right| \\
        &\quad=\left|\mathbb{E}_{(s,a)\sim\mathcal{D}}\left[W_1(\hat{Z}_{i},\hat{Z}_{j})\right]-\mathbb{E}_{\pi}\left[W_1(\hat{Z}_{i},\hat{Z}_{j})\right]\right| + \left|\mathbb{E}_\pi\left[ W_1(\hat{Z}_{i},\hat{Z}_{j})-W_1(Z_{i},Z_{j})\right] \right|
    \end{align}
    
    Following the triangle inequality of Wasserstein distance, we have:
    \begin{equation}
        W_1(\hat{Z}_i,\hat{Z}_j)-W_1(Z_i,Z_j)\leq W_1(\hat{Z}_i,Z_i) + W_1(Z_j,\hat{Z}_j)
        \label{eq: triangle}
    \end{equation}
    
    Then,
    \begin{align}
    &\Pr\left( \left|\mathbb{E}_{(s,a)\sim\mathcal{D}}\left[W_1(\hat{Z}_{i},\hat{Z}_{j})\right]-\mathbb{E}_{\pi}\left[W_1\left(Z_{i}, Z_{j}\right)\right]\right| \geq \frac{d}{2} \right) \\
    &\leq \Pr\left( \left|\mathbb{E}_{(s,a)\sim\mathcal{D}}\left[W_1(\hat{Z}_{i},\hat{Z}_{j})\right]-\mathbb{E}_{\pi}\left[W_1(\hat{Z}_{i},\hat{Z}_{j})\right]\right| + \left|\mathbb{E}_\pi\left[ W_1(\hat{Z}_{i},\hat{Z}_{j})-W_1(Z_{i},Z_{j})\right] \right| \geq \frac{d}{2} \right) \\
    &\leq \Pr\left( \left|\mathbb{E}_{(s,a)\sim\mathcal{D}}\left[W_1(\hat{Z}_{i},\hat{Z}_{j})\right]-\mathbb{E}_{\pi}\left[W_1(\hat{Z}_{i},\hat{Z}_{j})\right]\right| + \left|\mathbb{E}_\pi\left[ W_1(\hat{Z}_{i},Z_{i})+W_1(\hat{Z}_{j},Z_{j})\right] \right| \geq \frac{d}{2} \right) \\
    &\leq \Pr\left( \left|\mathbb{E}_{(s,a)\sim\mathcal{D}}\left[W_1(\hat{Z}_{i},\hat{Z}_{j})\right]-\mathbb{E}_{\pi}\left[W_1(\hat{Z}_{i},\hat{Z}_{j})\right]\right| \geq \frac{d}{2} \right)\\ 
    &\quad+ \Pr\left( \left|\mathbb{E}_\pi\left[ W_1(\hat{Z}_{i},Z_{i})\right] \right| \geq \frac{d}{2} \right)+ \Pr\left( \left|\mathbb{E}_\pi\left[ W_1(\hat{Z}_{j},Z_{j})\right] \right| \geq \frac{d}{2} \right)
    \label{eq: wasserstein-process}
    \end{align}

    % \begin{align}
    %     &\textrm{Pr}\left( |W_1(\hat{Z}_i,\hat{Z}_j)-W_1(Z_i,Z_j)| \geq d/2\right) \leq \\ &\textrm{Pr}\left( |W_1(\hat{Z}_i,\hat{Z}_j)-W_1(Z_i,\hat{Z}_j)| + |W_1(Z_i,\hat{Z}_j)-W_1(Z_i,Z_j)|\geq d/2\right)  \leq \\
    %     &\textrm{Pr}\left( |W_1(\hat{Z}_i,\hat{Z}_j)-W_1(Z_i,\hat{Z}_j)|\geq d/2 \right) + \textrm{Pr}\left( |W_1(Z_i,\hat{Z}_j)-W_1(Z_i,Z_j)|\geq d/2 \right)
        
    % \end{align}
    Since $\hat{Z}$ is the empirical measure of $Z$ and the estimated return is 1-dimensional, $Q\in [0,Q_{\max}],W_1(Z_i,Z_j)\in[0,Q_{\max}]$. Following Theorem 1 (Hoeffding's inequality) in~\cite{boucheron2003concentration} , Corollary 5.2 and Remark 5 in~\cite{lei2020convergence}, we have the following equation where $C_{\text{dist}_1}, C_{\text{dist}_2}\in (0, \mathcal{O}(Q_{\max}^2)] $ are constant values.
    \begin{align}
        \Pr\left( \left|\mathbb{E}_{(s,a)\sim\mathcal{D}}\left[W_1(\hat{Z}_{i},\hat{Z}_{j})\right]-\mathbb{E}_{\pi}\left[W_1(\hat{Z}_{i},\hat{Z}_{j})\right]\right| \geq d/2 \right) &\leq 2\exp(-n_{\textup{dist}}d^2/C_{\text{dist}_1}) \\
        \Pr\left( \mathbb{E}_\pi\left[ W_1(\hat{Z}_{i},Z_{i})\right]  \geq d/2 \right) &\leq 2\exp(-n_{\textup{dist}}d^2/C_{\text{dist}_2}) \\
        \Pr\left( \mathbb{E}_\pi\left[ W_1(\hat{Z}_{j},Z_{j})\right]  \geq d/2 \right) &\leq 2\exp(-n_{\textup{dist}}d^2/C_{\text{dist}_2})
    \end{align}
        
    Combining with Equation \ref{eq: wasserstein-process}, we can obtain the following equation where $C_{\text{dist}}=\max(C_{\text{dist}_1}, C_{\text{dist}_2})>0$
    \begin{equation}
        \mathrm{Pr}\left( \mathbb{E}_{(s,a)\sim D}\left(W_1(\hat{Z}_i,\hat{Z}_j)\right)-\mathbb{E}_{\pi}\left(W_1(Z_i,Z_j)\right)\leq -d/2\right) \leq 6\exp(-n_{\textup{dist}}d^2/C_{\text{dist}})
    \end{equation}
    In the same way, consider the extreme case when $\exists l_i,l_j\in \mathcal{L}, g_i \not \leftrightarrow g_j,\mathbb{E}_{\pi}\left(W_1(Z_i,Z_j)\right)=d$. We have
    \begin{equation}
        \mathrm{Pr}\left( \mathbb{E}_{(s,a)\sim D}\left(W_1(\hat{Z}_i,\hat{Z}_j)\right)\leq d/2\right) \leq 6\exp(-n_{\textup{dist}}d^2/C_{\text{dist}})
    \end{equation}
    Consider all task pairs with different real semantics $\left(\begin{array}{c}
     M  \\
    2
    \end{array}\right)\approx \frac{M^2}{2}$:
    \begin{equation}
        \mathrm{Pr}\left( \exists i,j ,g_i\not\leftrightarrow g_j,\mathbb{E}_{(s,a)\sim D}\left(W_1(\hat{Z}_i,\hat{Z}_j)\right)\leq d/2\right) \leq 3M^2 \exp(-n_{\textup{dist}}d^2/C_{\text{dist}})
    \end{equation}
    Similarly, to ensure a confidence level of at least $1-\eta$ for avoiding task ambiguity, we need 
    \[
        3M^2 \exp(-n_{\textup{dist}}d^2/C_{\text{dist}}) \leq \eta.
    \]
    Then,
    \begin{equation}
        n_{\textup{dist}} \geq \frac{C_{\text{dist}}}{d^2}\log\frac{3M^2}{\eta}
    \end{equation}
\end{proof}

\begin{corollary}
In a multi-task RL setting, to avoid task ambiguity with confidence level $ 1-\eta$, learning the distribution over Q-values requires fewer samples than learning point estimates of Q-values when the number of tasks $M$ is sufficiently large. Formally, $n_{\textup{value}} \geq n_{\textup{dist}}$, where $n_{\textup{value}}, n_{\textup{dist}}$ denote the samples needed to avoid task ambiguity for value-based and distributional settings.
\label{cor: sample-complexity}
\end{corollary}

\begin{proof}
    From Theorem \ref{thm: sample-complexity-appendix}, we have
    \begin{align}
    n_{\text{value}}\geq& \frac{C_{\text{value}}}{\delta^2}\log\frac{3M^2}{\eta} \\
        n_{\text{dist}} \geq& \frac{C_{\text{dist}}}{d^2}\log\frac{3M^2}{\eta}
    \label{eq: comparison}
    \end{align}
    First, $Q\in [0, Q_{\max}], W_1(Z_i,Z_j)\in[0, Q_{max}]$, following Theorem 1 (Hoeffding's inequality) in~\cite{boucheron2003concentration} , Corollary 5.2 and Remark 5 in~\cite{lei2020convergence} we obtain that the constant $C$ satisfies: 
    \begin{equation}
        C_{\text{value}}\leq \mathcal{O}(Q_{\max}^2),C_{\text{dist}}\leq \mathcal{O}(Q_{\max}^2).
        \label{eq: constant}
    \end{equation}
    Then, we can prove that for any Q-value distribution $Z_i,Z_j$,
    \begin{equation}
        W_1(Z_i,Z_j)\geq |\mathbb{E}_{Z_i}(Q)-\mathbb{E}_{Z_j}(Q)|
    \end{equation}
    The definition of Wasserstein-1 distance is
    \[
    W_1(Z_i,Z_j)=\inf_{\pi \in \Pi(Z_i,Z_j)}\mathbb{E}_{(X,Y)\sim \pi}|X-Y|
    \]
    Here, $\Pi(Z_i,Z_j)$ denotes the set of all joint distributions with marginals $Z_i,Z_j$. 

    For $(Z_i,Z_j)\sim \pi$, we have
    \[
    \mathbb{E}[|X-Y|]\geq |\mathbb{E}(X-Y)|=|\mathbb{E}_{Z_i}(X)-\mathbb{E}_{Z_j}(Y)|=|\mathbb{E}_{Z_i}(X)-\mathbb{E}_{Z_j}(X)|
    \]
    Thus, for any joint $\pi$,
    \[
    \mathbb{E}_{(X,Y)\sim \pi}|X-Y|\geq|\mathbb{E}_{Z_i}(X)-\mathbb{E}_{Z_j}(X)|
    \]

    For the optimal joint $\pi$,
    \begin{equation}
        W_1(Z_i,Z_j)\geq |\mathbb{E}_{Z_i}(Q)-\mathbb{E}_{Z_j}(Q)|
        \label{eq: wasserstein-bound}
    \end{equation}

    Thus, we must choose a much smaller threshold $\delta$ for point estimates than that $d$ for distribution learning, $\delta \leq d$. In some scenarios, when the mean difference of the Q function is small but the distribution difference is large, $\delta<<d$.

    Combining Equation \ref{eq: constant}, \ref{eq: wasserstein-bound}, we can obtain that
    \begin{equation}
        \frac{C_{\text{value}}}{\delta^2}\log \frac{3M^2}{\eta} \geq \frac{C_{\text{dist}}}{d^2}\log \frac{3M^2}{\eta}
        \label{eq: approx2}
    \end{equation}
     we prove that $n_{\text{value}} \geq n_{\text{dist}}$, even $n_{\text{value}} >> n_{\text{dist}}$ in some scenarios.
    
\end{proof}
\clearpage
\section{Experimental Details}
\label{appendix:implement}
\subsection{Toy Experiment}
In this toy experiment based on the Minigrid environment~\cite{MinigridMiniworld23}, we demonstrate that vanilla offline RL fails to establish the relationship between the tasks and their underlying reward functions as the number of tasks explodes. 

The setup consists of 10 accessible goal positions $\mathcal{G}=\{g_0,g_1,...,g_{9}\}$, where $\mathcal{G}$ represents the set of all possible goal positions. 
The agent (red triangle-shaped) must follow a given instruction $l\in \mathcal{L}$ to navigate to a specific goal position. 
We simulate instructions using numerical task IDs, making $\mathcal{L}\subset \mathbb{N}$. We employ a random mapping $F:\mathcal{L} \to \mathcal{G}$ to assign each $l$ to goal position $g$, which simulates semantics of instructions and is hidden from the agent. With these settings, we can control the number of instructions $|\mathcal{L}|$ by simply adjusting the set of valid task IDs and the mappings. We conduct 10 experiments, varying the number of instructions from 1 to $2^9$. For each experiment, we generate a new mapping $F$ and collect 1024 random trajectories as the offline dataset. 

% Our experimental map has 10 goal positions as illustrated in the Right of Figure \ref{fig:toy_appendix}). Tasks are distinguished by task IDs, and each task is assigned 1 goal position, which is invisible to the agent. During training and testing, the agent receives a task ID, which simulates the language instruction in the language-conditioned setting. 
The agent always starts from the center of the map. For each step, the agent receives the task ID and current state as input, and the agent can choose to move forward, turn left, or turn right. A reward of $1-0.9\times(\texttt{STEP\_COUNT}/\texttt{MAX\_STEP})$ is given for success and 0 for failure. \texttt{MAX\_STEP} is fixed to 12 in our toy experiment.
\begin{figure*}[h]
    \centering
    \includegraphics[width=1\linewidth]{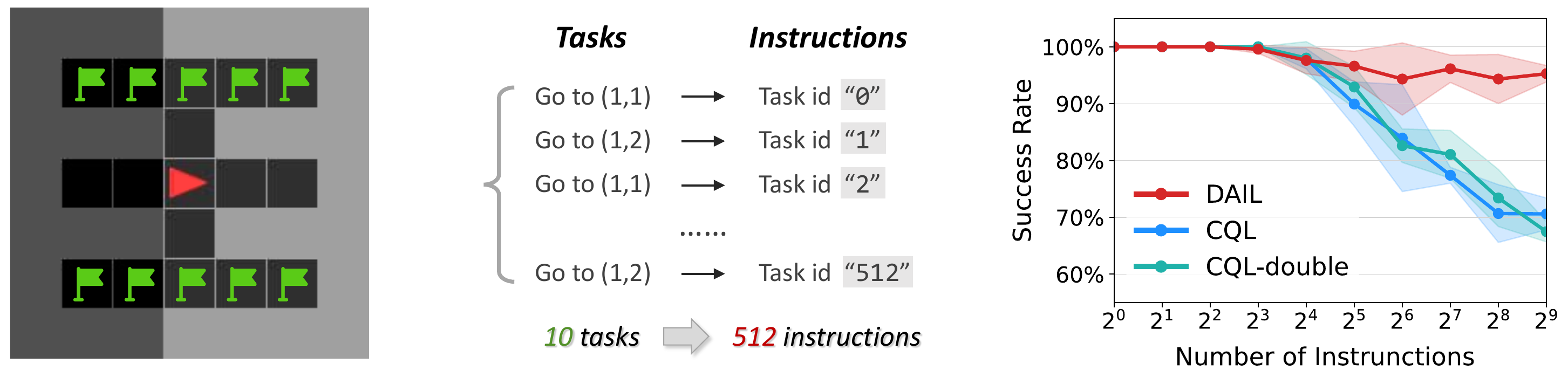}
    \vspace{-10pt}
    \caption{\textbf{Left:} The green flags in the map denote the accessible goals.
    \textbf{Middle:} An illustration of the mapping between goal positions and instructions.
    \textbf{Right:} Average success rates over 100 evaluations for each number of instructions and 3 seeds.}
    \label{fig:toy_appendix}
\end{figure*}
% We vary the number of tasks from 1 to 512. 
For offline datasets, a random policy with $\texttt{MAX\_STEP}=12$ is used to generate $64\times n$ trajectories for 8 or fewer tasks, where $n$ is the number of tasks, and 1024 trajectories for 16 or more tasks, respectively. For all these datasets, the success rate is fixed to be 0.5. 

For the observation signals in this toy experiment, we use a simple architecture, which combines Bag-of-Words encoding~\cite{mikolov2013efficient} and a two-layer CNN. We use 1 fully connected layer to embed task IDs into task representations and a two-layer MLP as the output network. We use 64 as the feature size for CQL and our method, 128 for CQL-double separately. For each number of tasks, each agent is trained with $\alpha=\{0, 0.01,0.1,0.5,1,2\}$ over 3 seeds. We calculate the average of the results from these 3 seeds and record the best performance among them as the performance for this number of instructions, as illustrated in the Right of Figure~\ref{fig:toy_appendix}.

% In the Left of Figure \ref{fig:toy_appendix}, we demonstrate the performance change when the number of tasks climbs. There is a significant performance drop in two CQL agents. The experimental results before the number of tasks reached 16 indicate that, in this environment, the agents can achieve a very high success rate using no more than 64 trajectories for each true goal. However, even though the average number of trajectories for each goal remains unchanged, as the number of tasks explodes, the two CQLs gradually fail to establish the mapping between the tasks and their true goals. 

% To look into this problem more deeply, we divide the language-conditioned architecture into two modules: the task representation module, and the policy learning module. As for the former module, $f_{\psi}$ is learned, and a task representation $z=f_{\psi}(l)$ is provided to the later module.

% Ultimately, we argue that when the task scale explodes, vanilla RL struggles to discern the relationship between the tasks and their underlying reward functions, leading to a dramatic decline in performance in multi-task problems. This flaw is further amplified in language-conditioned settings, as there can be a variety of different expressions in natural language aiming at the same goal.

\subsection{Offline Datasets}
\subsubsection{BabyAI}
\begin{figure}
    \centering
    \includegraphics[width=0.4\linewidth]{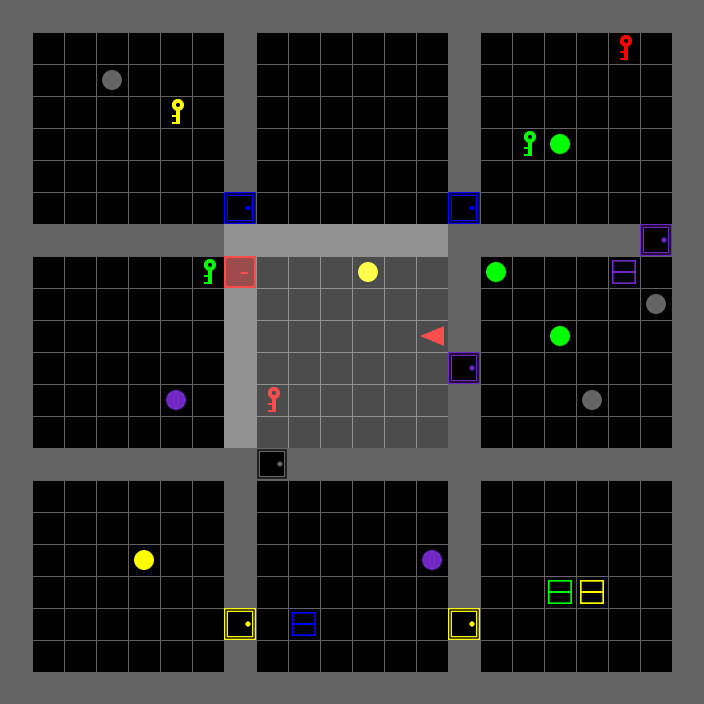}
    \caption{Level \texttt{SynthLoc} in BabyAI}
    \label{fig:synthloc}
\end{figure}
BabyAI~\cite{chevalier2018babyai} is a language-conditioned research platform built on MiniGrid~\cite{MinigridMiniworld23}, which provides different levels of tasks equipped with varied language instructions. We choose level $\verb|SynthLoc|$ as the benchmark, which is the union of all instructions from $\verb|PutNext|$, $\verb|Open|$, $\verb|Goto|$, and $\verb|PickUp|$. The agent needs to deal with synthetic Baby Language and interact with the specified objects at the goal position. Some examples of language instructions are ``put the green key behind you next to a box'', ``go to the red ball behind you'', and ``pick up a green box''. 

We follow the default map configuration of BabyAI, where each room has a size of $7\times7$, arranged in a $3\times3$ grid with a total of 9 rooms. Each room may be connected to others via a door, as illustrated in Figure \ref{fig:synthloc}. The agent has a field of view of $7\times7$ in front of it, with the observation size being ($7\times7\times3$). Each grid cell in the observation contains the values (\texttt{OBJECT\_IDX}, \texttt{COLOR\_IDX}, \texttt{STATE}), all represented in a structured format. A reward of `1 - 0.9 * ($\text{step\_count}$ / $\text{max\_steps}$)' is only given for success, and `0' in all other cases. Agents are permitted to take up to 300 steps before truncation in our setting.
\begin{figure}[h]
    \centering
    \includegraphics[width=0.5\textwidth]{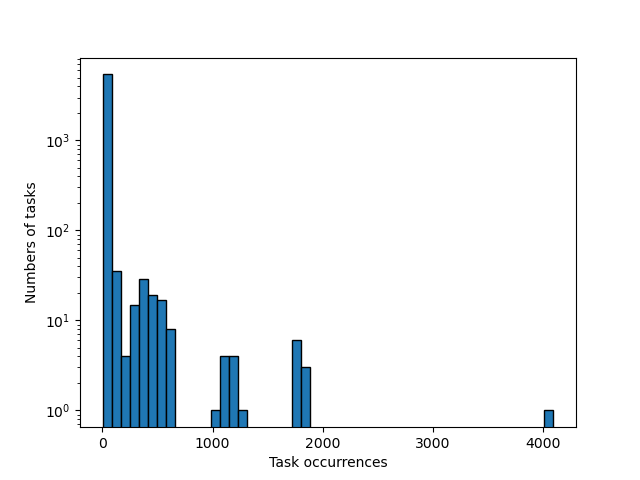} 
    \caption{The histogram shows the frequency distribution of tasks in $10^6$ samples. The x-axis represents the total frequency of each task, while the y-axis indicates the number of tasks that fall within each frequency interval.}
    \label{fig:babyai-taskcount}
\end{figure}
In $\verb|SynthLoc|$, the total number of tasks is large, and their distribution is sparse. We reset the environment $10^6$ times and obtain over 5500 unique instructions. From the histogram of task count (Figure \ref{fig:babyai-taskcount}), it can be observed that most tasks only appear less than 500 times in $10^6$, while some tasks appear over 4000 times. The distribution of tasks highlights the highly uneven distribution of tasks. We divide the task set into two subsets, designating approximately 60\% of the tasks as in-distribution tasks. All trajectories in the offline dataset are collected under in-distribution instructions, while tasks encountered during testing outside this set are considered out-of-distribution tasks.

To construct the offline dataset, we collect three types of data: expert data, gathered by a pre-designed bot within the environment; medium data, collected by a well-trained agent; and random data.
The built-in bot has access to global information to accomplish every possible task with a near-optimal solution. 
We train an IL agent following BabyAI 1.1~\cite{hui2020babyai}, the state-of-the-art model proposed by the original environmental authors. Trained on a dataset of 100k expert trajectories, it achieved approximately 87.9\% success rate across all tasks. Random agent achieves a 10.5\% success rate during data collection.
We conduct a high-quality dataset with 50k expert trajectories, 50k IL agent trajectories, and 25k random trajectories; a medium-quality dataset with 12.5k expert trajectories, 25k IL agent trajectories, and 40k random trajectories. All the trajectories in the dataset are generated under in-distribution instructions.

\subsubsection{ALFRED}

ALFRED~\cite{shridhar2020alfred} benchmarks sequential decision-making tasks involving household activities (e.g., cleaning, heating food) through language instructions and first-person vision, shown in Figure~\ref{fig:example_scene} and Table~\ref{table:instruction_example}. 
The dataset provides 8055 expert demonstrations with 25k human-annotated language instructions detailing both high-level goals and sub-goal step-by-step guidance. 
As our work primarily focuses on low-level policy learning rather than high-level planning, we specifically concentrate on the $\verb|GOTO|$ sub-goal setting for our evaluation.
In this task set, the agent must go to specific locations according to instructions like ``Move to other side of couch on the right side of the table before the door''. To simulate the presence of noisy data in real-world applications, we augment the training set with 30k random-agent trajectories, resulting in 97896 total trajectories with 53442 unique instructions across 108 household scenes.

As for the experiment, we use the Modeling Quickstart dataset, which is recommended~\cite{shridhar2020alfred}, including trajectory JSONs and pre-generated ResNet features. The ResNet features are obtained using a pre-trained ResNet-18~\cite{he2016deep} to extract $512 \times 7\times 7$ features from the \texttt{conv5} layer, which are used as observation input during training and evaluation.

\begin{figure}[h]
    \centering
    \begin{subfigure}{\textwidth}
        \centering
        \includegraphics[width=0.5\textwidth]{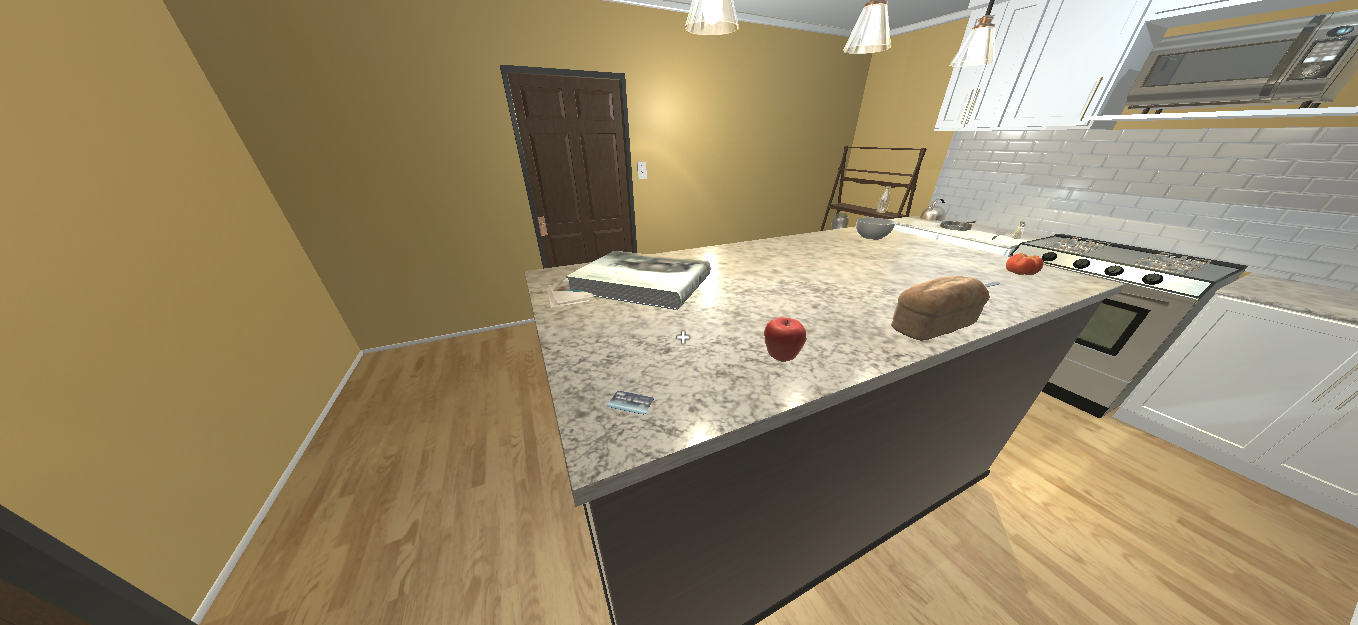} 
        \label{fig:scene1}
    \end{subfigure}

    % \vspace{1em} 

    \begin{subfigure}{\textwidth}
        \centering
        \includegraphics[width=0.5\textwidth]{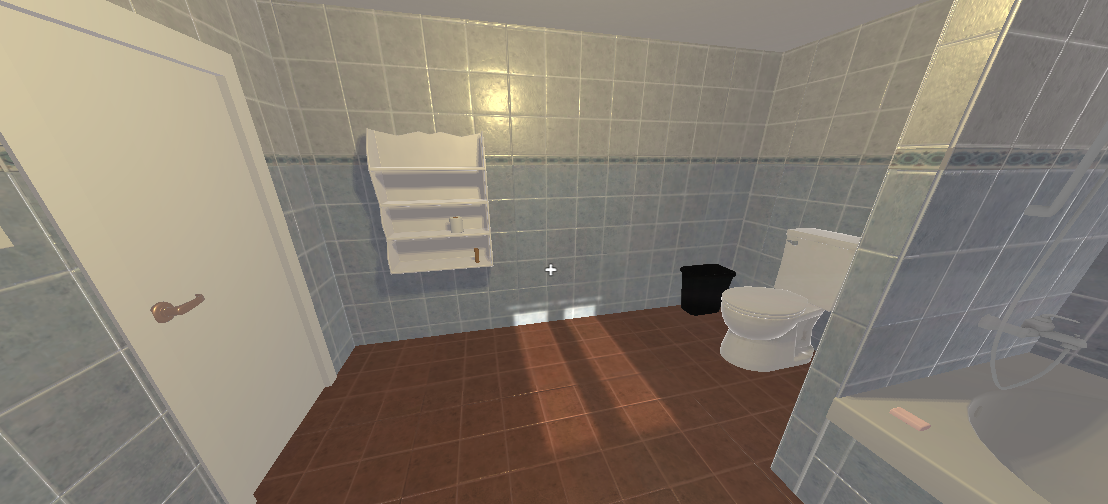} 
        \label{fig:scene2}
    \end{subfigure}

    \caption{Two example scenes from ALFRED.}
    \label{fig:example_scene}
    \vspace{-0.1in}
\end{figure}
\begin{table*}[h]
    \caption{Instruction examples in ALFRED \texttt{GOTO} task set.}
    \label{table:instruction_example}
    \centering
    \begin{tabular}{|p{0.3cm}|p{12cm}|}
        \hline
        \# & \multicolumn{1}{|c|}{Instructions} \\ \hline
        1    & Go left and turn to the right to face the couch.   \\ \hline
        2    & Turn around and make a left immediately after the toilet turn a quick left to face the side of the toilet.   \\ \hline
        3    & Move to other side of couch on the right side of the table before the door.   \\ \hline
        4    & go back to your right to the fridge and open the door   \\ \hline
        5    & Turn around and walk to the white stove on the right.   \\ \hline
        6    & Turn to the right and go to the sink across from you. \\ \hline
        7    & Turn around and walk towards the toilet, then turn right and walk towards the door, turn left to face the counter. \\ \hline
        % 8    & Move to stand in front of the dishwasher and microwave. \\ \hline
        % 9    & Bring the cleaned fork to the table with the coffee machine. \\ \hline
        8   & Walk forward, then hang a right and walk across the room, turn left and walk up to the chair. \\ \hline
    \end{tabular}
\end{table*}

For sub-goal evaluation, we follow~\cite {shridhar2020alfred} to use the expert trajectory to move the agent until the start state of the tested sub-goal, and the agent takes over to operate based on the instructions and observations. Episodes with 5 or more failed actions or exceeding the $\texttt{MAX\_STEP}=32$ are counted as failures immediately. A reward of 1 is given for success and 0 for failure. Failed actions refer to actions that cannot be successfully executed in the current state~(for example, moving against the walls or other obstacles in the room). 
% Additionally, it does not require extra training in the masking network to select items, making it more aligned with our needs. 

\clearpage
\section{Architecture, Training, and Evaluation Details}
\label{appendix: architecture}
\subsection{Details of Encoders}

This section describes the network architecture of we used for language and observation encoding in BabyAI and ALFRED environments. In all the experiments, all models share the same encoders, detailed as follows, unless otherwise specified. 

\paragraph{Language Encoder} For the language signals, we use the Transformer model~\cite{vaswani2017attention} from the pre-trained CLIP network~\cite{radford2021learning} for language encoding in BabyAI and ALFRED experiments, freezing the entire model during training and adding an additional fully connected layer at the end for fine-tuning. 

\paragraph{Observation Encoder} 
Given the differing input structures of BabyAI and ALFRED, we employ separate observation encoders.

For \textbf{BabyAI}, we adopt the original visual encoding framework from BabyAI~\cite{hui2020babyai}, which integrates a Bag-of-Words embedding layer~\cite{mikolov2013efficient}, a convolution backbone, and a linear layer. The BOW module first turns the structural inputs with size $7\times7\times3$ into $7\times7\times256$ embeddings. The subsequent convolutional backbone processes these features through two sequential blocks and a max-pooling layer: each block contains a $3\times 3$ convolutional layer, followed by batch normalization and ReLU activation. The output is then processed with a $7\times7$ max-pooling layer. The features are then flattened and projected to 256 dimensions through a linear layer, producing a 256-dimensional vector as the observation encoder's final output. 
% , which combines Bag-of-Words encoding a two-layer CNN, FiLM, and LSTM.  
% We largely retain this encoding format but modify the FiLM layer to be applied after the LSTM, enabling it to encode historical states independently of instructions, $x_t=h_{w_{seq}}(x_{t-1},s_t,a_t)$. 
% We use 256 as the feature size for all the baselines and our method in BabyAI.

For \textbf{ALFRED}, we use the original encoding framework in ALFRED~\cite{shridhar2020alfred} for all implemented methods, which contains two sequential blocks: each block contains a $1\times 1$ convolutional layer, followed by batch normalization and ReLU activation. The features are then flattened and projected to 512 dimensions through a linear layer, producing a 512-dimensional vector as the observation encoder's final output.

\paragraph{FiLM and Sequence Encoder} We follow~\cite{hui2020babyai} to use FiLM~\cite{perez2018film} to fuse language and observation encodings through feature-wise affine transformations. For history encoding, all baseline methods employ a two-layer unidirectional LSTM to model temporal dependencies. 
% The sole exception is the DT, which replaces the LSTM with a causal self-attention architecture.

\subsection{Architecture Details of DAIL}
\label{subsec:arch}
The overview architecture of DAIL is shown in Figure \ref{fig:neuralnetwork}. We adopt the language and observation encoders described above to encode instructions and observations separately. 
To use the same sequence encoder for both trajectory-wise semantic alignment and history encoding, we modify the sequence encoder to process observation-action trajectories jointly and output history information $x_t$.
% The sequence model $h_{w_{seq}}$ is a two-layer unidirectional LSTM with a hidden dimension of the feature size. We apply 2 feed-forward layers with ReLU activations to obtain the discrete value distribution as output. 
% The pseudo-code for our complete method is provided in the Appendix~\ref{appendix:pseudocode}.

The computation process of state-action value $q(s_t,a_t,x_{t-1},l)$ is shown on the Left of Figure \ref{fig:neuralnetwork}, given instruction $l$, observation $s_t$. The outputs of the FiLM network are concatenated with the outputs of the sequence encoder, then processed through a Multi-Layer Perceptron (MLP) and a Softmax layer to generate the final value distribution with dimension $M$.
For Trajectory-Wise Semantic Alignment as shown on the Right of Figure~\ref{fig:neuralnetwork}, we derive embeddings from instruction $l$ and trajectory $\tau$. The trajectory embedding is represented by the sequence model's final output $x_{\tau}$.
\begin{figure*}[t]
    \centering
    \includegraphics[width=1.0\textwidth]{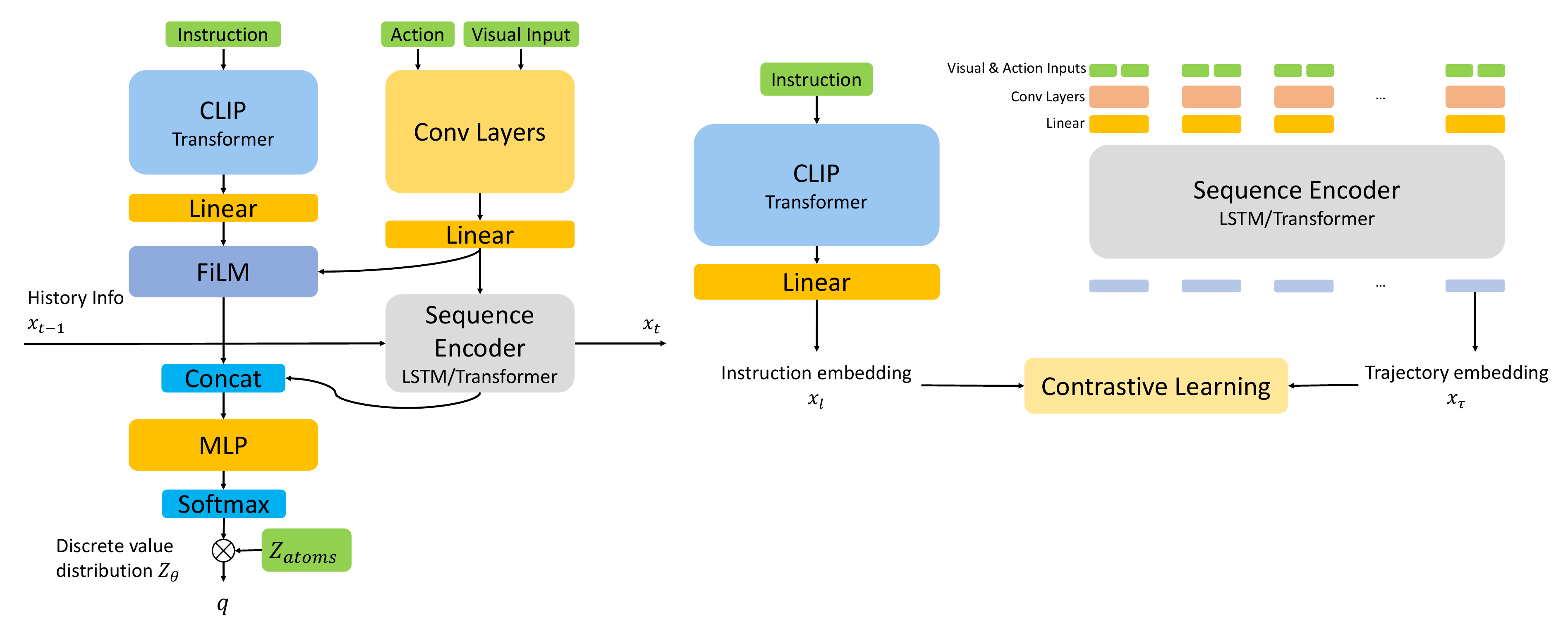} 
    \vspace{-20pt}
    \caption{
    Overview of our algorithm. \textbf{Left:} Computation process of state-action value. \textbf{Right:} trajectory-wise semantic alignment.
    }
    \label{fig:neuralnetwork}
\end{figure*}

% \begin{table}[h]
%     \centering
%     \caption{Parameters employed in experiments.}
%     \vskip 0.15in
%     \begin{tabular}{|c|c|}
%         \hline
%         Parameters & \multicolumn{1}{|c|}{Value} \\ \hline
%         Learning rate & $3e^{-4}$   \\ \hline
%         Batch size    & 64   \\ \hline
%         Feature size    & 256   \\ \hline
%     \end{tabular}
%     \label{table:instruction_example}
% \end{table}

\subsection{Training Details}

All models are implemented with PyTorch, and trained with a batch size of 64, using the Adam optimizer~\cite{kingma2014adam} at a learning rate of $3e-4$. All layers in the networks utilize PyTorch’s default weight initialization, and the network outputs fixed-dimensional embeddings suitable for downstream tasks. In BabyAI experiments, all methods were trained for 50 epochs over 3 seeds. And in the ALFRED experiments, all methods were trained for 20 epochs over 3 seeds following~\cite{shridhar2020alfred}.

As for DAIL, we fix $\alpha=2$ and $\lambda=0.2$ except for the toy experiment and ablation experiment of $\lambda$. We use $V_{MAX}=-V_{MIN}=20,M=51$ in all our experiments following~\cite{kumar2022offline}. 

\subsection{Baseline Details}
In this work, all the baselines share the same language encoder and similar observation encoder, and only differ in the decision module. In the following part, we introduce the decision modules of several baselines used in our paper:
% \begin{itemize}[leftmargin=2em]

\paragraph{GCBC:} language-conditioned behavior cloning with all data to maximize 
\begin{equation}
    \begin{aligned}
    \mathcal{J}_{\rm BC}(\pi_\theta)=\mathbb{E}_{(s_t,a_t,l)\sim \mathcal{D}, x_t\sim h_{w}}[\log\pi_\theta(a_t|x_t,l)]
    \end{aligned}
\end{equation}

\paragraph{BC-Z:} learning two task encodings from language and trajectory (video), and aligning them through similarity. We use cosine distance $D_{cos}(p,q)=\frac{p^T q}{|p||q|}$ to measure the similarity. We denote $(\tau,l)$ as a pair of trajectory and instruction.
\begin{equation}  
\begin{aligned}   
    \mathcal{J}_{\rm BC-Z}(\pi_\theta)=\mathbb{E}_{(s_t,a_t,l)\sim \mathcal{D}, x_t\sim h_{w}}[\log\pi_\theta(a_t|x_t,f_{\varphi}(l))]-\mathbb{E}_{(\tau, l)\sim \mathcal{D}}[D_{\cos}(q_{\phi}(\tau),f_{\varphi}(l))], 
 \end{aligned}
\end{equation}
where $q_\phi(\cdot)$ is the video encoder proposed by BC-Z to encode trajectory information, and $f_{\varphi}(\cdot)$ is the instruction encoder.

\paragraph{GRIF:} explicitly aligning the representations of language-conditioned tasks through contrastive learning with similarity measure $\mathcal{C}(s,g,l)=D_{cos}(f_{\varphi}(l), h_{\psi}(s_0,g))$, where $\varphi,\psi$ are learnable parameters of the language encoder and goal encoder respectively and $g$ is the last state of a trajectory. Positive data $(\tau^+, l^+)\sim p_{\mathcal{D}}(\cdot, \cdot)$ are uniformly sampled from the dataset, with $s^+, g^+$ being the start state and end state of $\tau^+$ respectively. Negative examples $s^-,g^-$ are the start state and end state of a randomly sampled trajectory from the dataset. Negative instruction $l^-\sim p_l(\cdot)$ is the instruction of another random trajectory. For each positive example, $k$ negative examples are sampled noted as $\{s_i^-,g_i^-\}_{i=1}^{k}$ and $\{l_i^-\}_{i=1}^k$.
\begin{equation}
    \begin{aligned}
    \mathcal{L}_{\rm lang\to goal}(\varphi, \psi)&=-\log \frac{\exp (\mathcal{C}(s^+, g^+, l^+)/\tau)}{\exp (\mathcal{C}(c^+,g^+,l^+)/\tau)+\sum_{i=1}^k\exp (\mathcal{C}(c^-_i,g^-_i,l^+)/\tau)}\\
    \mathcal{L}_{\rm goal\to lang}(\varphi, \psi)&=-\log \frac{\exp (\mathcal{C}(s^+, g^+, l^+)/\tau)}{\exp (\mathcal{C}(c^+,g^+,l^+)/\tau)+\sum_{i=1}^k\exp (\mathcal{C}(c^+_i,g^+_i,l^-)/\tau)}
    \end{aligned}
\end{equation}
where $\tau$ is the temperature parameter. We employ the last state $s_T$ in the trajectory as the goal state: $h_{\psi}(s_0,g)=h_{\psi}(s_0,s_T)$. Then the policy network is trained with behavior cloning by maximizing the likelihood of the actions:
\begin{equation}
    \begin{aligned}
    \mathcal{J}_{\rm GRIF}(\pi_\theta)=\mathbb{E}_{(s_t,a_t,l)\sim \mathcal{D}, x_t\sim h_{w}}[\log\pi_\theta(a_t|x_t,f_{\varphi}(l))]
    \end{aligned}
\end{equation}
We use the \textbf{GRIF(Joint)} setting to train the model~\cite{myers2023goal}.
 
\paragraph{CQL:} we implement CQL~(w/o distributional) based on DDPG, 
\begin{equation}
    \begin{aligned}
    \mathcal{J}_{CQL}(\pi_\theta)=\mathbb{E}_{(s_t,a_t,l)\sim \mathcal{D}, x_t\sim h_{w}}[Q(x_t,\pi_\theta(x_t,l), l)]
    \end{aligned}
\end{equation}
where Q function is learned by minimizing:
\begin{equation}
    \begin{aligned}
    \mathcal{L}_{CQL(\mathcal{H})}(\theta)=&\mathbb{E}_{(s_t,a_t,r_t,s_{t+1},l)\sim \mathcal{D}, (x_t,x_{t+1})\sim h_{w}}[(Q_\theta(x_t,a_t,l)-\mathcal{B}^\pi Q_\theta(x_t,a_t,l))^2]+  \\
    &\alpha \mathbb{E}_{(s_t,l)\sim \mathcal{D}, x_t\sim h_{w}}[\mathrm{log}\sum_a\mathrm{exp}(Q_\theta(x_t,a,l))-\mathbb{E}_{a\sim\hat{\pi}_{\beta}(a|s_t,l)}[Q_\theta(x_t,a, l)]]
    \end{aligned}    
\end{equation}
where $\hat{\pi}_{\beta}$ is the behavior policy, $\mathcal{B}^\pi$ is the Bellman operator, and the balance ratio $\alpha=2$ in our experiment.

\paragraph{IQL:} training an additional value network and extracting policy through advantage weighted regression.
\begin{equation}
    \begin{aligned}
    \mathcal{L}_V(\psi)&=\mathbb{E}_{(s_t,a_t,l)\sim \mathcal{D}, x_t\sim h_{w}}[L_2^\tau(Q_{\hat{\theta}} (x_t,a_t,l)-V_\psi(x_t,l)]\\
    \mathcal{L}_Q(\phi)&=\mathbb{E}_{(s_t,a_t,r_t,s_{t+1},l)\sim \mathcal{D}, (x_t,x_{t+1})\sim h_{w}}[(r_t+\gamma V_\psi(x_{t+1},l)-Q_\phi(x_t,a_t,l)^2]\\
    \mathcal{J}(\pi_\theta)&=\mathbb{E}_{(s_t,a_t,l)\sim \mathcal{D}, x_t\sim h_{w}}[\exp (\beta Q_\phi(x_t,a_t,l)-V_\psi(x_t, l))\log \pi_\theta(a|x_t, l)]
    \end{aligned}
\end{equation}
The expectile $\tau=0.7$, $\beta=5$. We follow the authors' suggestions and subtract 1 from the reward if it equals 0.

% \paragraph{DT:} abstracting RL as a sequence modeling problem. We encode the observation and language as BabyAI 1.1, then enter it into the decision transformer as expanded states. We set the context length to 80 in BabyAI, and 12 in ALFRED separately. 

% \begin{equation}
%     \begin{aligned}
%     s_t'&=f(s_t, l),\tau'=(R_1,s_1', a_1,R_2,s_2',a_2,\dots,R_T,s_T,a_T) \\
%     \tau_{\rm pred}&=\mathsf{causal\ Transformer}(\tau'), \mathcal{J}_{DT}(\theta)=\mathbb{E}_{\tau'\sim \mathcal{D}}\log p(\tau_{\rm pred}[\text{action}]_t=a_t) 
%     \end{aligned}
% \end{equation}
To follow the original IL setting and simplicity, we merely use observation without action information in the history state encoding in BC and IQL. Our primary experiments show that it has little impact on the results.
We also investigate recent approaches of language-conditioned IL such as LLfP~\cite{lynch2020language} and R3M~\cite{nair2022r3m}, but they perform poorly in prior experiments, so only BC-Z is chosen as the representative in the final baselines. 

% For our study, we use DQN\cite{mnih2015human} architecture, and the greedy algorithm is used to choose the output action by computing $a^*=\mathrm{argmax}_{a} p(s, a|l)^Tv_{atoms}$.
\clearpage
\section{Visualization Supplementary Results}
\label{appendix:vis}

We apply t-SNE visualization in different algorithms and task types to show that our method substantially improves task representation on offline language-conditioned RL. Beyond the main text, here are some supplementary visualization results.
\begin{figure}[h]
    \centering
    \includegraphics[width=\textwidth]{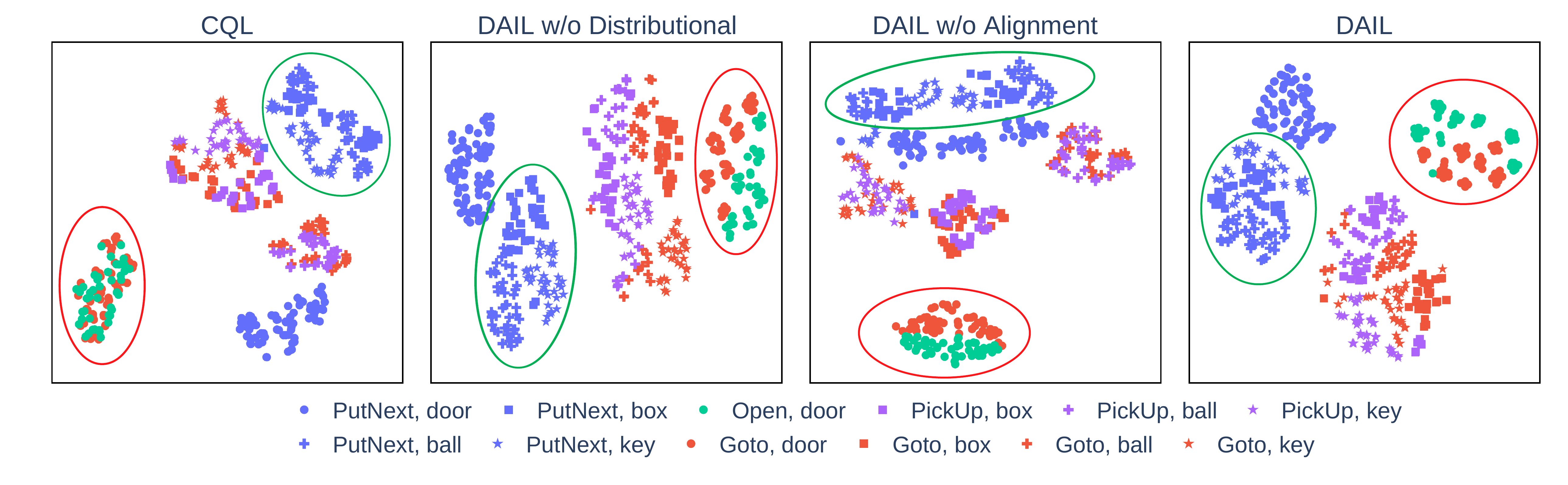}
    \vspace{-25pt}
    \caption{The t-SNE visualization of instructions from various tasks in BabyAI for different algorithms. The figure distinguishes between different \textbf{task categories}~(e.g., PutNext) and \textbf{target object types}~(e.g., box), using marker colors and shapes to represent each separately. }
    \label{fig:cql-alltask}
\end{figure}

\begin{figure}[h]
    \centering
    \includegraphics[width=\textwidth]{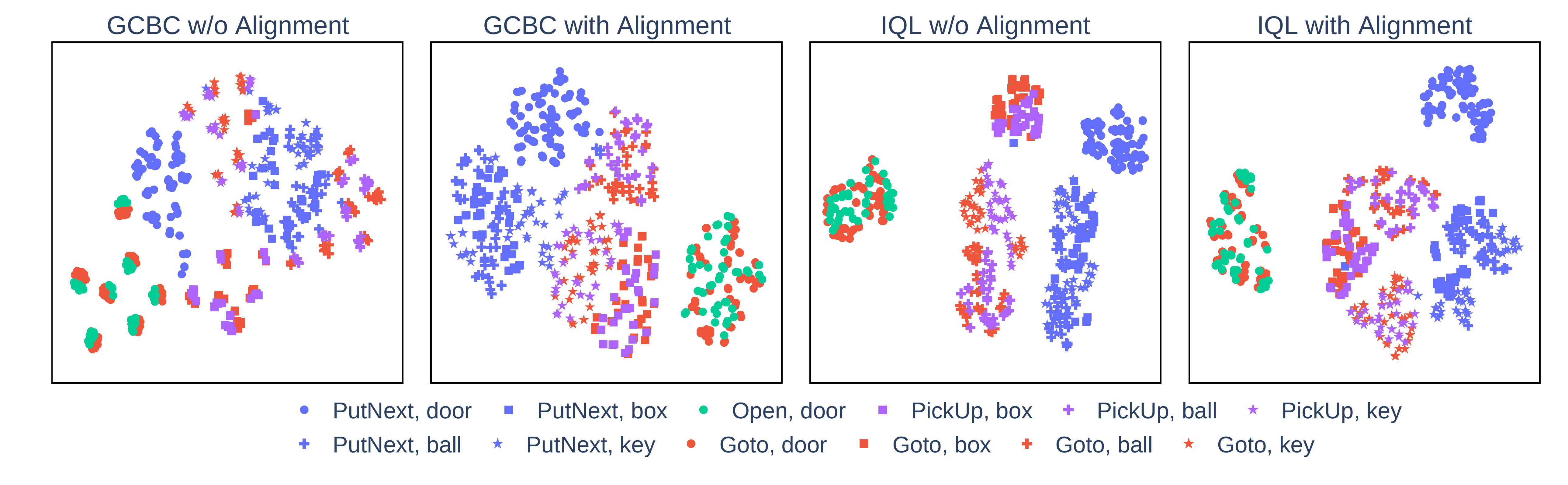}
    \vspace{-25pt}
    \caption{The t-SNE visualization of instructions from various tasks in BabyAI for different algorithms. The figure distinguishes between different \textbf{task categories}~(e.g., PutNext) and \textbf{target object types}~(e.g., box), using marker colors and shapes to represent each separately.}
    \label{fig:bciql-alltask}
\end{figure}

\textbf{Overall task representation}
As introduced in Section \ref{sec:exp}, in BabyAI $\verb|SynthLoc|$ tasks, there are four main categories of tasks: \texttt{Goto}, \texttt{PickUp}, \texttt{PutNext}, and \texttt{Open}. We sample tasks from all four categories to visualize in Figure \ref{fig:cql-alltask}. Our method demonstrates superior task embedding capabilities compared to other approaches in the \texttt{Goto}, \texttt{PickUp}, and \texttt{Open}—effectively reducing confusion as highlighted by the red circles. Some degree of confusion remains inevitable in the \texttt{PutNext} tasks, due to their complexity and variability (as indicated by the green circles).

We also visualize representations from GCBC and IQL (with and w/o alignment) in Figure \ref{fig:bciql-alltask}. GCBC shows reliable task representation, but confuses tasks from \textcolor{myOrange}{\texttt{Goto}} and \textcolor{myPurple}{\texttt{PickUp}}, ``Open, door'' \circlesymbolg and ``Goto, door'' \circlesymbolo, where their embeddings are tightly gathered into small clusters. Our further results in Figure \ref{fig:bciql-pickupgoto} show that each cluster represents a specific color. Similarly, alignment significantly eases this issue by separating each instruction. 
IQL shows similar results to CQL, while alignment has a more significant influence on CQL. This result explains why simple GCBC shows comparable performance in our settings while vanilla offline RL fails due to confusion in task encoding.

\begin{figure}[h]
    \centering
    \includegraphics[width=\textwidth]{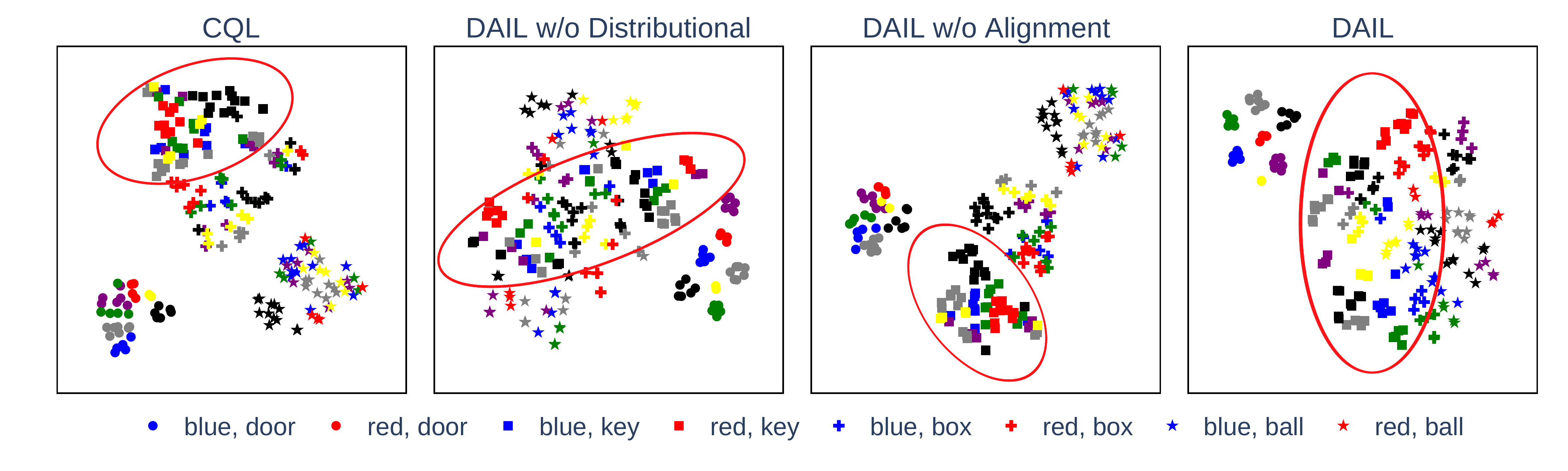}
    \vspace{-25pt}
    \caption{The t-SNE visualization of instructions from the same task categories with more detailed distinction. The figure distinguishes between different \textbf{target object types}~(e.g., door) and \textbf{target object colors}~(e.g., blue), using marker colors and shapes to represent each separately. For example, ``go to the red door'' corresponds to \textcolor{red}{$\bullet$};. ``go to a red ball behind you'' corresponds to \textcolor{red}{$\star$}, ``pick up the ball in front of you'' corresponds to $\star$.}
    \label{fig:cql-pickupgoto}
\end{figure}
\begin{figure}[h]
    \centering
    \includegraphics[width=\textwidth]{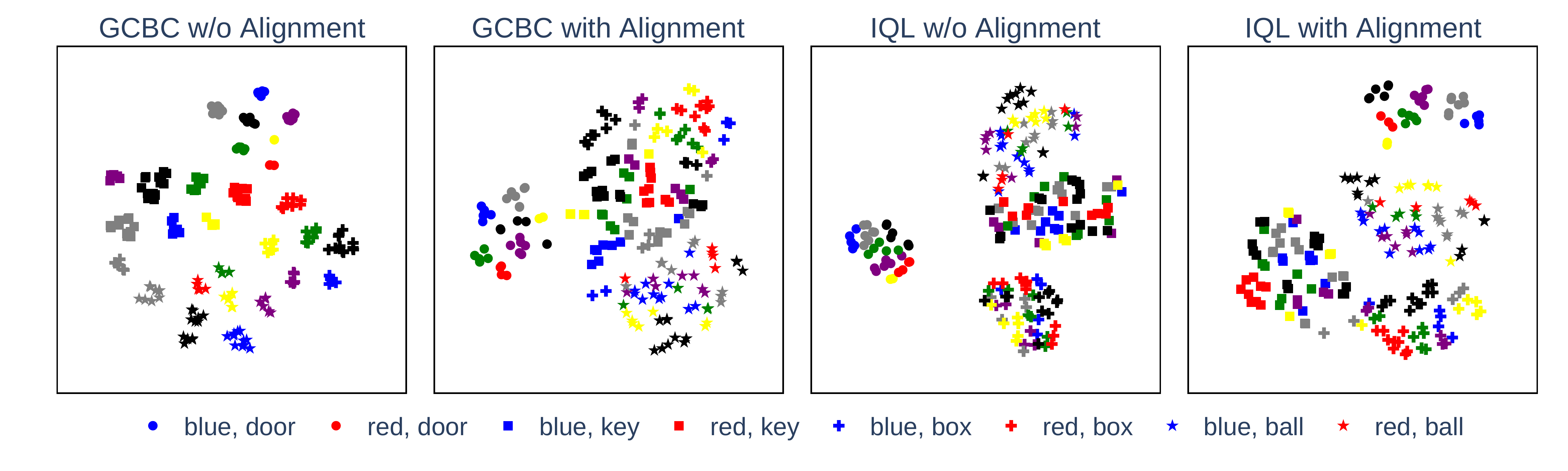}
    \vspace{-25pt}
    \caption{The t-SNE visualization of instructions from the same task categories~(\texttt{PickUp} and \texttt{Goto}) with more detailed distinction. 
    The figure distinguishes between different \textbf{target object types}~(e.g., door) and \textbf{target object colors}~(e.g., blue), using marker colors and shapes to represent each separately. The legend has the same meaning as in Figure \ref{fig:cql-pickupgoto}.}
    \label{fig:bciql-pickupgoto}
\end{figure}

\paragraph{Task \texttt{PickUp} and \texttt{Goto}}
We take a closer look at tasks in \texttt{GoTo} and \texttt{PickUp} that only have one target object. The task representations of DAIL and ablation algorithms are shown in Figure \ref{fig:cql-pickupgoto}, and GCBC and IQL (with and w/o alignment) are shown in Figure \ref{fig:bciql-pickupgoto}. The color of the markers indicates that of the target objects other than black markers.
As illustrated in Figure \ref{fig:visual}, our method further subdivides tasks (for example, \circlesymbolb and \circlesymbolo) into smaller clusters. Upon closer observation, these smaller clusters correspond to different target object colors. Similarly, CQL performs poorly in object color recognition, while distributional representation and semantic alignment substantially help task discrimination. As illustrated by the red circles, when the target object type is key, only our method succeeds in separating embeddings of different target colors while forming clusters for targets of the same color. In contrast, other methods tend to produce entangled representations, leading to less distinguishable task embeddings.

GCBC yields strong results in this scenario, whereas IQL relatively performs poorly. 
% GCBC clusters markers of the same shape and color, which implies the same target object in tasks. 
However, excessive overlap in GCBC suggests a confusion between the \texttt{PickUp} and \texttt{Goto} tasks. IQL can distinguish between different types of targets (different shapes of markers) but fails in differentiating target colors (for example, \textcolor{blue}{\rule{1ex}{1ex}} \textcolor{green}{\rule{1ex}{1ex}} \textcolor{red}{\rule{1ex}{1ex}} \dots).
After our alignment method was applied, its performance significantly improved. 

\paragraph{Representation with instruction texts}
We add some instruction texts to the representation map to better demonstrate the language instructions' representation results. 
We draw around 300 instructions from BabyAI \texttt{SynthLoc} level to visualize and sample around 20 tasks and include their original text in the figure, displayed to the right of the corresponding marker.
Figure \ref{fig:CQL-text} shows the detailed representation result of vanilla CQL, and Figure \ref{fig:CQL+C51-text} shows that of our method DAIL. 
The red dashed circles highlight the ``Goto, door \circlesymbolb'' and ``Open, door \textcolor{myGreen}{\textbf{+}}'' tasks. Our method DAIL successfully separates the two task categories and further organizes them into clusters (based on colors, see Figure \ref{fig:cql-pickupgoto}) while CQL fails. Similarly, the green dashed circles denote the ``PickUp, key \circlesymbolo'' and ``Goto, key \textcolor{myBlue}{\rule{1ex}{1ex}}'' tasks. While CQL fails to disentangle the \textcolor{myOrange}{\texttt{PickUp}} and \textcolor{myBlue}{\texttt{Goto}} task types, our method achieves a clear separation between them.
% This comparison provides a more detailed and intuitive demonstration of our method's ability to distinguish tasks with different target object types and colors and to effectively cluster similar instructions on a small scale.

\begin{figure}
    \centering
    \includegraphics[width=0.8\linewidth]{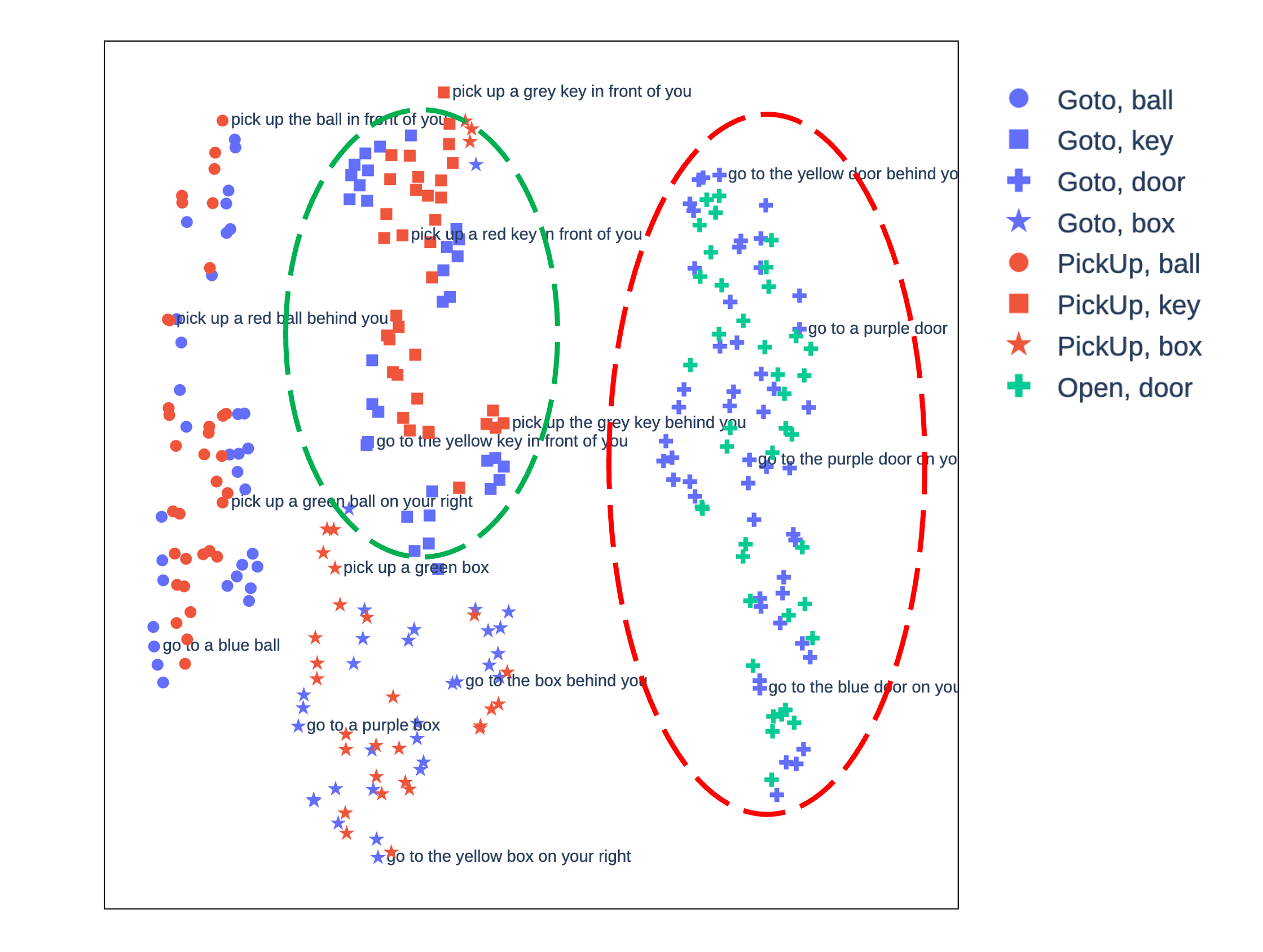}
    \vspace{-10pt}
    \caption{The t-SNE visualization of instructions with text annotations from \texttt{Open}, \texttt{Goto}, \texttt{PickUp} in BabyAI for CQL. The figure distinguishes between different \textbf{task categories}~(e.g., PickUp) and \textbf{target object types}~(e.g., box), using marker colors and shapes to represent each separately.}
    \label{fig:CQL-text}
\end{figure}
\begin{figure}
    \centering
    \includegraphics[width=0.8\linewidth]{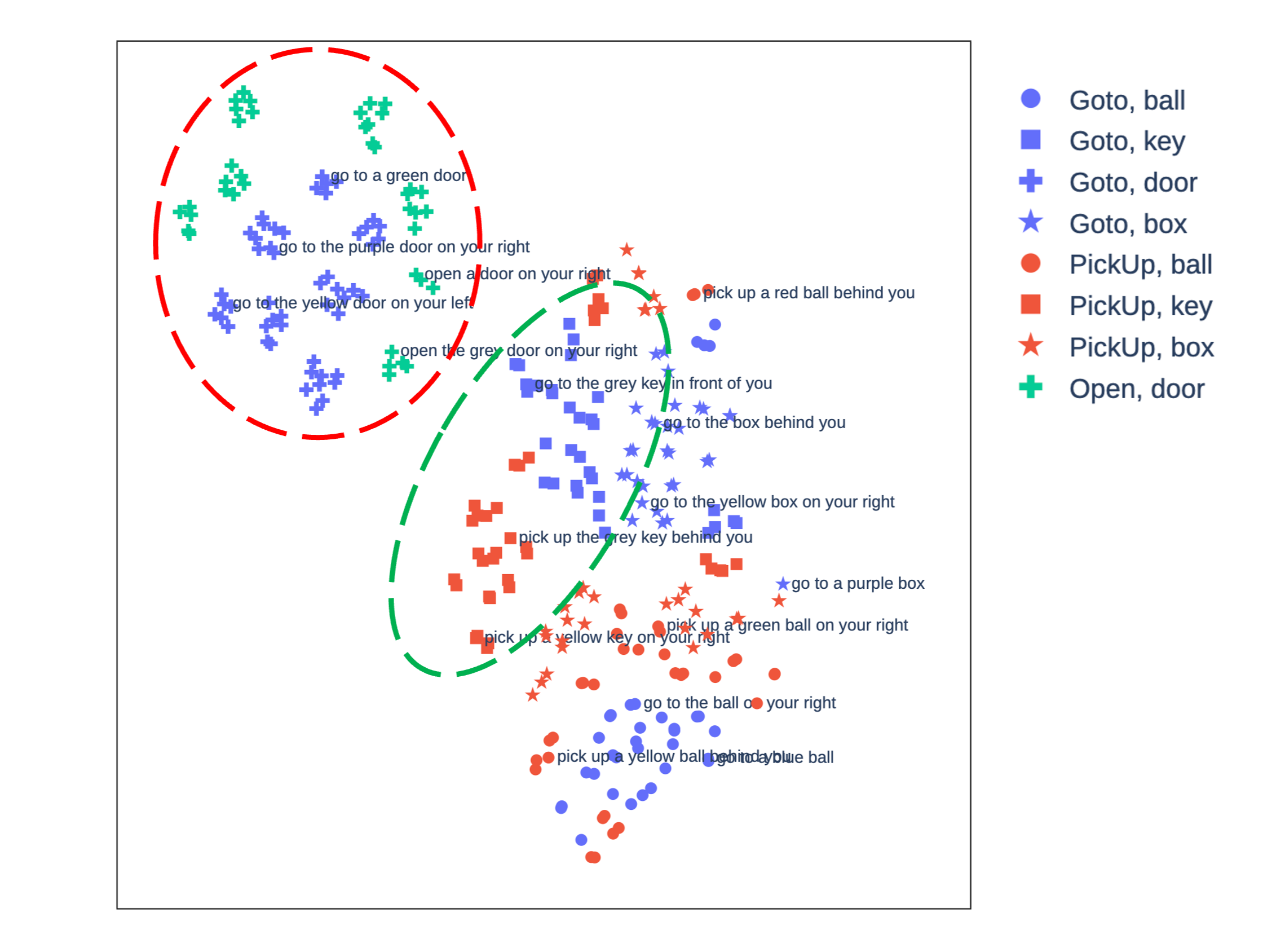}
    \vspace{-10pt}
    \caption{The t-SNE visualization of instructions with text annotations from \texttt{Open}, \texttt{Goto}, \texttt{PickUp} in BabyAI for our method. The figure distinguishes between different \textbf{task categories}~(e.g., PickUp) and \textbf{target object types}~(e.g., box), using marker colors and shapes to represent each separately. }
    \label{fig:CQL+C51-text}
\end{figure}
\clearpage

\section{Additional Results}
\label{appendix:add-res}

% \paragraph{Training curves on the high-quality dataset of BabyAI}

% In Figure~\ref{fig:babyai_curve}, we illustrate the training curves in the \texttt{ALL} and \texttt{PutNext} tasks using the high-quality dataset. The results are smoothed by applying an exponential moving average, with $span=5$. Due to the different update methods used by DT, only its final performance is reported. The experimental results indicate that our method has the fastest learning efficiency, particularly notable in the \texttt{PutNext} task category.

% \begin{figure}[h]
%     \centering
%     \includegraphics[width=\textwidth]{figures/training curve/Training Curve (title).png} 
%     \vspace{-20pt}
%     \caption{Training curves in the \texttt{All} and \texttt{PutNext} tasks on BabyAI using the high-quality dataset. The success rates are evaluated every 1000 time steps over 3 seeds. }
%     \centering
%     \label{fig:babyai_curve}
% \end{figure}

\paragraph{Results on the hight-quality dataset of BabyAI}
The success rates of in-distribution and out-of-distribution tasks on the high-quality dataset are shown in Table~\ref{table:babyai-high}. in both in-distribution and out-of-distribution tasks. 
% The experimental results reveal that RL algorithms like CQL and IQL underperform compared to imitation learning methods like GCBC and GRIF, which we attribute to the adverse impact of task ambiguity on RL-based approaches as discussed in Theorem~\ref{thm: sample-complexity}.
Our method demonstrates significant advantages over other approaches on out-of-distribution tasks, particularly achieving substantial performance improvements on complex tasks such as $\texttt{PutNext}$ compared to the baseline RL method CQL~(49.1\% vs. 27.6\%). 
Vanilla offline RL algorithms like CQL and IQL underperform compared to imitation learning methods like GCBC, which we attribute to the adverse impact of task ambiguity on RL-based approaches as discussed in Theorem~\ref{thm: sample-complexity}.
On the other hand, modified algorithms designed for language-conditioned IL (BC-Z and GRIF) perform poorly under our setting. This is primarily because their contrastive learning objectives are not robust in the presence of noisy or suboptimal data. In contrast, our alignment-based approach, built on an offline RL framework, maintains strong performance.

\begin{table*}[h]
\caption{Success rate of in-distribution tasks and out-of-distribution BabyAI tasks. Each score is evaluated over 3 seeds.}
\begin{center}
    \begin{tabular}{l|c|c|c|c|c} 
    \toprule
      Algorithm & Open & Goto & PickUp & PutNext & All \\
      \midrule
      \multicolumn{6}{c}{In Distribution} \\
      \midrule
      GCBC & 96.9\textpm0.8  & 91.8\textpm1.1 & 85.6\textpm0.0 & 27.6\textpm3.3 & 79.1\textpm1.3 \\

      BC-Z & 96.3\textpm 0.7 & 77.5\textpm 1.0 & 49.9\textpm 4.4 & 14.2\textpm0.7 & 64.0\textpm1.8\\
      GRIF & 96.6\textpm0.8 & 89.4\textpm2.5 & 87.6\textpm0.1 & 27.7\textpm3.7 &  78.6\textpm2.5 \\
      % DT  & 92.4\textpm3.2 & 79.3\textpm1.6 & 63.8\textpm1.0 & 24.3\textpm2.8 & 68.6\textpm2.6  \\
      IQL  & \textbf{98.2}\textpm0.4 & 87.9\textpm1.4 & 73.7\textpm1.1 & 26.2\textpm3.5 & 75.2\textpm0.7  \\
      CQL & \textbf{98.7}\textpm0.3 & 92.2\textpm0.9  & 83.8\textpm1.7  & 25.6\textpm2.5 & 78.1\textpm1.6  \\
      DAIL (ours)  & 97.2\textpm0.2 & \textbf{96.5}\textpm1.4  & \textbf{94.9}\textpm1.4  & \textbf{57.9}\textpm0.9 & \textbf{89.2}\textpm0.5  \\
      \midrule
      \multicolumn{6}{c}{Out of Distribution} \\
      \midrule
      GCBC & 94.4\textpm2.5 & 90.3\textpm1.6 & 78.4\textpm2.1 & 27.4\textpm1.6 & 74.1\textpm0.7 \\
      BC-Z & 93.7\textpm1.0 & 76.9\textpm3.0 & 45.4\textpm1.5 & 11.2\textpm3.3 & 57.9\textpm1.8 \\
      GRIF & 95.9\textpm1.7 & 88.8\textpm2.6 & 75.6\textpm3.9 & 22.5\textpm2.7 & 71.2\textpm2.6  \\
      % DT  & 94.2\textpm1.5 & 76.4\textpm1.0 & 59.5\textpm1.2 & 25.0\textpm1.3 & 65.4\textpm0.8  \\
      IQL  & 98.0\textpm0.4& 86.1\textpm1.2 & 70.4\textpm3.6 & 21.4\textpm3.1 & 69.7\textpm2.3  \\
      CQL & 98.8\textpm0.5 & 88.9\textpm2.1  & 71.9\textpm2.2 & 27.6\textpm0.8 & 72.6\textpm0.4  \\
      DAIL (ours)  & \textbf{99.0}\textpm0.2 & \textbf{91.3}\textpm1.0  & \textbf{87.6}\textpm2.0 & \textbf{49.1}\textpm1.8 & \textbf{81.7}\textpm1.3  \\
      \bottomrule
    \end{tabular}
\end{center}
\label{table:babyai-high}
\end{table*}

\paragraph{Results on the medium-quality dataset of BabyAI}
The success rates of in-distribution and out-of-distribution tasks on the medium-quality dataset are shown in Table~\ref{table:babyai-medium}.
Due to the lower proportion of successful trajectories, learning in this dataset is more challenging. 
As a result, all methods show a significant decline in performance compared to results on the high-quality dataset~(Table~\ref{table: babyai}). 
However, our method still achieves optimal results, especially in the $\verb|PutNext|$ task category.
\begin{table*}[h]
    \caption{Success rate of on the medium-quality dataset. Each score is evaluated over 3 seeds.}
    \begin{center}
        \begin{tabular}{c|c|c|c|c}
        \toprule
            \multirow{2}{*}{Algorithm} & \multicolumn{2}{c|}{In Distribution} & \multicolumn{2}{c}{Out of Distribution}\\
            \cmidrule(lr){2-3} \cmidrule(lr){4-5}
            & PutNext & All & PutNext & All \\
            \midrule
            % GCBC-F  & 26.7 & 77.5 & 22.6 & 69.6 \\
            GCBC  & 15.6\textpm1.5 & 58.0\textpm2.2 & 10.5\textpm1.3 & 52.6\textpm2.3\\
            BC-Z  & 4.4\textpm1.0 & 54.2\textpm0.2 & 5.0\textpm1.6 & 49.7\textpm1.4\\
            GRIF  & 7.0\textpm1.5 & 61.4\textpm0.2 & 4.9\textpm2.5 &54.7\textpm0.6 \\
            % DT  & 16.4\textpm1.8 & 62.6\textpm0.2 & 16.6\textpm2.4 & 57.7\textpm1.7 \\
            IQL  & 17.7\textpm2.8 & 67.3\textpm0.5 & 12.3\textpm0.9 & 61.4\textpm0.0\\
            CQL  & 13.5\textpm1.8 & 69.2\textpm0.7 & 14.5\textpm1.9 & 63.4\textpm1.3\\
            Ours  & \textbf{32.8}\textpm2.7 & \textbf{81.3}\textpm0.7 & \textbf{26.4}\textpm0.6 & \textbf{73.7}\textpm0.6\\
        \bottomrule
        \end{tabular}
    \end{center}
\label{table:babyai-medium}
\end{table*}

\paragraph{Ablation of components}
To study the contribution of each component in our learning framework, we conduct the following ablation study.
We compare the performance of algorithms that only apply trajectory-wise alignment or distributional language-guided policy alone with our method on \texttt{SynthLoc}. 
% The experimental results in Table~\ref{table:ablation} and training curves in Figure~\ref{fig:babyai_curve} show that both modules significantly improve the performance over vanilla CQL on in-distribution and out-of-distribution tasks.
The experimental results in Figure~\ref{fig:babyai_curve} show that both modules significantly improve the performance over vanilla CQL on in-distribution and out-of-distribution tasks.
Further, combining both components can achieve the best performance compared to other approaches.

We further evaluate the sample efficiency of each method by calculating the Area Under the Curve (AUC) for the success rates of their learning curves, and present the AUC comparisons of in-distribution learning curves at the key milestones of 10 and 20 epochs in Table~\ref{table:auc}. Statistical analysis confirms that at 20 epochs, DAIL holds a significant lead over the other three ablation models, as demonstrated by Kolmogorov-Smirnov tests on the AUC results (p = 0.004).

% \begin{table*}[h]
%     \caption{Ablation results for components of our method. Each score is evaluated over 3 random seed.}
%     \begin{center}
%         \begin{tabular}{c|c|c|c|c}
%         \toprule
%         \multirow{2}{*}{Algorithm} & \multicolumn{2}{c|}{In Distribution} & \multicolumn{2}{c}{Out of Distribution}\\
%         \cmidrule(lr){2-3} \cmidrule(lr){4-5}
%             & PutNext & All & PutNext & All \\
%         % \midrule
%         %     GCBC-F  & 38.0 & 83.3 & 35.6 & 78.9\\
%         %     \makecell{GCBC-F+Alignment} & 44.9 & 81.7 & 41.2 & 77.6 \\
%         \midrule
%         %     IQL  & 26.2\textpm3.5 & 75.2\textpm0.7 & 21.4\textpm3.1 & 69.7\textpm2.3 \\
%         %     \makecell{IQL+Alignment} & 30.6\textpm5.8 & 79.0\textpm1.1 & 25.8\textpm2.8 & 70.6\textpm1.6 \\
%         % \midrule
%             CQL  & 25.6\textpm2.5 & 78.1\textpm1.6 & 27.6\textpm0.8 & 72.6\textpm0.4 \\
%             \makecell{DAIL w/o Distributional} & 39.6\textpm0.6 & 83.3\textpm0.2 & 39.3\textpm1.7 & 77.3\textpm0.8\\
%             \makecell{DAIL w/o Alignment} & 39.1\textpm1.0 & 82.2\textpm1.3 & 32.0\textpm0.9 & 75.1\textpm1.6\\
%         \midrule
%             DAIL  & \textbf{57.9}\textpm0.9 & \textbf{89.2}\textpm0.5 & \textbf{49.1}\textpm1.8 & \textbf{81.7}\textpm1.3\\
%         \bottomrule
%         \end{tabular}
%     \end{center}
%     \label{table:ablation}
% \end{table*}
\begin{figure}[h]
    \centering
    \includegraphics[width=\textwidth]{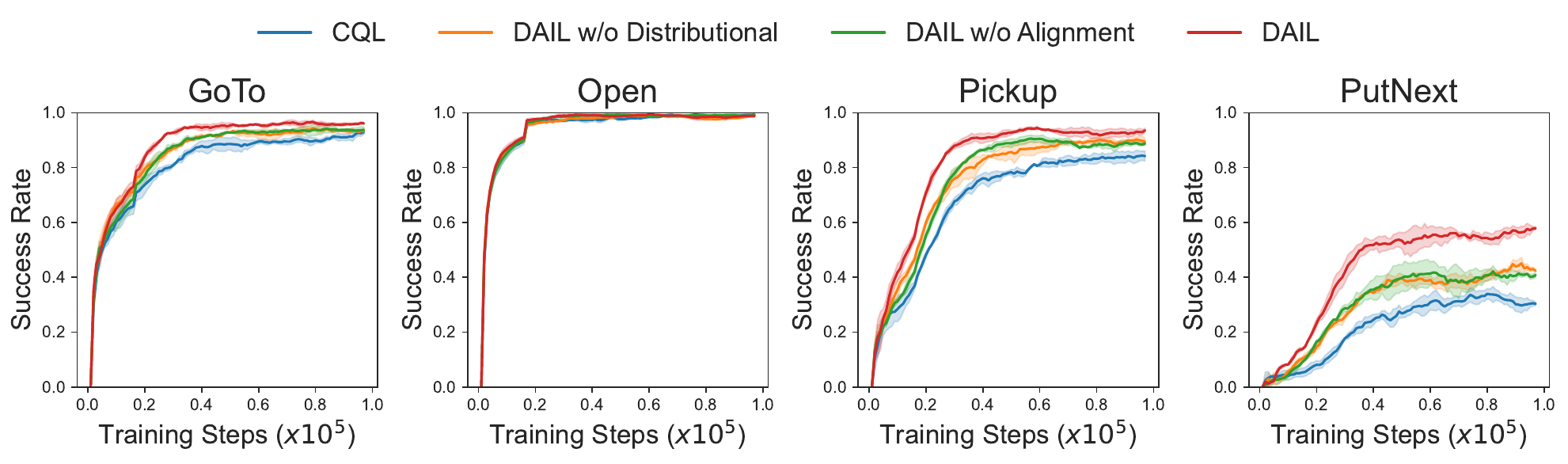}
    \caption{Ablation experiments on BabyAI tasks. The success rates are evaluated over 3 seeds. }
    \centering
    \label{fig:babyai_curve}
\end{figure}

\begin{table}
        \caption{Ablation results of AUC on the high-quality dataset. Each score is evaluated over 3 seeds.}
    \label{table:auc}
    \begin{center}
        \begin{tabular}{c|c|c}
        \toprule
            Algorithm & 10 & 20 \\
            \midrule
            CQL & 38,385.3\textpm853.5 & 81,876.0\textpm848.0 \\
            DAIL w/o Alignment & 40,374.7\textpm601.3 & 85,247.1\textpm540.2  \\
            DAIL w/o Distributional & 40,752.9\textpm248.7 & 85,416.5\textpm543.6 \\
            DAIL & \textbf{42,370.6\textpm404.8} & \textbf{88,450.5\textpm291.6} \\
        \bottomrule
        \end{tabular}

    \end{center}
\end{table}

\paragraph{Quantitative measure of clustering.} To quantitatively demonstrate the impact of different components in DAIL on representation clustering, we measure the clustering quality with the Silhouette score\cite{rousseeuw1987silhouettes}, and present the result in Table \ref{table:clustering}. The labels required to calculate the score are defined as follows: two instructions share the same label if they require the agent to perform the same action on the same kind of object with the same color as the final target object. The distance between language embeddings is measured using cosine distance. 

The results demonstrate that in the All-task setting, the clustering metric of CQL even exhibits negative values, indicating pronounced task confusion. In contrast, both proposed methods evidently improve clustering performance.

\begin{table*}[h]
    \caption{Silhouette score of in-distribution BabyAI tasks.}
    \begin{center}
        \begin{tabular}{c|c|c}
        \toprule
            Algorithm & All & PutNext\\
            \midrule
            % GCBC-F  & 26.7 & 77.5 & 22.6 & 69.6 \\
            CQL & -0.030\textpm0.005 & 0.004\textpm0.007 \\
            DAIL w/o Alignment & 0.024\textpm0.016 &  0.048\textpm0.008 \\
            DAIL w/o Distributional & 0.110\textpm0.019 & 0.088\textpm0.020 \\
            DAIL & \textbf{0.127\textpm0.013} & \textbf{0.107\textpm0.012} \\
        \bottomrule
        \end{tabular}
    \end{center}
\label{table:clustering}
\end{table*}

% \begin{table}[h]
%     \begin{center}
%         \caption{}
%         \vskip 0.15in
%         \begin{tabular}{c|c|c|c|c|c|c|c|c|c}
%         \toprule
%             \multirow{2}{*}{Algorithm} & \multirow{2}{*}{Task Category} & \multicolumn{8}{c|}{$\lamda$}\\
%            \midrule
%             &   &   &  0 & 0.01 & 0.1 & 0.2 & 0.5 & 1 & 2 & 5\\
%            \midrule
%            & Open    &  98.0  &  98.3    &  98.5   & 97.2    & 99.1 & 98.9 & 98.6 & 97.7\\
%            & Goto    & 90.9 & 94.6 & 96.1 & 96.5 & 94.9 & 94.7 & 94.3 &  92.2\\
%            & PickUp    & 89.0 & 87.5 & 91.1 & 94.9 & 90.9 & 91.0 & 90.3 & 89.2 \\
%            &PutNext    & 39.1 & 37.4 & 44.9 & 57.9 & 52.3 & 53.7 & 49.4 & 43.4\\
%            &All.    & 82.2 & 82.9 & 85.2 & 89.2 & 87.0 & 87.2 & 86.4 & 83.8\\
%            \midrule
%            &Open    & 96.4   &  98.4    &  98.5   &  99.0   & 96.8 & 98.1 & 98.0 & 97.4\\
%            &Goto    & 89.8 & 90.3 & 93.4 & 91.3 & 93.3 & 92.2 & 92.1 & 92.3 \\
%            &PickUp    & 79.0 & 76.8 & 81.9 & 87.6 & 86.4 & 88.3 & 84.6 & 81.7 \\
%            &PutNext    & 32.0 & 35.6 & 45.6 & 49.1 & 46.2 & 45.4 & 41.5 & 41.6 \\
%            &All.    & 75.1 & 76.0 & 79.9 & 81.7 & 80.9 & 81.1 & 79.1 & 79.4 \\
%         \bottomrule
%         \end{tabular}
%     \end{center}
% \end{table}

\paragraph{Ablation of $\alpha$}
In Equation~\ref{eq: total loss}, $\alpha$ is the weight of the CQL loss.
The ablation is done on the high-quality dataset in BabyAI tasks with various $\alpha$. 
The experimental results in Table~\ref{table:ablation_alpha} indicates that a wide range of $\alpha$ values (from 0.5 to 2) yield comparable performance, which drops off at the extremes ($\alpha=0.2$ and $\alpha=5$). We therefore recommend setting $\alpha$ between 0.5 and 2.

% \paragraph{Ablation of $\lambda$ and $\alpha$}
% In Equation~\ref{eq: total loss}, $\lambda$ and $\alpha$ are the weights of the alignment loss.
% For this reason, we evaluate the choice of $\lambda$ and $\alpha$. The ablation is done on the high-quality dataset in BabyAI tasks with various $\lambda$ and $\alpha$. 
% The experimental results in Figure~\ref{fig:lambda} show that there is effectively no obvious performance difference between $\lambda=0.2$ and $\lambda=1$. The performance begins to degrade when the influence of the loss is either too small~($\lambda=0.01$) or too large~($\lambda=2$). Therefore, we recommend choosing a value between 0.2 and 1.
% \begin{figure}[h]
%     \centering
%     \includegraphics[width=0.5\textwidth]{figures/ablation_stderr.pdf} 
%     \caption{Percent difference of the performance of an ablation over $\lambda$, compared to the average results of all $\lambda$s evaluated.}
%     \label{fig:lambda}
% \end{figure}

\begin{table*}[h]
    \caption{Ablation experimental results on $\alpha$. Each score is evaluated over 3 seeds.}
    \begin{center}
        \begin{tabular}{c|c|c|c|c}
        \toprule
        \multirow{2}{*}{$\alpha$} & \multicolumn{2}{c|}{In Distribution} & \multicolumn{2}{c}{Out of Distribution}\\
        \cmidrule(lr){2-3} \cmidrule(lr){4-5}
            & PutNext & All & PutNext & All \\
        % \midrule
        %     GCBC-F  & 38.0 & 83.3 & 35.6 & 78.9\\
        %     \makecell{GCBC-F+Alignment} & 44.9 & 81.7 & 41.2 & 77.6 \\
        \midrule
        %     IQL  & 26.2\textpm3.5 & 75.2\textpm0.7 & 21.4\textpm3.1 & 69.7\textpm2.3 \\
        %     \makecell{IQL+Alignment} & 30.6\textpm5.8 & 79.0\textpm1.1 & 25.8\textpm2.8 & 70.6\textpm1.6 \\
        % \midrule
            0.2 & 44.0\textpm3.6 & 83.7\textpm0.7 & 40.8\textpm5.1 & 78.4\textpm1.7 \\
            0.5 & 57.4\textpm2.2 & 88.1\textpm0.3 & 50.7\textpm2.7 & 83.1\textpm0.8\\
            1 & 56.2\textpm3.6 & 87.7\textpm1.3 & \textbf{50.8\textpm3.7} & \textbf{84.1\textpm0.7}\\
            2 & \textbf{57.9\textpm0.9} & \textbf{89.2\textpm0.5} & 49.1\textpm1.8 & 81.7\textpm1.3\\
            5 & 44.8\textpm3.2 & 85.2\textpm0.2 & 37.1\textpm5.1 & 77.4\textpm1.7\\
            10 & 35.7\textpm2.3 & 82.8\textpm1.1 & 39.1\textpm4.4 & 77.2\textpm0.7\\
        \midrule
        \bottomrule
        \end{tabular}
    \end{center}
    \label{table:ablation_alpha}
\end{table*}

% \begin{table*}[ht]
% \begin{center}
%     \caption{Success rate of our methods with varied $\lambda$ on high-quality dataset for all task categories. All values are percentages.}
%     \vskip 0.15in
%     \begin{tabular}{c|c|c|c|c|c|c|c|c|c|c} 
%     \toprule
%       \multirow{2}{*}{Weight $\lambda$} & \multicolumn{5}{c|}{In Distribution} & \multicolumn{5}{c}{Out of Distribution}\\
%       \cmidrule(lr){2-6} \cmidrule(lr){7-11}
%       & Open & Goto & PickUp & PutNext & All & Open & Goto & PickUp & PutNext & All \\
%       \midrule
%       0 & 98.0 & 90.9 & 89.0 & 39.1 & 
% 82.2 & 96.4 & 89.8 & 79.0 & 32.0 & 75.1 \\
%       0.01 & 98.3 & 94.6 & 87.5 & 37.4 & 
% 82.9 & 98.4 & 90.3 & 76.8 & 35.6 & 76.0 \\
%       0.1 & 98.5 & 96.1 & 91.1 & 44.9 & 
%  85.2 & 98.5 & \textbf{93.4} & 81.9 & 45.6 & 79.9\\
%       0.2 & 97.2 & \textbf{96.5} & \textbf{94.9} & \textbf{57.9} & 
% \textbf{89.2} & \textbf{99.0} & 91.3 & \textbf{87.6} & \textbf{49.1} & \textbf{81.7}\\
%       0.5 & \textbf{99.1} & 94.9 & 90.9 & 52.3 & 
% 87.0 & 96.8 & 93.3 & 86.4 & 46.2 & 80.9\\
%       1 & 98.9 & 94.7 & 91.0 & 53.7 & 
% 87.2 & 98.1 & 92.2 & 88.3 & 45.4 & 81.1\\
%       2 & 98.6 & 94.3 & 90.3 & 49.4 & 
% 86.4 & 98.0 & 92.1 & 84.6 & 41.5 & 79.1\\
%       5 & 97.7 & 92.2 & 89.2 & 43.4 & 
% 83.8 & 97.4 & 92.3 & 81.7 & 41.6 & 79.4\\
%       \bottomrule
%     \end{tabular}
%     \label{table:lambda}
%     % \vspace{-10pt}
% \end{center}
% \end{table*}
\clearpage
\section{Demonstration Trajectories in ALFRED}
\label{appendix:alfred_res}

We present extended trajectory visualizations of DAIL's task execution in the ALFRED benchmark, illustrating its semantic comprehension and generalization capabilities.

\begin{figure}[h]
    \centering
    \includegraphics[width=\textwidth]{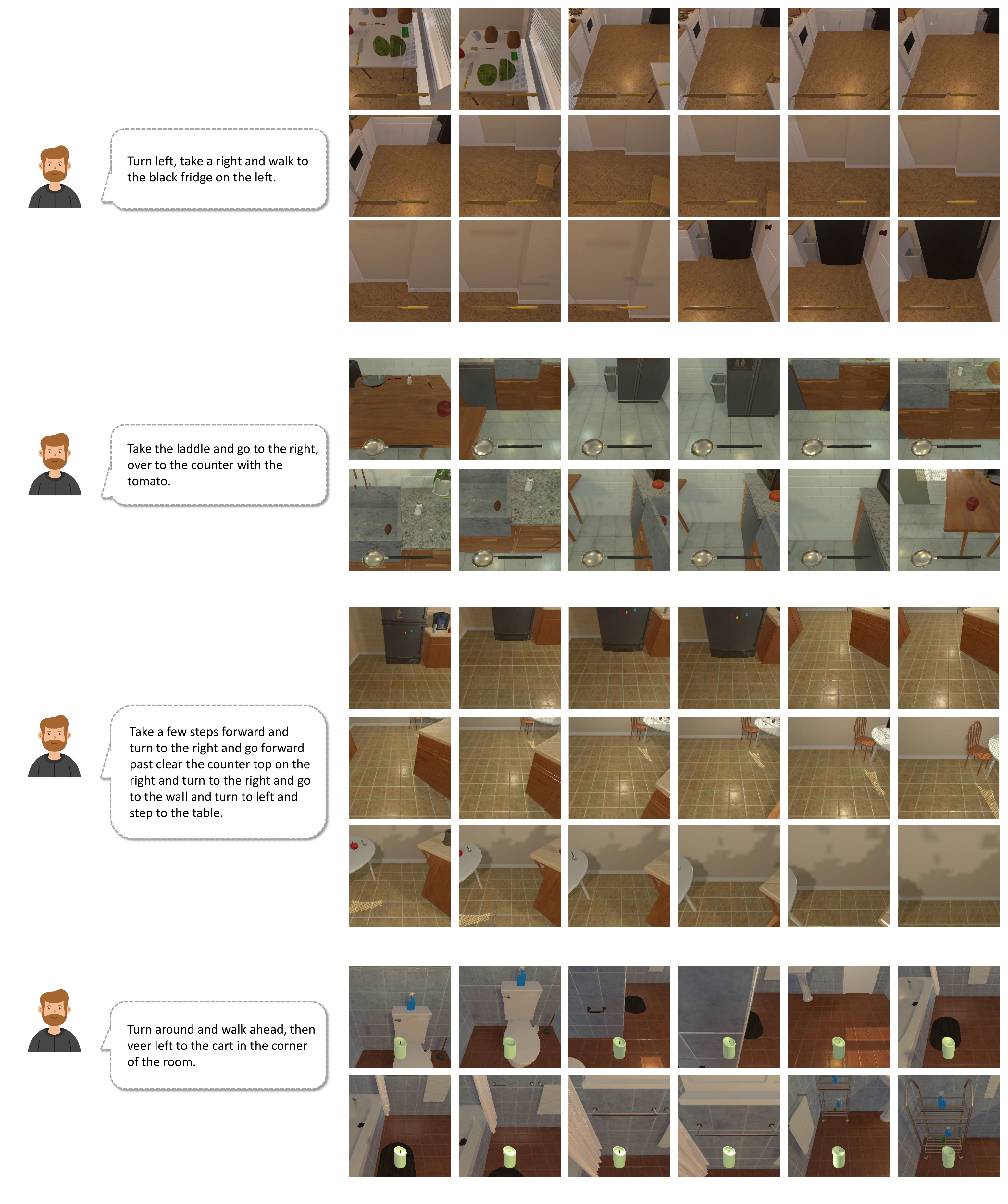}
    % \vspace{-25pt}
    \caption{Extended trajectories of DAIL in ALFRED validation tasks.}
    \label{fig:more_demonstration}
\end{figure}

\end{document}